\documentclass{article}
\usepackage{iclr2024_conference,times}

\usepackage[utf8]{inputenc}
\usepackage[T1]{fontenc}
\usepackage[english]{babel}
\usepackage{amsmath,amssymb,amsfonts}
\allowdisplaybreaks
\usepackage{amsthm}
\usepackage{geometry}
\usepackage{xcolor}
\usepackage{mathtools}
\mathtoolsset{showonlyrefs}
\usepackage{bm}
\usepackage{subcaption}
\usepackage{algorithm}
\usepackage{algpseudocode}
\usepackage{hyperref}
\usepackage[shortlabels]{enumitem}
\usepackage{listings}
\usepackage{tikz}
\usepackage{multirow}
\usepackage{arydshln}
\usetikzlibrary{spy}

\newtheorem{theorem}{Theorem}
\newtheorem{lemma}[theorem]{Lemma}
\newtheorem{proposition}[theorem]{Proposition}
\newtheorem{corollary}[theorem]{Corollary}

\theoremstyle{definition}

\theoremstyle{remark}
\newtheorem{remark}[theorem]{Remark}
\newtheorem{example}[theorem]{Example}

\newcommand{\R}{\mathbb{R}}
\newcommand{\N}{\mathbb{N}}
\renewcommand{\d}{\mathop{}\!\mathrm{d}}
\renewcommand{\P}{\mathcal{P}}
\renewcommand{\S}{\mathcal{S}}

\newcommand{\T}{\mathrm{T}}

\newcommand{\F}{\mathcal F}

\mathtoolsset{showonlyrefs}

\DeclareMathOperator{\Id}{Id}

\newcommand{\opt}{\mathrm{opt}}

\iclrfinalcopy

\begin{document}

\title{Posterior Sampling Based on Gradient Flows \\
of the MMD with Negative Distance Kernel}
\author{Paul Hagemann\thanks{hagemann@math.tu-berlin.de} \\Technische Universit\"at Berlin 
\And
Johannes Hertrich \\ University College London
\And
Fabian Altekr\"uger \\ Humboldt-Universit\"at zu Berlin 
\AND
Robert Beinert \\Technische Universit\"at Berlin 
\And
Jannis Chemseddine\\Technische Universit\"at Berlin 
\And
Gabriele Steidl\\Technische Universit\"at Berlin 
}    

\author{
P. Hagemann$^1$, J. Hertrich$^2$, F. Altekr\"uger$^3$, R. Beinert$^1$, J. Chemseddine$^1$, G. Steidl$^1$\\
   $^1$ Technische Universit\"at Berlin,
  $^2$ University College London,
  $^3$ Humboldt-Universit\"at zu Berlin\\
  Correspondence to: \texttt{hagemann@math.tu-berlin.de}
}

 \date{\today}
\maketitle

\begin{abstract}
We propose conditional flows of the maximum mean discrepancy (MMD)  with the negative distance kernel
for posterior sampling and conditional generative modelling. This MMD,  which is also known as energy distance,
has several advantageous properties like efficient computation via slicing and sorting. 
We approximate the joint distribution of the ground truth and the observations using discrete Wasserstein gradient flows and
establish an error bound for the posterior distributions.
Further, we prove that our particle flow is indeed a Wasserstein gradient flow of an appropriate functional. 
The power of our method is demonstrated by numerical examples 
including conditional image generation and inverse problems like superresolution, inpainting and computed tomography in low-dose and limited-angle settings.
\end{abstract}

%-------------------------------------
\section{Introduction}
%-------------------------------------

The tremendous success of generative models led to a rising interest in their application for inverse problems in imaging. 
Here, an unknown image $x$ has to be recovered from a noisy observation $y=f(x)+\xi$.
Since the forward operator $f$ is usually ill-conditioned, such reconstructions include uncertainties and are usually not unique.
As a remedy, we take a Bayesian viewpoint and consider $x$ and $y$ as samples from random variables $X$ and $Y$, 
and assume that we are given training samples from their joint distribution $P_{X,Y}$.
In order to represent the uncertainties in the reconstruction, we aim to find a process to sample from the posterior distributions $P_{X|Y=y}$.
This allows not only to derive different possible predictions, but also to consider pixel-wise standard deviations for identifying 
highly vague image regions.
Figure~\ref{fig:CT_limangle} visualizes this procedure on an example for limited angle computed tomography.

Nowadays, generative models like (Wasserstein) GANs \citep{ACB2017,GPM2014} and VAEs \citep{KW2013} have turned out to be a suitable tool for approximating probability distributions.
In this context, the field of gradient flows in measure spaces received increasing attention.
\citet{WT2011} proposed to apply the Langevin dynamics in order to generate samples from a known potential, which corresponds to simulating a Wasserstein gradient flow with respect to the Kullback-Leibler (KL)  divergence, see \citet{JKO1998}.
Score-based and diffusion models extend this approach by estimating the gradients of the potential from training data, see \citep{DTHD2021,ho2020denoising,SE2020,song2021scorebased} and achieved state-of-the-art results.
The simulation of Wasserstein gradient flows with other functionals than KL, based on the JKO scheme, was considered in \citet{AHS2023,ASM2022,Fan22,MKLGSB2021}.

In this paper, we focus on gradient flows with respect to  MMD with negative distance kernel $K(x,y)=-\|x-y\|$, which is also known as energy distance, see
\citep{mmd_energy_eq,szekely2002,szekely_brownian_dist,szekely_energy}. 
While MMDs have shown  great success at comparing two distributions in general, see \citep{Gretton2012, szekely_multvar, GBRSS2006}, their combination with the negative distance kernel
results in many additional desirable properties as translation and scale equivariance \citep{szekely_energy}, 
efficient computation \citep{HWAH2023}, a blue dimension independent sample complexity of $O(n^{-1/2})$ \citep{Gretton2012} blue and unbiased sample gradients \cite{bellemare2017cramer}.

To work with probability distributions in high dimensions, \citet{RPDB2012} proposed to slice them. 
Applied on gradient flows, this leads to a significant speed-up, see \citep{DLPYL2023,kolouri2018sliced,liutkus19}.
In particular, for MMD with negative distance kernel slicing does not change the metric itself and reduces the time complexity of calculating gradients from $O(N^2)$ to $O(N\log{N})$ for measures with $N$ support points, see \citet{HWAH2023}.

In order to use generative models for inverse problems, 
an additional conditioning parameter was added to the generation process in
\citep{ardizzone2018analyzing, ak21, chung2023diffusion, hagemann2022stochastic, mirza2014conditional}.
However, this approach cannot directly applied to the gradient flow setting, where the generator is not trained end-to-end.

\textbf{Contributions.} 
We simulate conditional MMD particle flows for posterior sampling in Bayesian inverse problems.
To this end, we provide three kinds of contributions. The first two address theoretical questions while the last one validates our findings numerically.
\begin{itemize}
\item 
Conditional generative models approximate the joint distribution  by learning a mapping $T$ such that
$P_{X,Y} \approx P_{T(Z,Y)),Y}$,
but in fact we are interested in the posterior distributions $P_{X|Y=y}$.
In this paper, we prove error bounds between posterior and joint distributions within the MMD metric in expectation.
The proofs of these results are based on relations between measure spaces and RKHS as well as Lipschitz stability results under pushforwards.
\item 
We represent the considered particle flows as Wasserstein gradient flows of a modified MMD functional.
As a side effect of this representation, we can provide a theoretical justification for the empirical method presented by \citet{DLPYL2023},
where the authors obtain convincing results by neglecting the velocity in the $y$-component in sliced Wasserstein gradient flows.
Based on locally isometric embeddings of the $\mathbb R^{N d}$ into the Wasserstein space, we can show that the result is again a Wasserstein gradient flow with respect to a modified functional.
\item We approximate our particle flows by conditional generative neural networks and apply the arising generative model in various settings.
On the one hand, this includes standard test sets like conditional image generation and inpainting on MNIST, FashionMNIST and CIFAR10 and superresolution on CelebA.
On the other hand, we consider very high-dimensional imaging inverse problems like superresolution of materials' microstructures as well as limited-angle and low-dose computed tomography.
\end{itemize}

\textbf{Related work.}
Many generative models like GANs, VAEs and normalizing flows can be used for posterior sampling by adding a conditioning parameter as an additional input, see \citep{ardizzone2018analyzing, ak21,batzolis2021conditional, hagemann2022stochastic, mirza2014conditional}.
The loss function of these methods is based on the Kullback--Leibler divergence.
In this case, stability results were proven by \citet{altekruger2023conditional} based on 
local Lipschitz regularity results from \citet{sprungk2020local}.

In this paper, we are interested in generative models which are based on gradient flows \citep{Fan22,kolouri2018sliced, HWAH2023, hertrich2023wasserstein, NH2022, MKLGSB2021}.
In this case, the above approach is not directly applicable since the network is not trained end-to-end.
For the sliced Wasserstein gradient flows, \citet{DLPYL2023} proposed to approximate the joint distribution while neglecting the velocity in one component.
They achieved very promising results, but evaluated their model empirically without giving theoretical justification.

Here, we consider gradient flows with respect to  MMD with negative distance kernel, which is also known as energy distance or Cramer distance \citep{mmd_energy_eq,szekely2002}.
The theoretical analysis of such gradient flows in the Wasserstein space turns out to be challenging, since discrete measures might become absolutely continuous and vice versa, see \citep{hertrich2023wasserstein}.
As a remedy, many papers consider particle flows as a space discretization, see \citep{CDEFS2020,daneshmand2023polynomialtime, daneshmand2023efficient,HWAH2023}.
The question whether the mean-field limit of these particle flows corresponds to the continuous Wasserstein gradient flow is so far only partially answered and still an active area of research, see \citep{CDEFS2020,daneshmand2023polynomialtime, daneshmand2023efficient}.
Further, there is plenty of literature covering the statistical properties of MMD in general \citep{sriperumbudur2010universality, Gretton2012, MD2023} as well as its applications to causal inference \citep{kremer2022functional,kremer2023estimation} via conditional moments.

\textbf{Outline of the paper.}
Section~\ref{sec:MMD} briefly recalls  MMD with negative distance kernel
and corresponding discrete Wasserstein gradient flows. Based on this,
we introduce conditional generative MMD flows in Section~\ref{sec:conditional_post} by the following path:
i)  we establish
relations between joint and posterior distributions, 
ii)
we present an interpretation of the conditional particle flow as Wasserstein gradient flow of an appropriate functional, and 
iii) we suggest a generative variant of the conditional MMD flow.
Numerical experiments are contained in Section~\ref{sec:exp} and conclusions are drawn in Section~\ref{sec:conclusions}.
The appendix contains all the
proofs, implementation details and additional experimental results.

\begin{figure}[t]
\centering
\begin{subfigure}[t]{.14\textwidth}
\begin{tikzpicture}[spy using outlines={rectangle,white,magnification=5.5,size=2.145cm, connect spies}]
\node[anchor=south west,inner sep=0]  at (0,0) {\includegraphics[width=\linewidth]{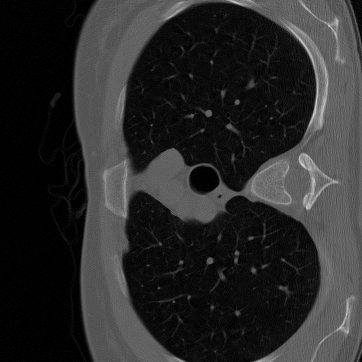}};
 \spy on (1.25,.99) in node [right] at (-0.01,-1.1);
\end{tikzpicture}
\caption*{GT}
\end{subfigure}%
\hfill
\begin{subfigure}[t]{.14\textwidth}
\begin{tikzpicture}[spy using outlines={rectangle,white,magnification=5.5,size=2.145cm, connect spies}]
\node[anchor=south west,inner sep=0]  at (0,0) {\includegraphics[width=\linewidth]{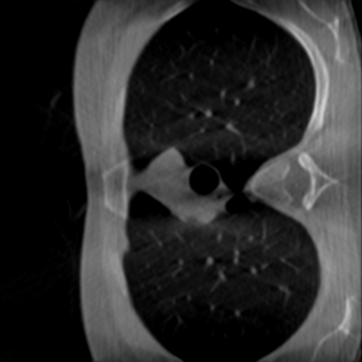}};
 \spy on (1.25,.99) in node [right] at (-0.01,-1.1);
\end{tikzpicture}
  \caption*{FBP}
\end{subfigure}%
\hfill
\begin{subfigure}[t]{.14\textwidth}
\begin{tikzpicture}[spy using outlines={rectangle,white,magnification=5.5,size=2.145cm, connect spies}]
\node[anchor=south west,inner sep=0]  at (0,0) {\includegraphics[width=\linewidth]{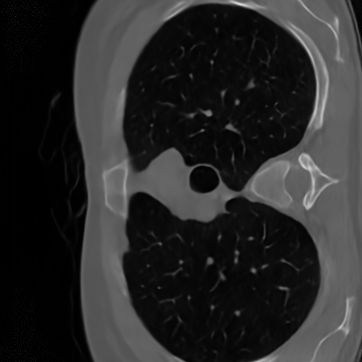}};
 \spy on (1.25,.99) in node [right] at (-0.01,-1.1);
\end{tikzpicture}
  \caption*{Reco 1}
\end{subfigure}%
\hfill
\begin{subfigure}[t]{.14\textwidth}
\begin{tikzpicture}[spy using outlines={rectangle,white,magnification=5.5,size=2.145cm, connect spies}]
\node[anchor=south west,inner sep=0]  at (0,0) {\includegraphics[width=\linewidth]{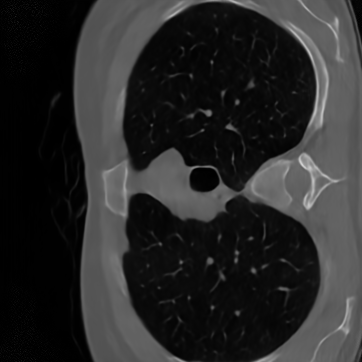}};
 \spy on (1.25,.99) in node [right] at (-0.01,-1.1);
\end{tikzpicture}
  \caption*{Reco 2}
\end{subfigure}%
\hfill
\begin{subfigure}[t]{.14\textwidth}
\begin{tikzpicture}[spy using outlines={rectangle,white,magnification=5.5,size=2.145cm, connect spies}]
\node[anchor=south west,inner sep=0]  at (0,0) {\includegraphics[width=\linewidth]{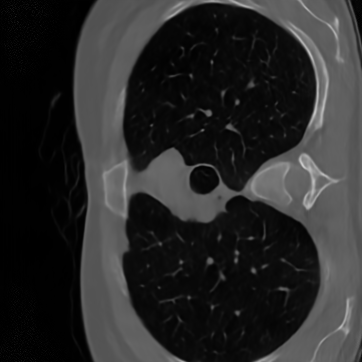}};
 \spy on (1.25,.99) in node [right] at (-0.01,-1.1);
\end{tikzpicture}
  \caption*{Reco 3}
\end{subfigure}%
\hfill
\begin{subfigure}[t]{.14\textwidth}
\begin{tikzpicture}[spy using outlines={rectangle,white,magnification=5.5,size=2.145cm, connect spies}]
\node[anchor=south west,inner sep=0]  at (0,0) {\includegraphics[width=\linewidth]{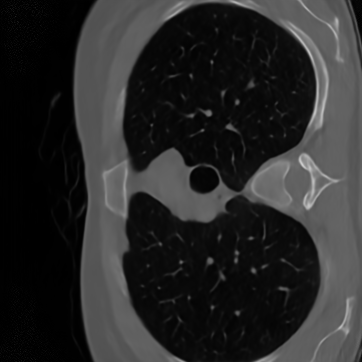}};
 \spy on (1.25,.99) in node [right] at (-0.01,-1.1);
\end{tikzpicture}
  \caption*{Mean}
\end{subfigure}%
\hfill
\begin{subfigure}[t]{.14\textwidth}
\begin{tikzpicture}[spy using outlines={rectangle,white,magnification=5.5,size=2.145cm, connect spies}]
\node[anchor=south west,inner sep=0]  at (0,0) {\includegraphics[width=\linewidth]{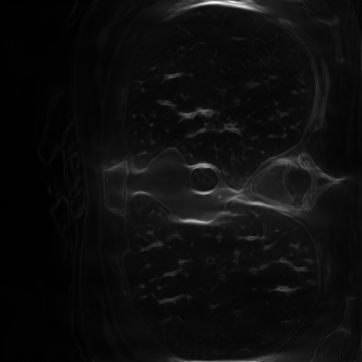}};
 \spy on (1.25,.99) in node [right] at (-0.01,-1.1);
\end{tikzpicture}
  \caption*{Std}
\end{subfigure}%
\caption{Generated posterior samples, mean image and pixel-wise standard deviation for limited angle computed tomography using conditional MMD flows.} \label{fig:CT_limangle}
\end{figure}

%--------------------------------------------------------------------------
\section{MMD and Discrete Gradient Flows} \label{sec:MMD}
%--------------------------------------------------------------------------
Let $\mathcal{P}_p(\R^d)$, $p \in (0,\infty)$ denote the space of probability measures with finite $p$-th moments.
We are interested in 
the MMD $\mathcal D_K \colon \mathcal P_2(\R^d) \times \mathcal P_2(\R^d) \to \mathbb R$ of the negative distance kernel $K(x,y) \coloneqq - \|x-y\|$
defined  by
\begin{equation} \label{mmd}
\begin{aligned}
\mathcal D_K^2(\mu,\nu) &\coloneqq  
\frac12 \int_{\R^d} \int_{\R^d} K(x,y)  \, \d \mu(x) \d \mu(y) 
\, -  
\int_{\R^d} \int_{\R^d}  K(x,y)  \, \d \mu(x) \d \nu(y)\\
& \quad +
 \frac12 \int_{\R^d} \int_{\R^d}  K(x,y)  \, \d \nu(x) \d \nu(y).
 \end{aligned}
\end{equation}
It is a metric on $\mathcal P_1(\R^d) \supset \mathcal P_2(\R^d)$, see e.g.~\citet{mmd_energy_eq,szekely_energy}.
In particular, we have that $\mathcal D_K(\mu,\nu) = 0$ if and only if $\mu=\nu$.

For a fixed measure $\nu\in \mathcal P_2(\mathbb R^d)$, 
we consider Wasserstein gradient flows 
of the functional $\F_\nu\colon\P_2(\R^d)\to(-\infty,\infty]$ defined by 
\begin{equation}\label{eq:F_nu}
\mathcal F_\nu(\mu) \coloneqq \mathcal D_K^2(\mu,\nu) - \text{const},
\end{equation}
where the constant is just the third summand in \eqref{mmd}.
 For a definition of Wasserstein gradient flows, we refer to 
Appendix \ref{proof:posterior_flow}.
Since the minumum of this functional is $\nu$, we can use the flows to sample from this target distribution.
While the general analysis of these flows is theoretically challenging, 
in particular for the above non-smooth and non-$\lambda$-convex negative distance kernel, see \citet{CDEFS2020,hertrich2023wasserstein}, 
we focus on the efficient numerical simulation via particle flows as proposed in \citet{HWAH2023}. 
To this end, let 
$P_N  \coloneqq \{\frac1N\sum_{i=1}^N\delta_{x_i}: x_i \in \R^{d} , x_i \not = x_j , i \not = j\}$
denote the set of empirical measures on $\R^d$ with $N$ pairwise different anchor points.
Given $M$ independent samples $p=(p_1,...,p_M)\in(\R^{d})^M$ of the measure $\nu$, we deal with its empirical version
$\nu_M\coloneqq \frac1M\sum_{i=1}^M\delta_{p_i}$ and
consider the Euclidean gradient flow of the \emph{discrete MMD functional} 
$F_p\colon(\R^{d})^N\to\R$
defined for $x=(x_1,...,x_N)\in(\R^d)^N$ by
\begin{align} \label{eq:discrete_mmd_functional}
F_p(x) \coloneqq \mathcal F_{\nu_M} (\mu_N) 
= 
-\frac{1}{2N^2}
\sum_{i=1}^N\sum_{j=1}^N \|x_i-x_j\|+\frac{1}{MN}\sum_{i=1}^N\sum_{j=1}^M\|x_i-p_j\|.
\end{align}
Then, a curve $u=(u_1,...,u_N)\colon [0,\infty)\to(\R^d)^N$ solves the ODE
$$
\dot u=-N\nabla F_p(u),\quad u(0)=(u_1^{(0)},...,u_N^{(0)}),
$$
if and only if the curve $\gamma_{N} \colon (0,\infty)\to \mathcal P_2(\R^d)$ defined by 
$\gamma_{N}(t)=\frac1N\sum_{i=1}^N \delta_{u_i(t)}$ 
is a Wasserstein gradient flows with respect to the functional 
$\mathcal J_{\nu_M}\colon\P_2(\R^d)\to\R\cup\{\infty\}$ 
given by
$$
\mathcal J_{\nu_M} (\mu) \coloneqq 
\left\{
\begin{array}{ll}
\F_{\nu_M}(\mu), &\text{ if } \mu\in P_N,\\
\infty, & \text{ otherwise}.
\end{array}
\right. 
$$
In \citet{HWAH2023}, this simulation of MMD flows was used to derive a generative model.

%--------------------------------------------------------------------------------------
\section{Conditional MMD Flows for Posterior Sampling}\label{sec:conditional_post}
%--------------------------------------------------------------------------------------
In this section, we propose \emph{conditional}  flows for posterior sampling.
We consider two random variables $X \in \R^d$ and $Y \in \R^n$. Then, we aim to sample from the posterior distribution $P_{X|Y=y}$. One of the most important applications for this are Bayesian inverse problems. Here $X$ and $Y$
are related by
$
Y=\mathrm{noisy}(f(X)),
$
where $f\colon\R^d\to\R^n$ is some ill-posed forward operator and ``$\mathrm{noisy}$'' denotes some noise process.
Throughout this paper, we assume that we are only given samples 
from the joint distribution $P_{X,Y}$.
In order to sample from the posterior distribution $P_{X|Y=y}$, we will use conditional generative models.
More precisely, we aim to find a mapping $T\colon\R^d\times\R^n\to\R^d$ such that 
\begin{equation}\label{eq:cond_gen_mod}
T(\cdot,y)_\#P_Z= P_{X|Y=y},
\end{equation} 
where $P_Z$ is an easy-to-sample latent distribution and $T(\cdot,y)_\#P_Z = P_Z(T^{-1}(\cdot,y))$ defines the pushforward of a measure.
The following proposition summarizes the main principle of conditional generative modelling and provides a sufficient criterion for \eqref{eq:cond_gen_mod}.
To make the paper self-contained, we add the proof in Appendix~\ref{proof:fundamental_conditional_generative}.

\begin{proposition}\label{lem:fundamentalthm_of_conditional_generative_modelling}
Let $X,Z \in \R^d$ be independent random variables and $Y \in \mathbb R^n$ be another random variable.
Assume that $T\colon\R^d\times\R^n\to \R^d$ fulfills
$
P_{T(Y,Z),Y}=P_{X,Y}
$.
Then, it holds $P_Y$-almost surely that
$
T(\cdot,y)_\#P_Z=P_{X|Y=y}.
$
\end{proposition}

In Subsection~\ref{post_joint}, we extend this result and show that \eqref{eq:cond_gen_mod} still holds true approximately, whenever the distributions $P_{T(Y,Z),Y}$ and $P_{X,Y}$  are close to each other. 
Then, in Subsection~\ref{condMMD}, we construct such a mapping $T$ based on Wasserstein gradient flows with respect to a conditioned version of the functional $\mathcal F_{P_{X,Y}}$.
Finally, we propose an approximation of this Wasserstein gradient flow by generative neural networks in Subsection~\ref{sec:conditional}.

%-------------------------------------------------------------------------------
\subsection{Posterior Versus Learned Joint Distribution} \label{post_joint}

In Proposition~\ref{lem:fundamentalthm_of_conditional_generative_modelling}, we assume that $T$ is perfectly learned,
which is rarely the case in practice.  
Usually, we can only approximate the joint distribution
such that $\mathcal D_K(P_{T(Z,Y),Y},P_{X,Y})$ becomes small.
Fortunately, we can prove under moderate additional assumptions that then the expectation with respect to $y$ 
of the distance of the posteriors $\mathcal D_K \left(T(\cdot,y)_\#P_Z,P_{X|Y=y} \right)$
becomes small too. A similar statement was shown by \citet[Proposition 3]{CondWasGen}; note that their RHS is not equal to the MMD of joint distribution, but a modified version. Such a statement involving MMDs can not generally hold true, see Example \ref{cex:comp-range}.

\begin{theorem} \label{thm:fundamental}
Let $S_n \subset \R^n$ and $S_d \subset \R^d$ be compact sets. Further, 
let $X, \tilde X \in S_d$ and $Y \in S_n$ be absolutely continuous random variables,
and
assume that $P_{X,Y}$ and $P_{\tilde X,Y}$ have densities fulfilling
\begin{align}
 |p_{X|Y = y_1} - p_{X|Y = y_2}| 
 &\le C_{S_n} \| y_1 - y_2\|^\frac12,
 \label{eq:HC-X}
 \\
    |p_{\tilde X|Y = y_1} - p_{\tilde X|Y = y_2}| 
    &\le C_{S_n} \| y_1 - y_2\|^\frac12
    \label{eq:HC-tildeX}
\end{align}
a.e.\ on  $S_d$ for all $y_1,y_2 \in S_n$.
Then it holds  
\begin{align} \label{result}
\mathbb{E}_{y \sim P_Y}[\mathcal{D}_K(P_{\tilde X|Y=y}, P_{X|Y=y})] \leq C \, \mathcal{D}_K(P_{\tilde X,Y}, P_{X,Y})^{\frac{1}{4(d+n+1)}}.
\end{align}
\end{theorem}

The proof, which uses relations between measure spaces and reproducing kernel Hilbert spaces (RKHS),  is given in  Appendix \ref{proof:fundamental}. 
Similar relations hold true for other ,,distances'' as the  Kullback--Leibler divergence
or the Wasserstein distance, see Remark~\ref{rem:non-comp}. For the latter one as well as for the MMD, it is necessary that 
the random variables are compactly supported.
Example~\ref{cex:comp-range} shows that \eqref{result}
is in general not correct if this assumption is neglected. 

Based on Theorem \ref{thm:fundamental}, we can prove pointwise convergence 
of a sequence of mappings in a similar way as it was done for Wasserstein distances by \citet{altekruger2023conditional}. 
For this,
we require the existence of $C, \tilde C > 0$ such that 
\begin{align}
\mathcal{D}_K( T(\cdot,y_1)_{\#}P_Z,T(\cdot,y_2)_{\#}P_Z) 
&\leq C 
\Vert y_1-y_2 \Vert^{\frac12} \quad \text{(stability under pushforwards)} ,\label{stab1}\\
\mathcal{D}_K(P_{X|Y=y_1},P_{X|Y=y_2}) &\leq \tilde C \Vert y_1-y_2 \Vert^{\frac12} \quad \text{(stability of posteriors)} \label{stab2}
\end{align}
for all $y_1,y_2 \in S_n$ with $p_Y(y_1), p_Y(y_2) > 0$.
In Appendix~\ref{proof:fundamental}, we show that
both stability estimates hold true under certain conditions,
see Lemma~\ref{stabpush} and \ref{stabpost}.
Furthermore, we prove
 following theorem.

\begin{theorem} \label{cor_pointwise}
Let $S_n \subset \R^n$ and $S_d \subset \R^d$ be compact sets and  $X, Z \in S_d$ and $Y \in S_n$  absolutely continuous random variables.
Assume that $\lvert p_Y(y_1) - p_Y(y_2)\rvert \le C' \lVert y_1 - y_2 \lVert^{\frac{1}{2}}$
for all $y_1,y_2 \in S_n$ and fixed $C' > 0$.
Let $P_{X,Y}$ have a density satisfying \eqref{eq:HC-X}.
 Moreover, let   
 $\{T^\varepsilon: \R^d \times \R^n \rightarrow \R^d\}$
 be a family of measurable mappings
 with
 $\mathcal{D}_K (P_{T^{\varepsilon}(Z,Y),Y}, P_{X,Y}) \le \varepsilon$,
 which fulfill \eqref{stab1} and \eqref{stab2}
 with uniform $C, \tilde C > 0$.
 Further, 
 assume that $P_{\tilde X, Y}$ with $\tilde X = T^\varepsilon(Z,Y)$
 have densities satisfying \eqref{eq:HC-tildeX} 
 with uniform $C_{S_n} > 0$.
  Then, for all $y \in S_n$ with $p_Y(y) > 0$, it holds 
\begin{align}
\mathcal{D}_K(T^{\varepsilon}(\cdot,y)_{\#}P_Z, P_{X|Y=y}) \rightarrow 0 
\quad \text{as} \quad
\varepsilon \to 0.
\end{align}   
\end{theorem}

%----------------------------
\subsection{Conditional MMD Flows}\label{condMMD}
%----------------------------

In the following, we consider particle flows to approximate the joint distribution $P_{X,Y}$.
Together with the results from the previous subsection,
this imposes an approximation of the posterior distributions
$\smash{P_{X|Y=y}}$.
Let $N$ pairwise distinct samples
$(p_i,q_i)\in\R^d\times\R^n$ 
from the joint distribution $P_{X,Y}$
be given,
and set \smash{$(p,q) \coloneqq \left( (p_i,q_i)\right)_{i=1}^N$}.
Let $P_Z$ be a $d$-dimensional latent distribution, where we can easily sample from.
We draw a sample $z_i$ for each $i$ and consider 
the particle flow 
$t \mapsto (u(t), q)$
starting at $((z_i,q_i))_{i=1}^N$,
where the second component remains fixed and the first component 
$u=(u_1,...,u_N)\colon[0,\infty)\to\R^{dN}
$
follows the gradient flow 
\begin{equation}\label{eq:posterior_ODE}
 \dot u(t)=-N\nabla_x F_{(p,q)} \left((u,q) \right),
 \quad 
 u(0)=(z_1,\dots,z_N),
\end{equation}
with the function $F_{(p,q)}$
in \eqref{eq:discrete_mmd_functional} and the gradient $\nabla_x$ with respect to the first component.
Since the MMD is a metric, the function $x\mapsto F_{(p,q)}((x,q))$ admits the
global minimizer $x=(p_1,...,p_N)$.
In our numerical examples, we observe that the gradient flow \eqref{eq:posterior_ODE} approaches
this global minimizer as $t\to\infty$.
Finally, we approximate the motion of the particles by a mapping $T\colon\R^d\times\R^n\to\R^d$, which describes how
the initial particle $(z_i,q_i)$ moves to the particle $u_i(t_\mathrm{max})$ at some fixed time $t_\mathrm{max}$.
Moreover, by the convergence of the gradient flow \eqref{eq:posterior_ODE}, we have that
$$
P_{T(Z,Y),Y}\approx \frac1N\sum_{i=1}^N\delta_{T(z_i,q_i),q_i}=\frac1N\sum_{i=1}^N\delta_{u_i(t_\mathrm{max}),q_i}\approx \frac1N\sum_{i=1}^N\delta_{p_i,q_i} \approx P_{X,Y}.
$$
In particular, the arising mapping $T$ fulfills the assumptions from the previous subsection such that we obtain
$$
T(\cdot,y)_\#P_Z\approx P_{X|Y=y},\quad \text{for }P_Y-\text{a.e. }y\in\R^n. 
$$

The following theorem states that the solutions of \eqref{eq:posterior_ODE} correspond to 
Wasserstein gradient flows with respect to a conditioned MMD functional. To this end, set
$P_{N,q}  \coloneqq \{\frac1N\sum_{i=1}^N\delta_{x_i,q_i}: (x_i,q_i) \in \R^{d} \times \R^n , (x_i,q_i) \not = (x_j,q_j) , i \not = j\}$.

\begin{theorem}\label{thm:posterior_flow}
For given $(p_i,q_i) \in \R^d \times \R^n$, $i=1,\ldots,N$, set
$\nu_{N,q}\coloneqq\frac{1}{N}\sum_{i=1}^N\delta_{p_i,q_i}$.
Let $u=(u_1,...,u_N)\colon[0,\infty)\to(\R^{d})^N$ be a solution of \eqref{eq:posterior_ODE}, and assume $(u_i(t),q_i)\not = (u_j(t),q_j)$ for $i \not = j$ and all $t>0$.
Then the curve $\gamma_{N,q}\colon(0,\infty)\to\P_2(\R^d)$ defined by
$$
\gamma_{N,q}(t)=\frac{1}{N}\sum_{i=1}^N\delta_{u_i(t),q_i}
$$
is a Wasserstein gradient flow with respect to the functional $\mathcal J_{\nu_{N,q}}\colon\P_2(\R^d)\to\R\cup\{\infty\}$ given by 
\begin{align*}
\mathcal J_{\nu_{N,q}}\coloneqq
\begin{cases}
\F_{\nu_{N,q}},&\text{if } \mu \in P_{N,q},\\
\infty,&\text{otherwise,} 
\end{cases}
\end{align*}
where $\F_{\nu_{N,q}}$ is the functional defined in \eqref{eq:F_nu}.
\end{theorem}

The  definition of Wasserstein gradient flows 
and the proof 
are included in Appendix~\ref{proof:posterior_flow}. In particular, we will see that there is no flow in the
second component.

In order to approximate the joint distribution $P_{X,Y}$ starting in $P_{Z,Y}$, \citet{DLPYL2023} obtained convincing results by considering (discretized) Wasserstein gradient flows with respect to the sliced Wasserstein distance
$$
\mathcal{SW}_2^2(\mu,\nu)=\mathbb{E}_{\xi\in\S^{d-1}}[\mathcal W_2^2({P_\xi}_\#\mu,{P_\xi}_\#\nu)],\quad P_\xi(x)=\langle \xi,x\rangle.
$$
They observed in their experiments that there is nearly no flow in the second component, but acknowledged that they 
\emph{``are unable
to provide a rigorous theoretical justification for the time
being.''} Our proof of Theorem \ref{thm:posterior_flow}, more precisely Corollary \ref{du}  in Appendix \ref{proof:posterior_flow}
delivers the theoretical justification of their empirical result.

%--------------------------------------------------------------------------------------
\subsection{Conditional Generative MMD Flows}\label{sec:conditional}
%--------------------------------------------------------------------------------------

In this subsection we want to learn the mapping $T\colon\R^d\times\R^n \to \R^d$ describing the evolution, how the initial particles $(z_i,q_i)$ move to $u_i(t_\mathrm{max})$, where $u=(u_1,...,u_N)\colon[0,\infty)\to\R^{dN}$ solves the ODE \eqref{eq:posterior_ODE}.
To this end, we adopt the generative MMD flows from \citet{HWAH2023} and simulate the ODE \eqref{eq:posterior_ODE} using an explicit Euler scheme.
More precisely, we compute iteratively $u^{(k)}=(u_1^{(k)},...,u_N^{(k)})$ by 
\begin{equation}\label{eq:time-discrete-ODE}
u^{(k+1)}=u^{(k)}-\tau N \nabla_x F_{(p,q)}((u^{(k)},q)),\quad u_i^{(0)}=z_i.
\end{equation}
In order to evaluate the gradient efficiently and to speed up the computations, we use the sliced computation of $\nabla_x F_{(p,q)}((u^{(k)},q))$ and the momentum form of \eqref{eq:time-discrete-ODE} as proposed in \citet{HWAH2023}.
Now, we train neural networks $\Phi_1,...,\Phi_L\colon \R^d\times\R^n\to\R^d$ taking 
$(u_i,q_i)$ as an input, such that each network
approximates  a fixed number $T_l$ of explicit Euler steps from \eqref{eq:time-discrete-ODE}.
The precise training procedure is outlined in Algorithm~\ref{alg:training_gen_MMD_flows} in the appendix.
Once the networks $\Phi_l$ are trained, the approximating mapping $T$ is given by
$T(\cdot,y)=\Phi_L(\cdot,y)\circ\cdots\circ\Phi_1(\cdot,y)$.
In particular, we can generate samples from the approximated posterior distribution $P_{X|Y=y}$ by drawing $z\sim P_Z$ and computing $T(z,y)$, even for samples $y$ of $Y$ which are not contained in the training set.

%-----------------------------
\section{Experiments} \label{sec:exp}
%-----------------------------
We apply our conditional generative MMD flows to generate images for  given conditions $y$ in two settings, namely
i) class-conditional image generation, and 
ii) reconstruction from  posterior distributions $P_{X|Y=y}$  in  inverse problems. 
The chosen networks $(\Phi_l)_{l=1}^L$ are UNets \citep{RFB2015}, where we adopted the implementation from \cite{huang2021variational} based on \cite{ho2020denoising}.
Further details are given in Appendix~\ref{app:implementation}.

%---------------------------------------
\subsection{Class-conditional image generation}

We choose the condition $Y$ to be the one-hot vectors of the class labels in order to generate samples of MNIST \citep{LBBH1998}, FashionMNIST \citep{XRV2017} and CIFAR10 \citep{K2009} for given class labels. 
Figure~\ref{fig:classcond_samples} illustrates the generated samples and reports the average class conditional FID values. That is, we compute the FID between the generated samples from a specific class with all test samples from the same class. Note that class conditional FID values are not comparable with unconditional FIDs. We compare the results with $\ell$-SWF of \citet{DLPYL2023}. We observe that the conditional MMD flow generates samples of good visual quality and outperforms \cite{DLPYL2023}. Further examples are given in Appendix~\ref{app:further_examples}.

\begin{figure}
\begin{figure}[H]
\begin{minipage}[t]{.59\textwidth}
\begin{subfigure}[t]{.33\textwidth}
\includegraphics[width=\linewidth]{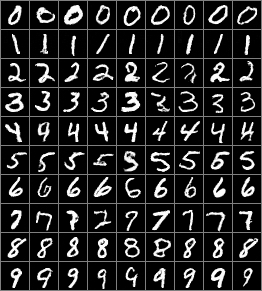}
\caption*{MNIST}
\end{subfigure}%
\hfill
\begin{subfigure}[t]{.33\textwidth}
\includegraphics[width=\linewidth]{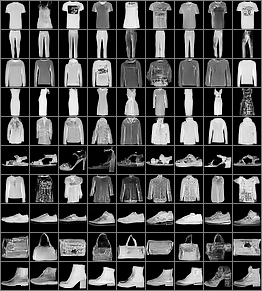}
  \caption*{FashionMNIST}
\end{subfigure}%
\hfill
\begin{subfigure}[t]{.33\textwidth}
\includegraphics[width=\linewidth]{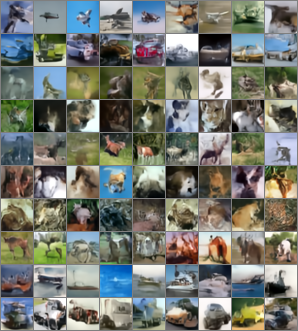}
  \caption*{CIFAR10}
\end{subfigure}%
\end{minipage}
\begin{minipage}[t]{.4\textwidth}
\vspace{-2.5cm}
\scalebox{.7}{
\begin{tabular}[t]{c|ccc} 
   Class conditional       & $\ell$-SWF          & Cond. MMD Flow   \\
   FID          & \citep{DLPYL2023}   &   (ours) \\        
\hline
MNIST        &   18.9              &  14.6    \\ 
FashionMNIST &   25.9              & 27.5\\
 CIFAR10     & 118.1               & 101.8   \\ 
\end{tabular}
}
\vspace{.1cm}

\end{minipage}
\caption{Class-conditional samples of MNIST, FashionMNIST and CIFAR10 and average class conditional FIDs. Note that these FID values are \textbf{not} comparable
to unconditional FID values. A more detailed version is given in Table~\ref{table:FID_classcond}.} 
\label{fig:classcond_samples}
\end{figure}

\begin{figure}[H]

\centering
\begin{subfigure}[t]{\textwidth}
\includegraphics[width=\linewidth]{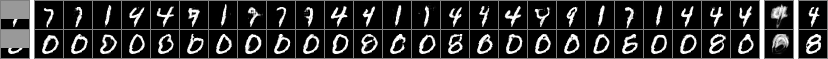}
\end{subfigure}%

\begin{subfigure}[t]{\textwidth}
\includegraphics[width=\linewidth]{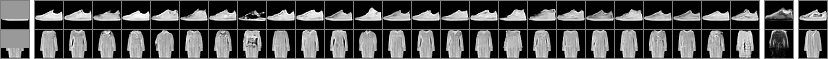}
\end{subfigure}%

\begin{subfigure}[t]{\textwidth}
\includegraphics[width=\linewidth]{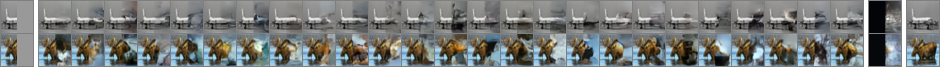}
\caption{Image inpainting (from top to bottom) for MNIST, FashionMNIST, CIFAR10.} 
\end{subfigure}%

\begin{subfigure}[t]{\textwidth}
\includegraphics[width=\linewidth]{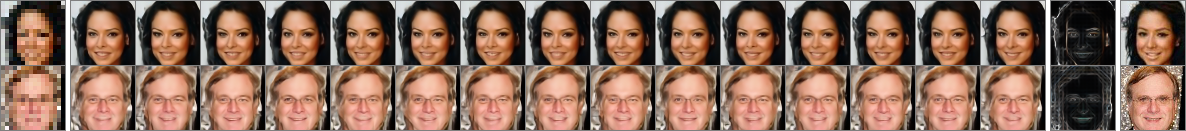}
\caption{Image superresolution for CelebA with magnification factor 4.} 
\end{subfigure}%
\caption{Image inpainting and superresolution for different data sets.}
\label{fig:inpainting}
\end{figure}

\begin{figure}[H]
\centering
\scalebox{1.}{
\begin{minipage}[t]{.63\textwidth}
\begin{subfigure}[t]{.195\textwidth}
\begin{tikzpicture}[spy using outlines={rectangle,black,magnification=10,size=1.83cm, connect spies}]
\node[anchor=south west,inner sep=0]  at (0,0) {\includegraphics[width=\linewidth]{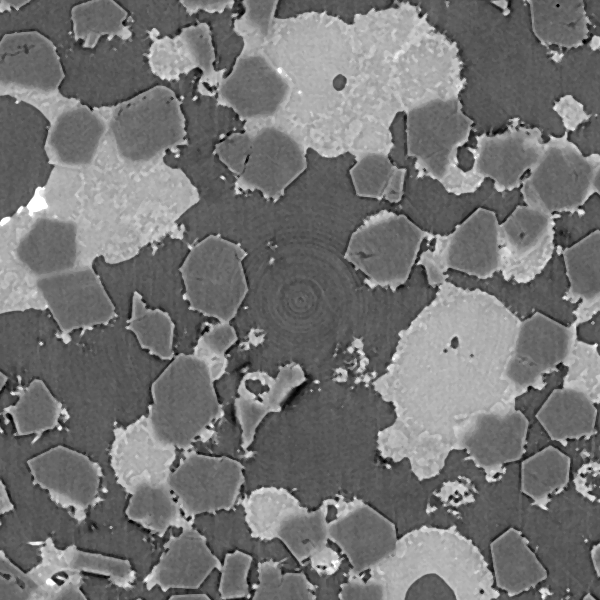}};
 \spy on (0.87,.67) in node [right] at (0.02,-.97);
\end{tikzpicture}
\caption*{HR image}
\end{subfigure}%
\hfill
\begin{subfigure}[t]{.195\textwidth}
\begin{tikzpicture}[spy using outlines=
{rectangle,black,magnification=10.5,size=1.83cm, connect spies}]
\node[anchor=south west,inner sep=0]  at (0,0) {\includegraphics[width=\linewidth]{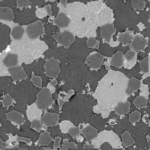}};
 \spy on (0.85,.69) in node [right] at (0.02,-.97);
\end{tikzpicture}
\caption*{LR image}
\end{subfigure}%
\hfill
\begin{subfigure}[t]{.195\textwidth}
\begin{tikzpicture}[spy using outlines={rectangle,black,magnification=10,size=1.83cm, connect spies}]
\node[anchor=south west,inner sep=0]  at (0,0) {\includegraphics[width=\linewidth]{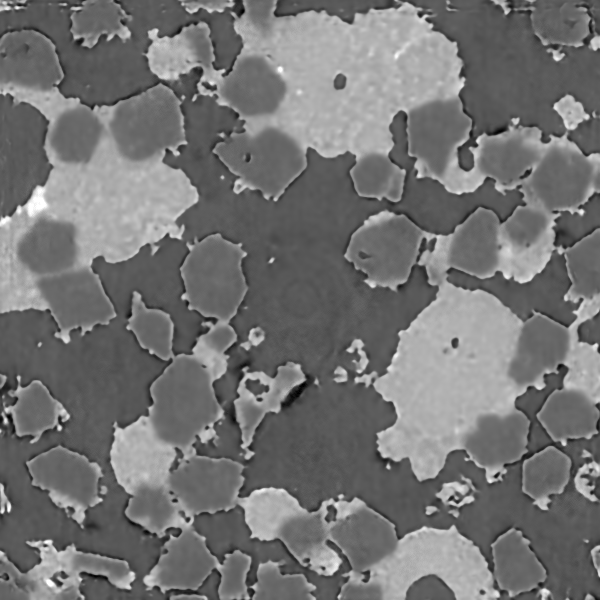}};
 \spy on (0.87,.67) in node [right] at (0.02,-.97);
\end{tikzpicture}
\caption*{Prediction 1}
\end{subfigure}%
\hfill
\begin{subfigure}[t]{.195\textwidth}
\begin{tikzpicture}[spy using outlines={rectangle,black,magnification=10,size=1.83cm, connect spies}]
\node[anchor=south west,inner sep=0]  at (0,0) {\includegraphics[width=\linewidth]{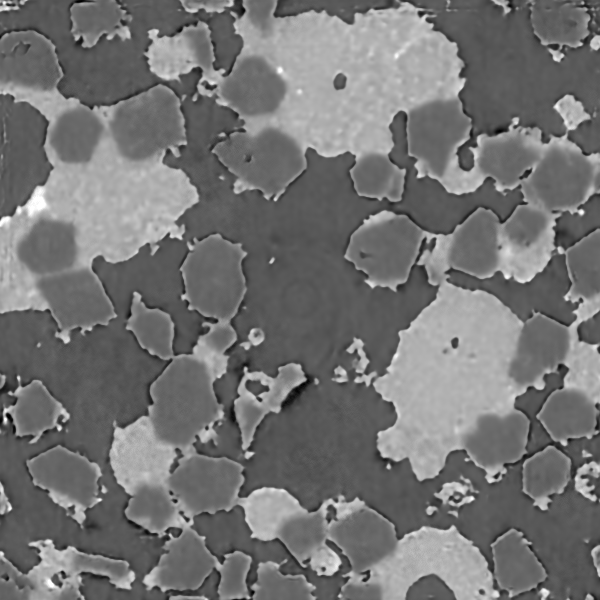}};
 \spy on (0.87,.67) in node [right] at (0.02,-.97);
\end{tikzpicture}
\caption*{Prediction 2}
\end{subfigure}%
\hfill
\begin{subfigure}[t]{.195\textwidth}
\begin{tikzpicture}[spy using outlines={rectangle,black,magnification=10,size=1.83cm, connect spies}]
\node[anchor=south west,inner sep=0]  at (0,0) {\includegraphics[width=\linewidth]{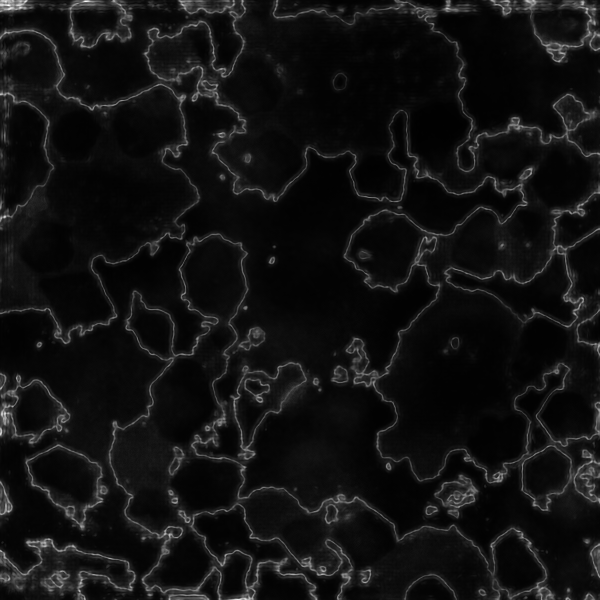}};
 \spy on (0.87,.67) in node [right] at (0.02,-.97);
\end{tikzpicture}
\caption*{Std}
\end{subfigure}%
\end{minipage}
\begin{minipage}[t]{.35\textwidth}
\vspace{-2.9cm}
\scalebox{.7}{
\begin{tabular}[t]{c|ccc} 
              & PSNR    & SSIM  \\
\hline
WPPFlow         & \multirow{2}{*}{25.89} & \multirow{2}{*}{0.657}       \\ 
\citep{AH2022}\\ \hdashline
SRFlow         & \multirow{2}{*}{25.11}   & \multirow{2}{*}{0.637}       \\  
\citep{LDVT2020}\\\hdashline
Cond. MMD Flow         &  \multirow{2}{*}{27.21} & \multirow{2}{*}{0.774} \\
(ours)
\end{tabular}} 
\end{minipage}}
\caption{Two different posterior samples and pixel-wise standard deviation for superresolution using conditional MMD flows.} \label{fig:SiC_main}
\end{figure}
\end{figure}

\subsection{Inverse problems}

\textbf{Inpainting.}
For the inpainting task with the mask operator $f$, 
let $y$ be the partially observed images. 
Inpainted images of MNIST, FashionMNIST and CIFAR10 are shown in Figure~\ref{fig:inpainting}. 
The observed images are the leftmost ones, while the unknown ground truth images are given in the rightmost column. 
The various generated samples in the middle column have very good reconstruction quality and are in particular consistent with the observed part. Their pixelwise standard deviation (std) is given in the second last column.

\textbf{Superresolution.}
For image superresolution, we consider the superresolution operator $f$ and low-resolution images $y$. Reconstructed high-resolution images of CelebA \citep{LLWT2015} are illustrated in Figure~\ref{fig:inpainting}. For CelebA, we centercrop the images to $140\times 140$ and then bicubicely downsample them to $64\times 64$. Again, the observations are the leftmost ones, the unknown ground truth the rightmost ones and the reconstructions in the middle column are of good quality and high variety. Another superresolution example on high-dimensional and real-world images of materials' microstructures is presented in Figure~\ref{fig:SiC_main}. 
We benchmark our conditional MMD flow against a conditional normalizing flow  ``SRFlow'' \citep{LDVT2020} and WPPFlow \citep{AH2022} in terms of PSNR and SSIM \citep{WBSS04}. A more detailed description is given in Appendix~\ref{app:SiC_superres}.

\textbf{Computed Tomography.}
The forward operator $f$ is the discretized linear Radon transform and the noise process is given by a scaled negative log-Poisson noise, for details see, e.g., \cite{ADHHMS2023,LoDoPaB21}.
The data is taken from the LoDoPaB dataset of \cite{LoDoPaB21} for \emph{low-dose CT imaging}.
The results are shown in Figure~\ref{fig:CT_reco}. We illustrate the mean image of $100$ reconstructions, its  error towards the ground truth and the standard deviation of the reconstructions. A quantitative comparison with \citep{DSLM2021} for this example is given in Fig~\ref{fig:CT_reco}. Here we provide PSNR and SSIM of the mean images for the whole testset containing 3553 images.
More examples towards limited angle CT and low-dose CT are given in Figures~\ref{fig:add_CT_lowdose} and \ref{fig:add_CT_limangle} in Appendix~\ref{app:further_examples}.

\begin{figure}[t]
\centering
\scalebox{1.}{
\begin{minipage}[t]{.63\textwidth}
\begin{subfigure}[t]{.195\textwidth}
\begin{tikzpicture}[spy using outlines={rectangle,white,magnification=5,size=1.89cm, connect spies}]
\node[anchor=south west,inner sep=0]  at (0,0) {\includegraphics[width=\linewidth]{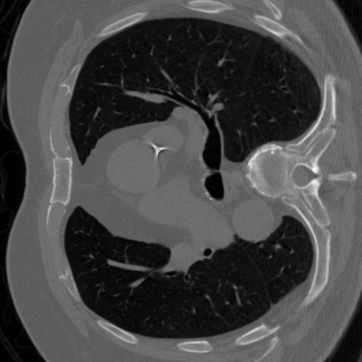}};
 \spy on (1.45,.98) in node [right] at (-0.01,-.95);
\end{tikzpicture}
\caption*{GT}
\end{subfigure}%
\hfill
\begin{subfigure}[t]{.195\textwidth}
\begin{tikzpicture}[spy using outlines={rectangle,white,magnification=5,size=1.89cm, connect spies}]
\node[anchor=south west,inner sep=0]  at (0,0) {\includegraphics[width=\linewidth]{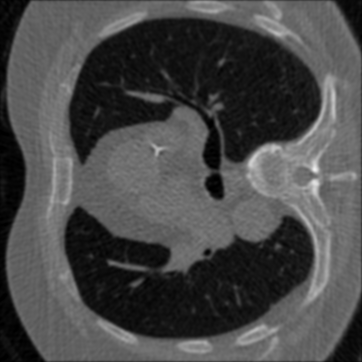}};
 \spy on (1.45,.98) in node [right] at (-0.01,-.95);
\end{tikzpicture}
  \caption*{FBP}
\end{subfigure}%
\hfill
\begin{subfigure}[t]{.195\textwidth}
\begin{tikzpicture}[spy using outlines={rectangle,white,magnification=5,size=1.89cm, connect spies}]
\node[anchor=south west,inner sep=0]  at (0,0) {\includegraphics[width=\linewidth]{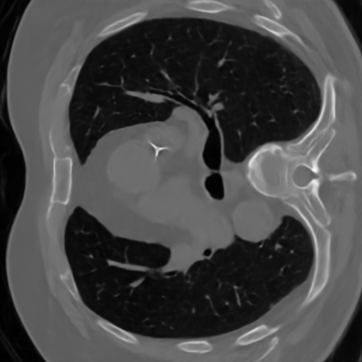}};
 \spy on (1.45,.98) in node [right] at (-0.01,-.95);
\end{tikzpicture}
  \caption*{Mean}
\end{subfigure}%
\hfill
\begin{subfigure}[t]{.195\textwidth}
\begin{tikzpicture}[spy using outlines={rectangle,white,magnification=5,size=1.89cm, connect spies}]
\node[anchor=south west,inner sep=0]  at (0,0) {\includegraphics[width=\linewidth]{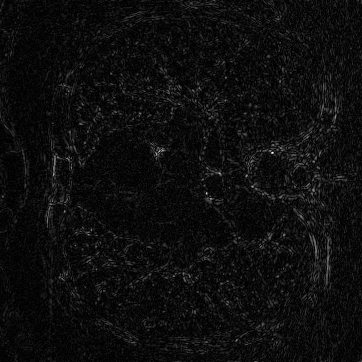}};
 \spy on (1.45,.98) in node [right] at (-0.01,-.95);
\end{tikzpicture}
  \caption*{Error}
\end{subfigure}%
\hfill
\begin{subfigure}[t]{.195\textwidth}
\begin{tikzpicture}[spy using outlines={rectangle,white,magnification=5,size=1.89cm, connect spies}]
\node[anchor=south west,inner sep=0]  at (0,0) {\includegraphics[width=\linewidth]{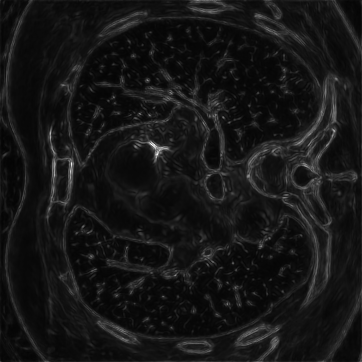}};
 \spy on (1.45,.98) in node [right] at (-0.01,-.95);
\end{tikzpicture}
  \caption*{Std}
\end{subfigure}%
\end{minipage}
\begin{minipage}[t]{.35\textwidth}
\vspace{-2.9cm}
\scalebox{.7}{
\begin{tabular}[t]{c|ccc} 
              & PSNR   & SSIM \\
\hline
WPPFlow        & \multirow{2}{*}{31.16}   & \multirow{2}{*}{0.774}          \\ 
\citep{AH2022}                 \\ 
\hdashline

Cond. Normalizing Flow         & \multirow{2}{*}{35.07}   & \multirow{2}{*}{0.831}          \\ 
\citep{DSLM2021}                 \\ 
\hdashline
Cond. MMD Flow         & \multirow{2}{*}{35.37}   &  \multirow{2}{*}{0.835}      \\ 
(ours) \\

\end{tabular}} 
\end{minipage}}
\caption{Generated mean image, error towards ground truth and pixel-wise standard deviation for low
dose computed tomography using conditional MMD flows.} \label{fig:CT_reco}
\end{figure}

\section{Conclusion}\label{sec:conclusions}
We introduced conditional MMD flows with negative distance kernel and applied them for posterior sampling in inverse problems by approximating the joint distribution. 
To prove stability of our model, we bounded the expected approximation error of the posterior distribution by the error of the joint distribution. 
We represented our conditional MMD flows as Wasserstein gradient flows, which also provides additional insights for the recent paper by \citet{DLPYL2023}.
Finally, we applied our algorithm to conditional image generation, inpainting, superresolution and CT. 
From a theoretical viewpoint it would be interesting to check whether the dimension scaling in Theorem~\ref{thm:fundamental} is optimal. 
In this paper we focused on the negative distance kernel because it can be computed efficiently via slicing and sorting. 
It would be interesting if similar results hold for other kernels. 
Moreover, so far we only considered discrete gradient flows with a fixed number $N$ of particles.
The convergence properties of these flows in the mean-field limit $N\to\infty$ is only partially answered in the literature, see e.g.~\citep{CDEFS2020} for the one-dimensional setting or \citep{arbel19} for a result with smooth kernels.
From a numerical perspective, it could be beneficial to consider more general slicing procedures, see e.g.~\citep{KNSBR2019,NH2023}, or other kernels, see \cite{H2024}.

\textbf{Limitations.}
It is the aim of our numerical examples to demonstrate that our method can be used for highly ill-posed and high-dimensional imaging inverse problems. In particular, we do \emph{not} claim that our computed tomography experiments are realistic for clinical applications.
In practice, the availability and potential biases of high-quality datasets are critical bottlenecks in medical imaging.
Moreover, even slight changes in the forward operator or noise model require that the whole model is retrained, which is computational costly and demands a corresponding dataset. Since the particles are interacting, an important next step would be to enable batching to train our model. 
Finally, we see our work as mainly theoretical but also provide evidence 
for its scalability
to high-dimensional and complicated inverse problems. 
The precise evaluation of posterior sampling algorithms in high dimensions is very hard due to the lack of meaningful quality metrics. 

\section*{Acknowledgments}

P.H. acknowledges funding by the German Research Foundation (DFG) within the project
SPP 2298 "Theoretical Foundations of Deep Learning", J.H. within the project STE
571/16-1 and by the EPSRC programme grant ``The Mathematics of Deep Learning'' with reference EP/V026259/1, F.A. by the DFG under Germany‘s Excellence Strategy – The Berlin
Mathematics Research Center MATH+ (project  AA5-6), R.B. and J.C. within the BMBF project 'VI-Screen' (13N15754) and G.S. acknowledges support by the BMBF Project ``Sale'' (01|S20053B).

The material data from Section~\ref{sec:exp} has been acquired in the frame of the EU Horizon 2020 Marie
Sklodowska-Curie Actions Innovative Training Network MUMMERING (MUltiscale, Multimodal and Multidimensional imaging for EngineeRING, Grant Number 765604) at the beamline TOMCAT of the SLS by A Saadaldin, D Bernard, and F Marone Welford. We acknowledge the Paul Scherrer Institut, Villigen, Switzerland for provision of synchrotron radiation
beamtime at the TOMCAT beamline X02DA of the SLS.

We would like to thank Chao Du for providing us class-conditional samples of his method for a quantitative comparison in Figure~\ref{fig:classcond_samples}.

\bibliographystyle{iclr2024_conference}
\bibliography{references}

\newpage
\appendix

\section{Supplement to Section~\ref{sec:conditional_post}}

\subsection{Proof of Proposition~\ref{lem:fundamentalthm_of_conditional_generative_modelling}}
\label{proof:fundamental_conditional_generative}

By the definition of the conditional probability as the disintegration of the joint measure, we have to prove that
$
P_Y\times k = P_{Y,X},
$
 with 
$
k(y,\cdot)=T(\cdot,y)_\#P_Z.
$
To this end, let $A\in\mathcal B(\R^n\times \R^d)$ be Borel measurable.
Then it holds
\begin{align}
(P_Y\times k)(A)&=\int_{\R^n} \int_{\R^d} 1_A(x,y) \d T(\cdot,y)_\#P_Z(x)\d P_Y(y)\\
&=\int_{\R^n} \int_{\R^d} 1_A(T(z,y),y)\d P_Z(z)\d P_Y(y)
\end{align}
Since $Y$ and $Z$ are independent by assumption, this is equal to
\begin{align}
&\int_{\R^n\times \R^d} \int_{\R^d} 1_A(T(z,y),y)\d P_{Y,Z}(y,z)\\
&=\int_{\R^n\times \R^d} \int_{\R^d} 1_A(x,y)\d P_{Y,T(Z,Y)}(y,x)=P_{Y,T(Z,Y)}(A)=P_{Y,X}(A),
\end{align}
where the first equality follows by the transformation formula and the last equality follows by assumption.
Since this holds for all $A\in\mathcal B(\R^n\times \R^d)$, this completes the proof.\hfill$\Box$

%------------------------------------------------------------------------------------
\subsection{Supplement to Subsection \ref{post_joint}}\label{proof:fundamental}
%------------------------------------------------------------------------------------
The proof of Theorem \ref{thm:fundamental}
requires some preliminaries on RKHS which can be found more detailed, e.g. in \citet{SC2008}.
The negative distance kernel $K(x,y) = -\|x-y\|$ is a conditionally positive definite function meaning that for every
$N \in \N$ and $a_i \in \R$ with $\sum _{i=1}^N a_i = 0$ and $x_i \in \R^d$, $i=1,\ldots,N$, it holds
\begin{equation} \label{spd}
\sum_{i=1}^N a_i a_j K(x_i,x_j)  \ge 0,
\end{equation}
where equality is only possible if $a_i = 0$ for all $i=1,\ldots,N$.
The associated kernel
$$
\tilde K(x,y) \coloneqq - \|x-y\| + \|x\| + \|y\|
$$ 
is positive definite, i.e.,\ $\tilde K$ fulfills \eqref{spd} without the sum constraint on the $a_i$.
For any $\mu,\nu \in \mathcal P_1(\R^d)$, we have
$$
\mathcal D_{\tilde K} (\mu,\nu) = \mathcal D_{K} (\mu,\nu).
$$
A space of functions
$\mathcal H(\R^d)$ mapping from $\R^d$ to $\R$ 
with the property that point evaluations 
$f \mapsto f(x)$  are continuous for all $f \in \mathcal H(\R^d)$
is called \emph{reproducing kernel Hilbert space} (RKHS). It is a Hilbert space with an inner product $\langle \cdot,\cdot \rangle_{\mathcal H}$ and associated norm.
For symmetric, positive definite functions and in particular for $\tilde K$, 
there exists a unique RKHS such that the \emph{reproducing kernel property} 
\begin{equation}
f(x) = \langle f, \tilde K(x,\cdot) \rangle_{\mathcal H} \quad \text{for all } f \in \mathcal H(\R^d)
\end{equation}
holds true. We denote this RKHS by $\mathcal H_{\tilde K}(\R^d)$. 
Further,  there is an injective mapping from $\mathcal P_\frac12(\R^d)$ to $\mathcal H_{\tilde K}(\R^d)$, called \emph{kernel mean embedding}
defined by
\begin{equation}\label{kme}
\hat \mu (x) \coloneqq \langle \tilde K(x,\cdot) , \mu \rangle = \int_{\R^d} \tilde K(x,y) \, \d \mu(y),
\end{equation}
see, e.g., \citet{MD2023}.
Note that this has nothing to do with the characteristic function of $\mu$. Since we do not address the later one in this paper, there is no notation mismatch.
The kernel mean embedding is not surjective, see \citet{SZ2021}.
We have for $h \in \mathcal H_{\tilde K}(\R^d)$ and $\mu \in \mathcal P_\frac12 (\R^d)$  the representation
    \begin{align} 
        \langle h, \mu \rangle = \int_{\R^d} h(x) \, \d {\mu}(x)
        &= \int_{\R^d} \langle h, \tilde K(x, \cdot) \rangle_{\mathcal H_{\tilde K}} \, \d {\mu}(x)\\
        &= \Big\langle h, \int_{\R^d} \tilde K(x, \cdot) \, \d {\mu}(x) \Big\rangle_{\mathcal H_{\tilde K}}  
        = \langle h, \hat {\mu} \rangle_{\mathcal H_{\tilde K}}.\label{eq:repr}
    \end{align}
This implies together with the dual representation of the MMD \citep{NW2010} for $\mu,\nu \in \mathcal P_\frac12 (\R^d)$ that
\begin{align} 
\mathcal D_{\tilde K} (\mu,\nu) 
&= 
\sup_{ \|h\|_{\mathcal H_{\tilde K}} \le 1} \langle h, \mu - \nu \rangle
= 
\sup_{\|h\|_{\mathcal H_{\tilde K}} \le 1} \langle h, \hat \mu - \hat \nu\rangle_{\mathcal H_{\tilde K}} 
= 
\|\hat \mu - \hat \nu\|_{\mathcal H_{\tilde K}}. \label{dual-mmd}
\end{align}

Equipped with the Wasserstein-$p$ distance defined by
\begin{equation}
\mathcal W_p(\mu,\nu) \coloneqq
\left\{
\begin{array}{ll}
\bigl(\inf_{\pi \in \Pi(\mu,\nu)} \int_{\mathbb R^d \times \mathbb R^d} \|x-y\|^p \, \d \pi(x,y) \bigr)^{\frac1p}&\text{ if } p\in [1,\infty),\\[0.5ex]
\inf_{\pi \in \Pi(\mu,\nu)} \int_{\mathbb R^d \times \mathbb R^d} \|x-y\|^p \, \d \pi(x,y) &\text{ if } p\in (0,1),
\end{array}
\right.
\label{eq:Wp}
\end{equation}
where $\Pi(\mu,\nu)$ denotes the set of measures with marginals $\mu$ and $\nu$, the space
$\mathcal{P}_p(\R^d)$ becomes a complete and separable metric space, see \citet{MD2023,Vil03}.
We will only need $p \in \{\frac12, 1,2\}$. 
Note that by Jensen's inequality,
$$
W_{\frac12}^2(\mu,\nu) \le W_{1}(\mu,\nu) \le  W_{2}(\mu,\nu),$$
i.e., convergence in $\mathcal P_p$ is stronger for larger values of $p$.
Finally, we need the dual representation of the Wasserstein-$\frac12$ distances 
\begin{equation} \label{dual-wasserstein}
 W_{\frac12} (\mu,\nu) = \sup_{|h|_{\mathcal C^{\frac12}} \le 1} \langle h, \mu - \nu \rangle,
\end{equation}
where $\mathcal C^{\frac12}(\R^d)$ denotes the space of $\frac12$-H\"older continuous functions together with the seminorm
$|f|_{\mathcal C^{\frac12}} \coloneqq \sup_{x \not = y} {(f(x) - f(y))}/{\|x-y\|}$,
see \citet{MD2023,Vil03}. For compactly supported measures, a relation between
the MMD with the negative distance kernel and the Wasserstein distance was proven in
\citet{HWAH2023}.

\begin{lemma} \label{W_MMD}
Let $K(x,y) \coloneqq - \|x-y\|$.
For $\mu, \nu \in \mathcal P_1(\mathbb{R}^d)$ it holds 
\begin{align}
    2 \, \mathcal{D}_{K}^2(\mu, \nu) 
		\leq  
		 W_{1}(\mu, \nu).
\end{align}
If $\mu$ and $\nu$ are additionally supported on the ball $B_R(0)$, then
there exists a constant $C_d>0$ such that
\begin{align}
    		 W_{1}(\mu, \nu) 
		\leq  
		C_d R^{\frac{2d+1}{2d+2}} 
		\mathcal{D}_{K}(\mu, \nu)^{\frac{1}{d+1}} .
\end{align} 
\end{lemma}

Now we can prove Theorem \ref{thm:fundamental}.

\paragraph{Proof of Theorem \ref{thm:fundamental}}
1. For any $y \in \mathbb R^n$, we consider the difference of the kernel mean embedding functions
$
f(\cdot,y) \coloneqq \hat{P}_{\tilde X|Y=y}-\hat{P}_{{ X}|Y=y}
$.
By \eqref{dual-mmd} and \eqref{eq:repr}, we obtain
\begin{align*}
     \mathcal{D}^2_K(P_{\tilde X|Y=y},P_{X|Y=y}) 
				&= \langle  f(\cdot,y),\hat{P}_{\tilde X|Y=y}-\hat{P}_{ X|Y=y} \rangle_{\mathcal H_{\tilde K}} 
				= \langle  f(\cdot,y),P_{\tilde X|Y=y}-P_{X|Y=y} \rangle
		\end{align*}
and further
\begin{align}\label{calc_bayes}
\mathbb{E}_{y \sim P_Y} \left[\mathcal{D}^2_K(P_{\tilde X|Y=y},P_{X|Y=y}) \right] 
&= 
\int_{\mathbb R^n}\langle  f(\cdot,y),P_{\tilde X|Y=y}-P_{X|Y=y} \rangle \, \d \P_Y(y)\\
& = 
\int_{\mathbb R^n} \int_{\mathbb R^d} f(x,y) (p_{\tilde X|Y = y}-p_{X|Y=y})(x) p_{Y}(y) \, \d x \d y .
\end{align}
Applying the definition of the posterior density $p_{X|Y = y} = \frac{p_{X,Y}}{p_Y}$, this can be rewritten as
\begin{align}\label{calc_bayes_1}
\mathbb{E}_{y \sim P_Y} \left[\mathcal{D}^2_K(P_{\tilde X|Y=y},P_{X|Y=y}) \right] 
&=  
\int_{\mathbb R^n} \int_{\mathbb R^d} f(x,y) (p_{\tilde X,Y}-p_{X,Y})(x,y) \, \d x \d y .
\end{align}
2.
Next, we show that the function $f(x,y)$ is $\frac12$-H\"older continuous with respect to both arguments.
First, we conclude by assumption on the $\frac12$-H\"older continuity of the posteriors with respect to $y$ that
\begin{align*}
    f(t,y_1) - f(t,y_2) 
		&= \int_{\R^d} \tilde K(x,t) \, \d {P}_{\tilde X|Y=y_1}(x) -  
		\int_{\R^d} \tilde K(x,t) \, \d{P}_{X|Y=y_1}(x) \\
		&
		\quad -\int_{\R^d} \tilde K(x,t) \, \d{P}_{\tilde X|Y=y_2}(x) +  \int_{\R^d} \tilde K(x,t) \, \d{P}_{X|Y=y_2}(x)     
    \\
		& = \int_{\R^d} \tilde K(x,t) (p_{\tilde X|Y = y_1}-p_{\tilde X|Y=y_2})(x) \, \d x \\
		&		
		\quad + \int_{\R^d} \tilde K(x,t) (p_{X|Y = y_2}-p_{X|Y=y_1})(x) \, \d x
    \\
		&\leq 2 \, C_{S_n} \int_{S_d} \tilde K(x,t) \, \d x \|y_1-y_2\|^{\frac{1}{2}} 
		\leq 2 \, C_{S_n} \, C_{S_d} \|y_1-y_2\|^{\frac{1}{2}},
\end{align*}
where $C_{S_d} \coloneqq  \max_{t \in S_d} \int_{S_d} \tilde K(x,t) \, \d x$.
Interchanging the role of $y_1$ and $y_2$ yields 
\begin{align}
    |f(t,y_1) - f(t,y_2)| \le  C_{S} \|y_1-y_2\|^{\frac{1}{2}}, \quad C_S :=  2 \, C_{S_n} \, C_{S_d}
\end{align}
for all $t \in S_d$ and all $y_1,y_2 \in S_n$. 
Now the triangle inequality 
and the relation
$
|h(x)-h(y)| 
=  
|\langle h, \tilde{K}(x,\cdot)-\tilde{K}(y,\cdot) \rangle_{\mathcal H_{\tilde K}} |
\leq 
\|h\|_{\mathcal H_{\tilde K}} \|\tilde{K}(x,\cdot)-\tilde{K}(y,\cdot)\|_{\mathcal H_{\tilde K}} 
$ for all $h \in \mathcal H_{\tilde K}$,
implies for $x_1,x_2 \in S_d$ and  $y_1,y_2 \in S_n$ that
\begin{align}\label{inter}
\vert f(x_1,y_1) - f(x_2,y_2) \vert 
&\le
\vert f(x_1,y_1) - f(x_1,y_2) \vert + \vert f(x_1,y_2) - f(x_2,y_2) \vert\\
&\le C_{S_d} \|y_1-y_2\|^{\frac{1}{2}} + \| f(\cdot,y_2)\|_{\mathcal H_{\tilde K}}
\|\tilde{K}(x_1,\cdot)-\tilde{K}(x_2,\cdot)\|_{\mathcal H_{\tilde K}}   .
\end{align}
By the reproducing kernel property and definition of $\tilde K$ we have
$$
\|\tilde{K}(x_1,\cdot)-\tilde{K}(x_2,\cdot)\|_{\mathcal H_{\tilde K}}^2 = \tilde{K}(x_1,x_1) + \tilde{K}(x_2,x_2) - 2 \tilde{K}(x_1,x_2)
= 2 \|x_1-x_2\|,
$$
so that 
\begin{align}\label{inter_1}
\vert f(x_1,y_1) - f(x_2,y_2)  \vert
&\le
C_{S_d} \|y_1-y_2\|^{\frac{1}{2}} + \sqrt{2} \| f(\cdot,y_2)\|_{\mathcal H_{\tilde K}}  \|x_1-x_2\|^{\frac{1}{2}}\\
&\le
\frac{2}{2^\frac14}\max \left( C_{S_d} , \sqrt{2}\| f(\cdot,y_2)\|_{\mathcal H_{\tilde K}} \right) \Vert (x_1,y_1) -(x_2,y_2)\Vert^{\frac12}.
\end{align}
Finally, we obtain
\begin{align*}
    \|f(\cdot,y_2)\|_{\mathcal H_{\tilde K}} &=  \mathcal{D}_K(P_{\tilde X|Y=y_2},P_{X|Y=y_2}) 
    = \sup_{\|h\|_{\mathcal H_{\tilde K}} \leq 1} \int_{\R^d} h \, \d (P_{\tilde X|Y=y_2}-P_{X|Y=y_2}) 
    \\
		&\leq \sup_{\|h\|_{\mathcal H_{\tilde K}} \leq 1} \Big( |\int_{\R^d} h \, \d P_{\tilde X|Y=y_2}| 
		+ |\int_{\R^d} h \, \d P_{X|Y=y_2}| \Big)
		\leq 
		2 \, \sup_{\|h\|_{\mathcal H_{\tilde K}} \leq 1}  \|h\|_{\infty} 
    \\
		&\leq 2 \tilde{C}_{S_d} 
\end{align*}
where $\|h\|_{\infty} \coloneqq \max_{x \in S_d} |h(x)|$ and
the last inequality follows by 
$
|h(x)| = |\langle h, \tilde{K}(x,\cdot) \rangle_{\mathcal H_{\tilde K}}  |
\leq 
\|h\|_{\mathcal H_{\tilde K}} \|\tilde{K}(x,\cdot)\|_{\mathcal H_{\tilde K}} 
=
\sqrt{2} \|h\|_{\mathcal H_{\tilde K}} \|x\|^\frac12 
\le
\tilde{C}_{S_d}\|h\|_{\infty}
$
for all $x \in S_d$. In summary, we proved for all $x_1,x_2 \in S_d$ ad $y_1,y_2 \in S_n$ that
\begin{align}\label{inter_all}
\lvert f(x_1,y_1) - f(x_2,y_2) \rvert 
&\le
\tilde C \, \Vert (x_1,y_1) -(x_2,y_2)\Vert^{\frac12},
\end{align}
where $\tilde C \coloneqq \frac{2}{2^\frac14}\max \left( C_{S_d} , 2 \sqrt{2} \tilde{C}_{S_d} \right)$. Hence, $\frac{1}{\tilde C} f$ is $\frac12$-H\"older continuous with 
$|h|_{\mathcal C^{\frac12}} = 1$.

3. Using Jensen's inequality as well as 
\eqref{dual-wasserstein} and
\eqref{calc_bayes_1}, we conclude 
\begin{align}\label{wasser12}
\mathbb{E}_{y \sim P_Y} \left[\mathcal{D}_K(P_{\tilde X|Y=y},P_{X|Y=y}) \right] 
&\le
\left( \mathbb{E}_{y \sim P_Y} \left[\mathcal{D}^2_K(P_{\tilde X|Y=y},P_{X|Y=y}) \right] \right)^\frac12\\
&\le \tilde C^\frac12 \mathcal W_{\frac12} (P_{\tilde X,Y}, P_{X,Y})^\frac12  \le \tilde C^\frac12 \mathcal W_{1} (P_{\tilde X,Y}, P_{X,Y})^\frac14. 
\end{align}  
Finally, we apply Lemma \ref{W_MMD} to get
\begin{align}\label{est}
\mathbb{E}_{y \sim P_Y}[\mathcal{D}_K(P_{\tilde X|Y=y},P_{X|Y=y})] 
&\leq 
C \ \mathcal{D}_K
(P_{\tilde X,Y}, P_{X,Y})^{ \frac{1}{4 (d+n+1)} }
\end{align}
with appropriate constant $C > 0$.
This finishes the proof. \hfill $\Box$
\medskip

The assumption that the random vectors $X$, $\tilde X$, and $Y$
map onto compact sets cannot be neglected
as the following counterexample shows.

\begin{example}[Theorem \ref{thm:fundamental} fails for non-compactly supported measures]   \label{cex:comp-range}
    For fixed $m \in N$ and $0 < \epsilon < 6^{-1}$,
    we consider the random variables $X$, $\tilde X$, and $Y$ on $\R$
    with joint distributions 
    $P_{X,Y} \coloneqq 
        \frac{1}{2} \, \mathcal U_{Q_{\epsilon}(0,0)} 
        + \frac{1}{2} \, \mathcal U_{Q_\epsilon (m,1)}$
    and
    $P_{\tilde X,Y} \coloneqq 
        \frac{1}{2} \, \mathcal U_{Q_\epsilon (m,0)}
        + \frac{1}{2} \, \mathcal U_{Q_\epsilon (0,1)}$,
    where $\mathcal U_\bullet$ denotes the uniform distribution 
    on the indicated set 
    and 
    $Q_\epsilon(x,y)$ the $\epsilon$-ball around $(x,y)$
    with respect to the $\infty$-norm.
    The assumptions on the conditional densities
    in Theorem~\ref{thm:fundamental}
    are fulfilled with $C_{S_1} \coloneqq (1-2\epsilon)^{-1/2}$.
    Using Lemma~\ref{W_MMD}, we obtain
    $\mathcal{D}_K(P_{\tilde X,Y}, P_{X,Y}) \leq 2^{-\frac12}$
    for all $m \in \N$.
    Furthermore,     we have 
    \begin{align*}
        \mathbb{E}_{y \sim P_Y}[\mathcal{D}_{K}(P_{\tilde X|Y=y}, P_{X|Y=y})]
        = \frac{1}{2} \, \mathcal{D}_{K}(\mathcal U_{Q_\epsilon(0)},\mathcal U_{Q_\epsilon(m)}) 
        +  \frac{1}{2} \, \mathcal{D}_{K}(\mathcal U_{Q_\epsilon(m)},\mathcal U_{Q_\epsilon(0)})
        \ge \sqrt{m - 6 \epsilon}.
    \end{align*}
    Consequently,
    there cannot exist a constant $C$ 
    such that \eqref{result} holds true
    for random vectors with arbitrary non compact range. 
\end{example}

Estimates similar to \eqref{result} 
can also be established for other divergences.

\begin{remark}[Relation between joint and conditioned distributions for different distances]    \label{rem:non-comp}
i) For the \emph{Kullback--Leibler} (KL) \emph{divergence},
    the chain rule \citep[Thm~2.5.3]{cover_inf} implies
    \begin{align*}
    \mathbb{E}_{y \sim P_Y}[\mathrm{KL}(P_{\tilde X|Y=y},P_{X|Y=y})] = \mathrm{KL}(P_{\tilde X,Y},P_{X,Y}).
    \end{align*}
ii)     Using Lemma~\ref{W_MMD} and Jensen's inequality,
    the estimate \eqref{result} can be transferred 
    to the Wasserstein-1 distance.
    More precisely,
    under the assumptions in Theorem~\ref{thm:fundamental},
    we have
    \begin{align*}
        &\mathbb{E}_{y \sim P_Y}[W_1(P_{\tilde X|Y=y}, P_{X|Y=y})]  
        \leq 
        C' \, \mathbb{E}_{y \sim P_Y}[\mathcal{D}_{K}(P_{\tilde X|Y=y}, P_{ X|Y=y})^{\frac{1}{d+1}} ] 
        \\
        &\leq C'' \, \mathcal{D}_{K}(P_{\tilde X|Y=y}, P_{ X|Y=y})^{\frac{1}{4(d+n+1)(d+1)}} 
        \leq C \,  W_{1} (P_{\tilde X,Y}, P_{X,Y})^{\frac{1}{8(d+n+1)(d+1)}} 
    \end{align*}
    with constants $C,C',C'' > 0$.
\end{remark}

Next, we show that the relations \eqref{stab1} and   \eqref{stab2} hold true under certain assumptions.
We need an auxiliary lemma which was proven
for more general kernels by \citet[Thm~3.2]{BHKMS2023}.

\begin{lemma} \label{stab_push_gen}
Consider $K(x,y) \coloneqq - \lVert x-y \rVert$, and 
let $S_d \subset \R^d$ be a compact set,
and $Z \in S_d$ be a random variable. 
For $F, G \in L_{P_Z}^{2}\left(S_d, \mathbb R^d \right)$
it holds 
$$
\mathcal D_{K}\left(F_{\#} P_Z, G_{\#} P_Z\right) \leq 
\sqrt{2}\ \mathbb{E}_{z \sim P_Z}\left[ \Vert F(z) - G(z)  \Vert \right]^{\frac12}.
$$
\end{lemma}

Now we can prove the H\"older estimates. 

\begin{lemma}[Stability under Pushforward] \label{stabpush}
Consider $K(x,y) \coloneqq - \lVert x-y \rVert$, and 
let $S_n \subset \R^n$ and $S_d \subset \R^d$ be compact sets,
and $Z \in S_d$ be a random variable. 
Further,
let $T\colon \mathbb R^d \times S_n \to \mathbb R^d$
be measurable.
If the derivatives are uniform bounded by 
$\sup _{y \in S_n} \left\|\nabla_{y} T(z, y)\right\| \leq C$ for all $z \in S_d$,
then
$$
\mathcal{D}_{K}\left(T\left(\cdot,y_{1}\right)_{\#} P_{Z}, T\left(\cdot,y_{2}, \right)_{\#} P_{Z}\right) 
\leq 
\sqrt{2 C}\ \|y_1-y_2\|^{\frac{1}{2}}
\quad \text{for all} \quad y_1, y_2 \in S_n.
$$
\end{lemma}

\begin{proof}
For all  $y_1, y_2 \in S_n$,
Lemma \ref{stab_push_gen} yields 
$$
\mathcal{D}_{K}\left(T\left( \cdot,y_{1}\right)_{\#} P_{Z}, T\left(\cdot,y_{2} \right)_{\#} P_{Z}\right) 
\leq 
\sqrt{2}\ \mathbb{E}_{z \sim P_Z}\left[ \Vert T(z,y_1) - T(z,y_2)  \Vert \right]^{\frac12}.
$$
Further, 
the second fundamental theorem of calculus implies 
\begin{align}
\left\|T\left( z, y_{1}\right)-T\left( z, y_2\right)\right\| & =\left\|\int_{0}^{1} \nabla_{y} T\left( z,y_{1}+t\left(y_{2}-y_{1}\right)\right)\left(y_{1}-y_{2}\right) \mathrm{d} t\right\| \\
& \leq \int_{0}^{1}\left\|\nabla_{y} T\left(z,y_{1}+t\left(y_{2}-y_{1}\right)\right)\right\| \mathrm{d} t\left\|y_{1}-y_{2}\right\| 
 \leq C\left\|y_{1}-y_{2}\right\|. 
\end{align}
for all $z \in S_d$.
Applying this estimate in the above inequality yields the assertion.  
\end{proof}

For the stability with respect to the posteriors,
there exist sophisticated strategies to obtain optimal bounds, see \cite{garbunoinigo2023bayesian}. 
However, for our setting, we can just apply a result from \cite{altekruger2023conditional,sprungk2020local} on the Wasserstein-1 distance.

\begin{lemma}[Stability of Posteriors]\label{stabpost}
Consider $K(x,y) \coloneqq - \lVert x-y \rVert$.
Let $S_n \subset \R^n$ and $S_d \subset \R^d$ be compact sets,
and $X \in S_d$ and $Y \in S_n$ be random variables.
Assume that there exists a constant $M$ such that for all $y_{1}, y_{2} \in S_n$ and for all $x \in S_d$ it holds
\begin{align*}
\left|\log p_{Y \mid X=x}\left(y_{1}\right)-\log p_{Y \mid X=x}\left(y_{2}\right)\right| \leq M\left\|y_{1}-y_{2}\right\|.
\end{align*}
Then, there exists a constant $C>0$, such that for all $y_{1}, y_{2} \in S_n$ we have
$$\mathcal{D}_K \big(P_{X \mid Y=y_1},P_{X \mid Y=y_2}\big) \leq C \|y_1-y_2\|^{\frac12}.$$
\end{lemma}

\begin{proof}
By Lemma \ref{W_MMD}, we know that
$D_K(\mu,\nu) \leq C W_1(\mu,\nu)^{\frac12}$.
Now we can apply \citep[Lem.~3]{altekruger2023conditional} together with the fact that local Lipschitz continuity on compact sets implies just Lipschitz continuity.
\end{proof}

\paragraph{Proof of Theorem \ref{cor_pointwise}}
For
$\delta \coloneqq \mathbb{E}_{y\sim P_Y} [\mathcal{D}_K(T^\varepsilon (\cdot, y)_{\#}P_Z, P_{X|Y=y})]$,
Theorem~\ref{thm:fundamental} implies
$\delta \leq  C \,  \smash{\varepsilon^{\frac{1}{4(d+n+1)}}}$.
Adapting the lines of the proof of \citep[Thm.~5]{altekruger2023conditional}
with respect to MMD instead of the Wasserstein-1 distance,
we obtain
$$
\mathcal{D}_K(T^\varepsilon(\cdot,  y)_{\#}P_Z, P_{X|Y= y}) 
\leq \tilde D \, \delta^{\frac{1}{2(n+1)}} 
\le D \, \varepsilon^{\frac{1}{8(d+n+1)(n+1)}}
$$
for arbitrary $y \in S_n$,
where the constants $D,\tilde D>0$ depend on the dimension $n$,
the value $p_Y(y)$, 
the diameter of $S_n$,
the constants from  \eqref{eq:HC-X}, \eqref{eq:HC-tildeX}, \eqref{stab1}, and \eqref{stab2}
as well as on $C'$ from the assumptions.
Finally, taking $\varepsilon \rightarrow 0$ yields the assertion.
\hfill $\Box$

\subsection{Supplement to Subsection~\ref{condMMD}}\label{proof:posterior_flow}

A curve $\gamma\colon I\to\P_2(\R^d)$ 
on the interval $I\subseteq \R$ is called \emph{absolutely continuous}
if there exists a Borel velocity field $v_t\colon\R^d\to\R^d$ with 
$\int_I \|v_t\|_{L_{2,\gamma(t)}} \d t<\infty$ 
such that the continuity equation 
\begin{equation}
    \label{eq:cont-eq}
    \partial_t\gamma(t)+\nabla\cdot(v_t\gamma(t))=0
\end{equation}
is fulfilled on $I\times\R^d$ in a weak sense, see \citep[Thm.~8.3.1]{AGS2005}.
A locally absolutely continuous curve $\gamma\colon(0,\infty)\to\P_2(\R^d)$ 
with velocity field $v_t\in \T_{\gamma (t)}\P_2(\R^d)$ 
is a \emph{Wasserstein gradient flow} with respect to a functional $\smash{\F\colon\P_2(\R^d)\to(-\infty,\infty]}$ 
if 
\begin{equation}\label{eq:gf_condition}
v_t\in -\partial \F(\gamma(t))\quad \text{for a.e. } t>0,
\end{equation}
where $\T_{\mu}\P_2(\R^d)$ denotes the \emph{regular tangent space}
at $\mu \in \mathcal P_2(\R^d)$,
see \citep[Def.~8.4.1]{AGS2005},
and $\partial \F(\mu)$ the \emph{reduced Fr\'echet subdiffential} 
at $\mu \in \mathcal P_2(\R^d)$ 
consisting of all $\xi \in L_2(\R^d, \R^d; \mu)$ satisfying
\begin{equation}\label{eq:subdiff}
    \F(\nu) - \F(\mu)
    \ge 
    \inf_{\pi \in \Gamma^{\opt}(\mu,\nu)}
    \int\limits_{\R^{d} \times \R^d}
    \langle \xi(x), y - x \rangle
    \, \d \pi (x, y)
    + o(W_2(\mu,\nu)) 
    \quad\text{for all} \; \nu \in \P_2(\R^d),
\end{equation}
see \citep[Eq. (10.3.13)]{AGS2005}.
Here the infimum is taken over the set $\Gamma^\opt(\mu,\nu)$ 
consisting of all optimal Wasserstein-2 transport plans between $\mu$ and $\nu$,
i.e., the minimizer of \eqref{eq:Wp} for $p=2$.

In order to prove Theorem~\ref{thm:posterior_flow}, we first prove that an absolutely continuous curve $\gamma$ within the metric space $(\mathcal P_2(\R^d\times\R^n),\mathcal W_2)$ does not transport mass within the second component, whenever the second marginal is a constant empirical measure.
Note that the $i$th marginal can be written as $(P_i)_\# \gamma$
using the projections $P_1(x,y) \coloneqq x$ and $P_2(x,y) \coloneqq y$.

\begin{theorem}\label{thm:no_velocity_y}
    Let $\gamma\colon I\to\P_2(\R^d\times \R^n)$ 
    be an absolutely continuous curve 
    with associate vector field 
    $v_t=(v_{t,1},v_{t,2})\colon \R^d\times\R^n\to\R^d\times \R^n$.
    If  
    $(P_2)_\#\gamma(t)$ 
    is a constant empirical measure $\frac1N\sum_{i=1}^N\delta_{q_i}$ independent of $t$, 
    then
    $v_{t,2}$ vanishes $\gamma(t)$-a.e.
    for almost every $t\in I$.
\end{theorem}

\begin{proof}
Denote by $\pi_t^{t+h} \in \mathcal P_2((\R^d \times \R^n) \times (\R^d \times \R^n))$ an arbitrary optimal Wasserstein-2 plan 
between $\gamma(t)$ and $\gamma(t+h)$,
and by $\tilde P_i \colon (\R^d \times \R^n)^2 \to (\R^d \times \R^n)$
the projections to the first and second component,
i.e.\ $\tilde P_i((x_1,y_1),(x_2,y_2)) \coloneqq (x_i,y_i)$ 
for $i \in \{1,2\}$.
For almost every $t \in I$,
the associate vector field $v_t$ of $\gamma$ satisfies
\begin{equation}
    \label{eq:det-vec-field}
    \lim_{h \to 0} 
    \bigl( \tilde P_1, \tfrac{1}{h} \, (\tilde P_2 - \tilde P_1) \bigr)_\# \pi_t^{t+h}
    =
    ( \Id, v_t)_\# \gamma(t),
\end{equation}
where the left-hand side converges narrowly, see \citep[Prop~8.4.6]{AGS2005}.
For $n \in \N$, 
let $f_n \colon [0, \infty) \to [0,1]$ be a continuous function
with $f_n(0) = 0$, 
$f_n(t) > 0$ for $t \in (0,n)$ and
$f_n(t) = 0$ for $t \ge n$.
Consider the integral
\begin{align}
    F_n
    &\coloneqq
    \int_{(\R^d \times \R^n)^2}
    f_n( \lVert y_2 \rVert)
    \, \d 
    \bigl( \tilde P_1, \tfrac{1}{h} \, (\tilde P_2 - \tilde P_1) \bigr)_\# \pi_t^{t+h}
    ((x_1, y_1) , (x_2, y_2))
    \\
    &=
    \int_{(\R^d \times \R^n)^2}
    f_n( \tfrac{1}{h} \lVert y_2 - y_1 \rVert)
    \, \d \pi_t^{t+h} ((x_1, y_1) , (x_2, y_2)).
    \label{eq:int-velo}
\end{align}
Since $(P_2)_\#\gamma(t)=\frac{1}{N}\sum_{i=1}^N\delta_{q_i}$
is a constant empirical measure,
every plan $\pi_t^{t+h}$ is supported on
\begin{equation*}
    \bigcup_{i,j=1}^N 
    (\R^d \times \{q_i\}) 
    \times 
    (\R^d \times \{q_j\}).
\end{equation*}
Hence, on the support of $\pi_t^{t+h}$,
the norm $\lVert y_2 - y_1 \rVert$ in \eqref{eq:int-velo}
becomes either zero
or is bounded by
\begin{equation*}
    \lVert y_2 - y_1 \rVert
    \ge
    \min\{ \lVert q_i - q_j \rVert : i \ne j\}
    \eqqcolon
    S.
\end{equation*}
Thus, for $h \le S/n$,
the integral $F_n$ vanishes
and the narrow convergence in \eqref{eq:det-vec-field} implies
\begin{equation*}
    v_{t,2}(x_1,y_1) \not\in (-n,0) \cup (0,n) 
\end{equation*}
for $\gamma(t)$-a.e. $(x_1,y_1) \in \R^d \times \R^n$.
Since $n \in \N$ is arbitrary,
we obtain the assertion.
\end{proof}

Interestingly, Theorem \ref{thm:no_velocity_y} is in general not true if the second marginal is not an
empirical measure as Example \ref{ex:fail} below shows.
We can now adapt \citep[Prop. D.1]{AHS2023} for proving  Theorem~\ref{thm:posterior_flow}. 

\paragraph{Proof of Theorem~\ref{thm:posterior_flow}}
Let $\xi \coloneqq(\xi_1,\dots,\xi_M)\in\R^{dN}$ satisfy $(\xi_i,q_i)\neq (\xi_j,q_j)$ for all $i\neq j$. 
Then,
there exists an $\epsilon>0$ such that
the optimal transport plan between $\frac{1}{N} \sum_{i=1}^N \delta_{(\xi_i,q_i)}$ and $\frac{1}{N} \sum_{i=1}^N \delta_{(\eta_i,q_i)}$ is given by $\pi\coloneqq\frac{1}{N}\sum_{i=1}^N\delta_{((\xi_i,q_i),(\eta_i,q_i))}$
for all $\eta\in\R^{dN}$ with $\|\xi-\eta\| \le \epsilon$.
In particular, it follows
\begin{equation}\label{eq:isometry_M}
    \mathcal W_2^2\biggl( \frac{1}{N} \sum_{i=1}^N \delta_{(\xi_i,q_i)}, \frac{1}{N} \sum_{i=1}^N \delta_{(\eta_i,q_i)} \biggr) 
    = \frac{1}{N} \sum_{i=1}^N \|(\xi_i,q_i) - (\eta_i,q_i)\|_2^2 = \frac{1}{N} \sum_{i=1}^N \|\xi_i - \eta_i\|_2^2.
\end{equation}

By an analogous argumentation as in \cite{AHS2023},
$\gamma_{N,q}$ is a locally absolutely continuous curve;
and by Theorem~\ref{thm:no_velocity_y}, 
the second component of 
the associate velocity field $v_t=(v_{t,1},v_{t,2})$ 
vanishes,
i.e.\ $v_{t,2} \equiv 0$ for almost every $t\in(0,\infty)$.
Exploiting \cite[Prop.~8.4.6]{AGS2005}, we obtain
\begin{align*}
    0
    &=\lim_{h\to0}
    \frac{\mathcal W_2^2(\gamma_{N,q}(t+h)),(\Id+hv_t)_\# \gamma_{N,q}(t))}{|h|^2}
    \\
    &=\lim_{h\to0}
    \frac{\mathcal W_2^2(\frac1N\sum_{i=1}^N\delta_{(u_i(t+h),q_i)},\frac1N\sum_{i=1}^N \delta_{(u_i(t)+hv_{t,1}(u_i(t),q_i),q_i)})}{|h|^2}
    \\
    &=\lim_{h\to0}
    \frac1N\sum_{i=1}^N\Bigl\|\frac{u_i(t+h)-u_i(t)}{h}-v_{t,1}(u_i(t),q_i)\Bigr\|^2=\frac1N\sum_{i=1}^N\|\dot u_i(t)-v_{t,1}(u_i(t),q_i)\|^2
\end{align*}
for a.e.\ $t\in(0,\infty)$, where the first equality in the last line follows from \eqref{eq:isometry_M}.
In particular, this implies $\dot u_i(t)=v_{t,1}(u_i(t),q_i)$ a.e.\ 
such that $N\nabla_x F_{(p,q)}(u(t),q)=(v_{t,1}(u_1(t),q_1),\dots ,v_{t,1}(u_N(t),q_N))$.

For fixed $t$, 
we now consider an $\epsilon$-ball around $\gamma_{N,q}(t)$
where the Wasserstein-2 optimal transport
between $\gamma_{N,q}(t)$ and a measure from $\mu \in P_{N,q}$
becomes unique 
as discussed in the beginning of the proof.
More precisely,
the unique plan between 
$\gamma_{N,q}(t)$ and 
$\smash{\mu \coloneqq \frac1N\sum_{i=1}^N\delta_{(\eta_i,q_i)}}$
with $\mathcal W_2(\mu,\gamma_{N,q}(t)) \le \epsilon$
is then given by
$\pi = \frac1N \sum_{i=1}^N \delta_{((u_i(t),q_i),(\eta_i,q_i))}$.
Since $u$ is a solution of \eqref{eq:posterior_ODE},
we obtain
\begin{align*}
    0
    &\leq 
    F_{(p,q)}((\eta,q))
    - F_{(p,q)}((u(t),q))
    + \langle \nabla_x F_{(p,q)}((u(t),q)), \eta - u(t) \rangle
    + o(\|\eta-u(t)\|)
    \\
    &=
    \mathcal J_{\nu_{N,q}}(\mu)
    - \mathcal J_{\nu_{N,q}}(\gamma_{N,q}(t))
    + \frac1N\sum_{i=1}^N \langle v_{t}(u_i(t),q_i),(\eta_i,q_i)-(u_i(t),q_i)\rangle
    + o(\mathcal W_2(\mu,\gamma_{N,q}(t)))\\
    &=
    \mathcal J_{\nu_{N,q}}(\mu)
    - \mathcal J_{\nu_{N,q}}(\gamma_{N,q}(t))
    + \int_{(\R^d\times\R^n)^2} \langle v_t(x_1,y_1),(x_2,y_2)-(x_1,y_1)\rangle\d  \pi((x_1,y_1),(x_2,y_2))
    \\
    &\quad
    + o(\mathcal W_2(\mu,\gamma_{N,q}(t))).
\end{align*}
Since $\pi$ is the unique plan in $\Gamma^\opt(\gamma_{N,q}(t), \mu)$, 
and since $\mathcal J_{\nu_{N,q}}(\mu)=\infty$ for $\mu\not\in P_{N,q}$,
we have $v_t\in-\partial\mathcal J_{\nu_{N,q}}(\gamma(t))$ showing the assertion.\hfill$\Box$

Finally, we construct an explicit example 
showing that 
the restriction to empirical second marginals
in Theorem~\ref{thm:no_velocity_y}
is inevitable.

\begin{example}[Theorem~\ref{thm:no_velocity_y} fails for arbitrary marginals]\label{ex:fail}
Let $f \colon I \to [0,1]$ be a continuously differentiable function
on the interval $I \subseteq \R$,
and consider the curve $\gamma_f \colon I \to \mathcal P_2(\R)$
given by
\begin{equation*}
    \gamma_f (t) 
    \coloneqq 
    (1 - f(t)) \, \delta_0 
    + f(t) \, \lambda_{[-1,0]}.
\end{equation*}
Figuratively,
$f$ controls
how the mass on the interval $[-1,0]$
flows into or out of the point measure
located at zero.
In order to show that
such curves are absolutely continuous,
we exploit the associate quantile functions.
More generally,
for $\mu \in \mathcal P_2(\R)$,
the quantile function is defined as
\begin{equation*}
    Q_\mu (s) \coloneqq \min \{ x \in \R : \mu((-\infty,x]) \ge s \},
    \quad
    s \in (0,1).
\end{equation*}
For our specific curve,
the quantile functions are piecewise linear 
and given by
\begin{equation*}
    Q_{\gamma_f(t)} (s)
    =
    \min \{ (f(t))^{-1} \, s - 1 , 0 \},
\end{equation*}
see for instance \citet[Prop~1]{HBGS2023}.
Due to the relation between
the quantile function 
and the Wasserstein distance \citep[Thm~2.18]{Vil03},
we obtain
\begin{align*}
   \mathcal W_2^2( \gamma_f(t_1), \gamma_f(t_2))
    &=
    \int_0^1 
    \lvert Q_{\gamma_f(t_1)} (s) - Q_{\gamma_f(t_2)} (s) \lvert^2
    \, \d s
    \\
    &\le
    \int_0^1
    \lvert (f(t_1))^{-1} \, s - (f(t_2))^{-1} \, s \rvert^2
    \, \d s
    =
    \tfrac{1}{3} \,
    \lvert (f(t_1))^{-1} - (f(t_2))^{-1} \rvert^2.
\end{align*}
If the derivative of $1/f$ is bounded by $M$ on $I$,
the mean value theorem yields
\begin{equation*}
    \mathcal W_2(\gamma_f(t_1), \gamma_f(t_2))
    \le
    \tfrac{M}{\sqrt 3} \,
    \lvert t_1 - t_2 \rvert,
\end{equation*}
such that $\gamma_f$ is indeed absolutely continuous.
Based on the curves $\gamma_f$ on the line,
we now consider the curve 
$\gamma : [1/4, 3/4] \to \mathcal P_2 (\R \times \R)$
given by
\begin{equation*}
    \gamma(t)  
    \coloneqq
    \frac12 \, \Bigl(
    (1-t) \, \delta_{(0,0)}
    + t \, \lambda_{\{0\} \times [-1,0]}
    + t \, \delta_{(2,0)}
    + (1 - t) \, \lambda_{\{2\} \times [-1,0]}
    \Bigr),
\end{equation*}
where $\lambda_A$ denotes the uniform measure on the set $A \subseteq \R^2$. 
Restricting the transport between $\gamma(t_1)$ and $\gamma(t_2)$
along the lines segments $\{0\} \times [-1, 0]$ 
and $\{2\} \times [-1, 0]$,
and using the above considerations,
where the derivatives are bounded by $M = 4$,
we obtain
\begin{equation*}
    \mathcal W_2(\gamma(t_1), \gamma(t_2)) 
    \le \tfrac{4}{\sqrt 3} \, \lvert t_1 - t_2 \rvert,
    \quad
    t_1, t_2 \in [1/4, 3/4].
\end{equation*}
Thus $\gamma$ is absolutely continuous too.
Furthermore,
both marginals 
$(\pi_1)_\# \gamma(t) = \frac12 (\delta_0 + \delta_2)$
and
$(\pi_2)_\# \gamma (t) = \frac12 (\delta_0 + \lambda_{[-1,0]})$
are independent of $t$.
If Theorem~\ref{thm:no_velocity_y} would hold true
for non-empirical marginals,
the associate vector field
$v_t \colon \R^2 \to \R^2$ has to be the zero everywhere
implying that
$\gamma$ is constant.
Since this is not the case,
we would obtain a contradiction.
\end{example}

\paragraph{Theoretical justification of a numerical observation in \cite{DLPYL2023}}
By \citet[Thm. 8.3.1]{AGS2005}, absolutely continuous curves $\gamma\colon I\to\P_2(\R^d)$ in $\mathcal P_2(\R^d)$ 
correspond to weak solutions of the continuity equation \eqref{eq:cont-eq}.
For the sliced Wasserstein gradient flows with target measure $\nu$, an analytic representation of $v_t$ was derived in \citet{bonnotte} as
\begin{equation}\label{eq:vt_sliced_w}
v_t(x)=\mathbb{E}_{\xi\in\mathbb{S}^{d-1}}[\psi_{t,\xi}'(P_\xi(x))\xi],\quad P_\xi(x)=\langle \xi,x\rangle,
\end{equation}
where $\psi_{t,\xi}$ is the Kantorovic potential between ${P_\xi}_\#\gamma(t)$ and ${P_\xi}_\#\nu$.
In the case that $\gamma(0)=\frac1N\sum_{i=1}^N\delta_{z_i}$ and $\nu=\frac1M\sum_{i=1}^M\delta_{p_i}$, this implies that
$\gamma(t)=\frac1N\sum_{i=1}^N\delta_{u_i(t)}$, where $u=(u_1,...,u_N)$ is a solution of
$$
\dot u(t)= v_t(u(t))=\nabla G_p(u(t)),\quad G_p(x)=\mathcal{SW}_2^2\Big(\frac{1}{N}\sum_{i=1}^N\delta_{x_i},\frac1M\sum_{j=1}^M\delta_{p_j}\Big).
$$
In the context of posterior sampling, \citet{DLPYL2023} considered sliced Wasserstein gradient flows starting at $\gamma(0)=\frac1N\sum_{i=1}^N\delta_{(z_i,\tilde q_i)}$ with target measure
$\nu_{M,q}=\frac1M\sum_{j=1}^M\delta_{(p_i,q_i)}$ such that $\nu_{M,q}\approx P_{X,Y}$ and $\gamma(0)\approx P_Z\times P_Y$ for continuous random variables $X$ and $Y$ and a latent variable $Z$.
Then, they observed numerically that the second part $v_{t,2}$ of $v_t=(v_{t,1},v_{t,2})\colon\R^d\times\R^n\to\R^d\times\R^n$ is ``almost zero''.
In order to apply Proposition~\ref{lem:fundamentalthm_of_conditional_generative_modelling}, they set the component $v_{t,2}$ 
in their simulations artificially to zero, which corresponds to solving the ODE
\begin{equation}\label{eq:posterior_flow_du}
\dot u(t)=\nabla_x G_{(p,q)}(u(t),\tilde q).
\end{equation}
\citet{DLPYL2023} write by themselves that they are unable to provide a rigorous theoretical
justification of the functionality of their algorithm.
Using an analogous proof as for Theorem~\ref{thm:posterior_flow}, we can now provide this justification as summarized in the following corollary.
In particular, the simulations from \citet{DLPYL2023} are still Wasserstein gradient flows, but
with respect to a different functional which has the minimizer $\nu_{M,q}\approx P_{X,Y}$.

\begin{corollary} \label{du}
Let $u=(u_1,...,u_N)\colon[0,\infty)\to(\R^d)^N$ be a solution of \eqref{eq:posterior_flow_du} and assume that $(u_i(t),\tilde q_i)\neq(u_j(t),\tilde q_j)$ for $i\neq j$ and all $t>0$.
Then, the curve $\gamma_{N,\tilde q}\colon(0,\infty)\to\mathcal P_2(\R^d)$ defined by
$$
\gamma_{N,\tilde q}(t)=\frac1N\sum_{i=1}^N\delta_{u_i(t),\tilde q_i}
$$
is a Wasserstein gradient flow with respect to the functional
$$
\mu\mapsto\begin{cases}
\mathcal {SW}_2^2(\mu,\nu_{M,q}),&\text{if } \mu\in P_{N,\tilde q}\\
\infty,&\text{otherwise.} 
\end{cases}
$$
\end{corollary}

%-------------------------------------------------------------------
\section{Algorithm summary}
Here we summarize the training of our conditional MMD Flows, see Algorithm \ref{alg:training_gen_MMD_flows}.

\begin{algorithm}[t]
\begin{algorithmic}
\State \textbf{Input:} Joint samples $p_1,...,p_N,q_1,...,q_N$ from $P_{X,Y}$, initial samples $u_1^0,...,u_N^0$, momentum parameters $m_l\in[0,1)$ for $l=1,...,L$.
\State Initialize $(v_1,...,v_N)=0$.
\For{$l=1,...,L$}
\State - Set $(\tilde u_1^{(0)},...,\tilde u_N^{(0)})=(u_1^{(l-1)},...,u_N^{(l-1)})$.
\State - Simulate $T_l$ steps of the (momentum) MMD flow:
\For{$t=1,...,T_l$}
\State - Update $v$ by
\begin{align*}
(v_1,...,v_N)\leftarrow \nabla_x  F_{d+n}((\tilde{u}_i^{(k)},q_i)_{i=1}^N|(p_i,q_i)_{i=1}^N)+m_l (v_1,...,v_N)
\end{align*}
\State - Update the flow samples:
\begin{align*}
\tilde{u}^{(k+1)}=\tilde{u}^{(k)}-\tau N  (v_1,...,v_N)
\end{align*}
\EndFor
\State - Train $\Phi_l$ such that $\tilde u^{(T_l)}\approx \tilde u_i^{(0)}-\Phi_l(\tilde u_i^{(0)},q_i) $ by minimizing the loss
\begin{align*}
\mathcal L(\theta_l)=\frac1N\sum_{i=1}^N \|\Phi_l(\tilde u_i^{(0)},q_i)-(\tilde u_i^{(0)}-\tilde u_i^{(T_l)})\|^2.
\end{align*}
\State - Set $(u_1^{(l)},...,u_N^{(l)})= (u_1^{(l-1)},...,u_N^{(l-1)})-(\Phi_l(u_1^{(l-1)},q_1),...,\Phi_l(u_N^{(l-1)},q_N))$.
\EndFor
\end{algorithmic}
\caption{Training of conditional MMD flows}
\label{alg:training_gen_MMD_flows}
\end{algorithm}
\section{Implementation details} \label{app:implementation}

The code is written in PyTorch \citep{PyTorch2019} and is available online\footnote{\url{https://github.com/FabianAltekrueger/Conditional_MMD_Flows}}.

We use UNets $(\Phi)_{l=1}^L$\footnote{modified from \url{https://github.com/hojonathanho/diffusion/blob/master/diffusion_tf/models/unet.py}} which are trained using Adam \citep{KB2015} with a learning rate of $0.0005$. Since the differences between the particles $x^{(k+1)}$ and $x^{(k)}$ are large when starting simulating \eqref{eq:time-discrete-ODE} and smaller for larger step $k$, we increase the number of simulation steps $T_l$ up to a predefined maximal number of iteration steps $T_{max} = 30000$. 

For our experiments, we make use of several improvements, which are explained in the following.

\paragraph{Sliced MMD equals MMD}

Instead of computing the derivative of the MMD functional in \eqref{eq:time-discrete-ODE} directly, we use the sliced version of MMD shown in \citet{HWAH2023}. More precisely, let
${x} \coloneqq (x_1,\ldots,x_N) \in (\R^d)^N$,  
$p \coloneqq (p_1,\ldots,p_M) \in (\R^d)^M$ and
let $F_p^d(x) \coloneqq F_p(x)$ be the discrete MMD functional, where we explicitly note the dependence on the dimension $d$.
Then we can rewrite the gradient of the MMD $\nabla_{x_i}F_p^d(x)$ with the negative distance kernel as 
\begin{align*}
\nabla_{x_i}F_p^d(x)=c_d \mathbb{E}_{\xi \sim \mathcal U_{\mathbb{S}^{d-1}}} [\partial_i F_{\tilde p_\xi}^1(\langle \xi,x_1\rangle,...,\langle\xi,x_N\rangle)\xi ],
\end{align*}
where $\tilde p_\xi \coloneqq  (\langle \xi p_1 \rangle ,\ldots, \langle \xi ,p_M \rangle)$ and $c_d$ is a constant given by
\begin{align*}
c_{d} 
\coloneqq \frac{\sqrt{\pi}\Gamma(\frac{d+1}{2})}{\Gamma(\frac{d}{2})}.
\end{align*}
Thus it suffices to compute the gradient w.r.t $F_{\tilde p_\xi}^1$, which can be done in a very efficient manner, see \citet[Section 3]{HWAH2023}.

\paragraph{Pyramidal schedules}

In order to obtain fast convergence of the particle flow \eqref{eq:time-discrete-ODE} even in high dimensions, we make use of a \textit{pyramidal schedule}. The key idea is to simulate the particle flow on different resolutions of the image, from low to high sequentially. Given the target images $p_i \in \R^d$, where $d={C\cdot H \cdot W}$ with height $H$, width $W$ and $C$ channels for $i=1,...,N$, we downsample the image by a factor $S$. Then we start simulating the ODE \eqref{eq:time-discrete-ODE} with initial particles $x^{(0)} \in \R^{\frac{d}{S^2}}$, i.e., we simulate the ODE in a substantially smaller dimension. After a predefined number of steps $t$, we upsample the current iterate $x^{(t)}$ to the higher resolution and add  noise onto it in order to increase the intrinsic dimension of the images. Then we repeat the procedure until the highest resolution is attained.

\paragraph{Locally-connected projections}

Motivated by \citet{DLPYL2023,NH2022}, the uniformly sampled projections $\xi \in \mathbb{S}^{d-1}$ can be interpreted as a fully-connected layer applied to the vectorized image. Instead, for image tasks the use of locally-connected projections greatly improves the efficiency of the corresponding particle flow. More concretely, for a given patch size $s$ we sample local projections $\xi_{\ell}$ uniformly from $\mathbb{S}^{c s^2 -1}$. Then, we randomly choose a pair $(h,w)$ and extract a patch $E_{(h,w)}(p)$ of our given image $p \in \R^d$, where $(h,w)$ is the upper left corner of the patch and $E_{(h,w)} \colon \R^d \to \R^{c s^2}$ is the corresponding patch extractor. Using this procedure, we simulate the ODE \eqref{eq:time-discrete-ODE} for $E_{(h,w)}(x_i)$ and target $E_{(h,w)}(p_i)$, $i=1,...,N$, where the location of the patch is randomly chosen in each iteration. 

In order to apply the locally-connected projections on different resolutions, we upsample the projections to different scales, similar to \citet{DLPYL2023} and, depending on the condition $q_i$, locally-connected projections are also used here.
Note that here we introduced an inductive bias, since we do not sample uniformly from $\mathbb{S}^{d-1}$ anymore, but it empirically improves the performance of the proposed scheme. A more detailed discussion can be found in \citep{DLPYL2023,NH2022}.

\subsection{Class-conditional image generation}

For MNIST and FashionMNIST we use $N=20000$ target pairs $(p_i,q_i)$, $i=1,...,N$, where the conditions $q_i$ are the one-hot vectors of the class labels. We use $P=500$ projections for each scale of the locally-connected projections, for the observation part we use fully-connected projections. For MNIST, we use the patch size $s=5$ and apply the projections on resolutions $5,10,15,20$ and $25$. For FashionMNIST, the patch size $s=9$ is used on resolutions $9$ and $27$. In both cases, the networks are trained for $5000$ iterations with a batch size of $100$. 

For CIFAR10, we use $N=40000$ target pairs and make use of the pyramidal schedule, where we first downsample by a factor 8 to resolution $4 \times 4$ and then upsample the iterates $x^{(t)}$ by a factor of 2 after every $700000$ iterations. In the first two resolutions we use $P=500$ projections and in the third resolution we use $P=778$ projections. On the highest resolution of $32 \times 32$, we make use of locally-connected projections, where we choose the patch size $s=7$ on resolutions $7,14,21$ and $28$. Here we use $P=400$ projections and train the networks for 4000 iterations with a batch size of 100.

\subsection{Inpainting}
For MNIST and FashionMNIST we use $N=20000$ target pairs $(p_i,q_i)$, $i=1,...,N$, where the conditions $q_i$ are the observed parts of the image. We use again $P=500$ projections for each scale of the locally-connected projections, for the observation part we use fully-connected projections. For MNIST, we use the patch size $s=5$ and apply the projections on resolutions $5,10,15$ and $20$. For FashionMNIST, the patch size $s=7$ is used on resolutions $7,14$ and $21$. In both cases, the networks are trained for $4000$ iterations with a batch size of $100$. 
For CIFAR10, we use $N=30000$ target pairs and make use of the same pyramidal schedule as in the class-conditional part, but we increase the resolution after every 600000 iterations. 

\subsection{Superresolution}
We use $N=20000$ target pairs of CelebA, where the low-resolution images are bicubicely downsampled to resolution $16 \times 16$. Similarly to the pyramidal approach for CIFAR10, we downsample the particles by a factor of 8 and increase the resolution after every 600000 iterations by a factor of 2. While for the first 2 resolutions we use fully-connected projections with $P=500$ and $P=768$ projections, for resolutions $32 \times 32$ and $64 \times 64$ we make use of locally-connected projections with $P=500$ and $s=7$ on resolutions $7$ and $21$ for $32 \times 32$ and $7,21$ and $49$ for $64 \times 64$. For the observations we use fully-connected projections for all resolutions. 
The networks are trained for 5000 iterations and a batch size of 10.

\subsection{Computed Tomography}

We use the first $N=400$ target pairs of the LoDoPaB training set of size $362 \times 362$, where the observations are the FBP reconstructions of the observed sinograms. We use $P=500$ projections for each scale of the locally-connected projections with a patch size $s=15$ and use all resolutions between $15$ and $135$ as well as resolution $270$. 

\section{Further numerical examples} \label{app:further_examples}

Here we provide additional generated samples for the experiments from Section~\ref{sec:exp}. Moreover, the FID scores for the class-conditonal image generation is given in Table~\ref{table:FID_classcond} for each class separately. Note that the arising values are not comparable with unconditional FID values. 
Obviously, we outperform the $\ell$-SWF of \citet{DLPYL2023}, which is the conceptually closest method.

\begin{table}
\caption{FID scores of the class-conditional samples for MNIST, FashionMNIST and CIFAR10. Separated for each class.}
\label{table:FID_classcond}
\begin{minipage}{0.23\linewidth}
\caption*{MNIST}
\scalebox{.66}{
\begin{tabular}[t]{c|ccc} 
   FID       & $\ell$-SWF          & Cond. MMD Flow   \\
             & \citep{DLPYL2023}   &   (ours) \\        
\hline
Class 0   &   16.3              &  13.2    \\ 
Class 1   &   16.0              &  23.3  \\
Class 2   &   19.0              &  17.2   \\ 
Class 3   &   13.8              &  12.8    \\ 
Class 4   &   19.7              &  13.7  \\
Class 5   &   34.0              &  13.7   \\ 
Class 6   &   18.9              &  13.2    \\ 
Class 7   &   14.6              &  14.3\\
Class 8   &   19.4              &  13.0   \\ 
Class 9   &   17.5              &  12.0   \\ \hline\hline
Average   &   18.9              &  14.6
\end{tabular}}
\end{minipage}
\hspace{1.3cm}
\begin{minipage}{0.23\linewidth}
\caption*{FashionMNIST}
\scalebox{.66}{
\begin{tabular}[t]{c|ccc} 
   FID       & $\ell$-SWF          & Cond. MMD Flow   \\
             & \citep{DLPYL2023}   &   (ours) \\        
\hline
Class 0   &   27.9              &  30.4    \\ 
Class 1   &   16.0              &  22.9  \\
Class 2   &   26.2              &  28.0   \\ 
Class 3   &   28.3              &  27.3    \\ 
Class 4   &   25.5              &  23.6  \\
Class 5   &   27.7              &  31.1   \\ 
Class 6   &   28.5              &  29.1    \\ 
Class 7   &   21.3              &  23.6\\
Class 8   &   37.2              &  38.6   \\ 
Class 9   &   20.7              &  20.0   \\ \hline\hline
Average   &   25.9              &  27.5
\end{tabular}}
\end{minipage}
\hspace{1.3cm}
\begin{minipage}{0.23\linewidth}
\caption*{CIFAR10}
\scalebox{.66}{
\begin{tabular}[t]{c|ccc} 
   FID       & $\ell$-SWF          & Cond. MMD Flow   \\
             & \citep{DLPYL2023}   &   (ours) \\        
\hline
Class 0   &   127.6              &  98.8    \\ 
Class 1   &   126.9              &  100.9  \\
Class 2   &   125.1              &  117.9   \\ 
Class 3   &   141.0              &  99.4    \\ 
Class 4   &   103.9              &  99.7  \\
Class 5   &   126.2              &  112.2   \\ 
Class 6   &   109.6              &  87.1    \\ 
Class 7   &   112.5              &  109.3\\
Class 8   &   100.6              &  94.0   \\ 
Class 9   &   107.9              &  99.3   \\ \hline\hline
Average   &   118.1              &  101.8
\end{tabular}}
\end{minipage}
\end{table}

\begin{figure}[t]
\centering
\begin{subfigure}[t]{1\textwidth}
\includegraphics[width=\linewidth]{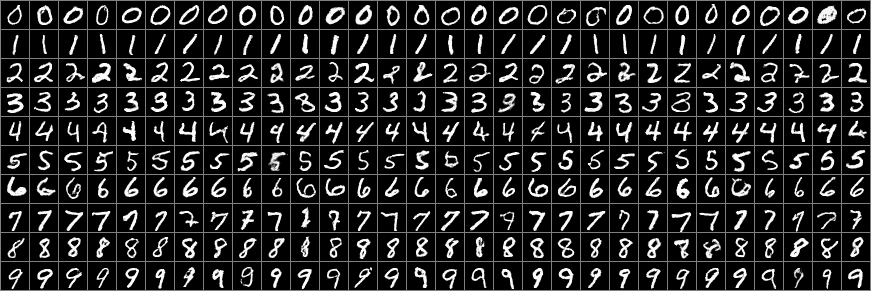}
\caption{MNIST}
\end{subfigure}%

\begin{subfigure}[t]{1\textwidth}
\includegraphics[width=\linewidth]{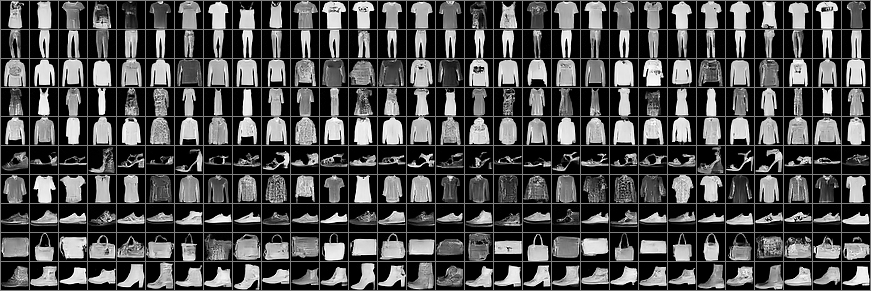}
  \caption{FashionMNIST}
\end{subfigure}

\begin{subfigure}[t]{1\textwidth}
\includegraphics[width=\linewidth]{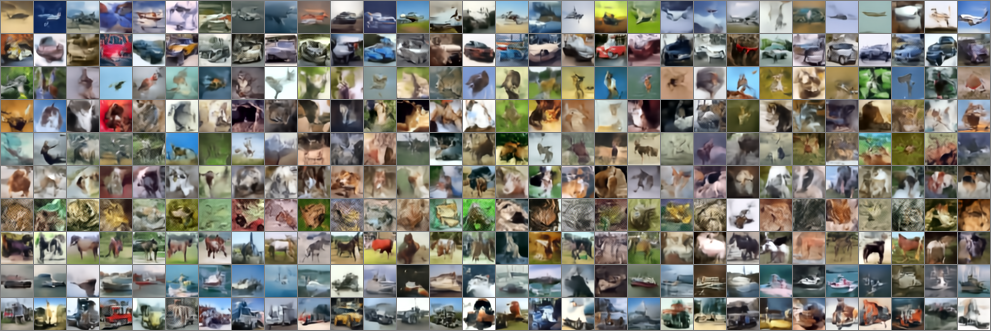}
  \caption{CIFAR10}
\end{subfigure}
\caption{Additional class-conditional samples of MNIST, FashionMNIST and CIFAR10.} \label{fig:add_classcond_samples}
\end{figure}

\begin{figure}[t]
\centering
\begin{subfigure}[t]{1\textwidth}
\includegraphics[width=\linewidth]{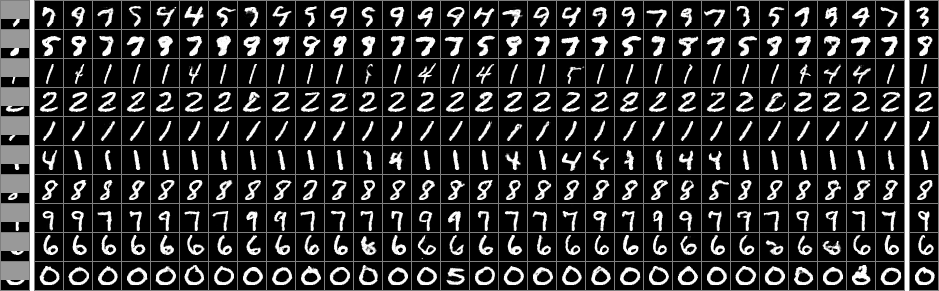}
\caption{MNIST}
\end{subfigure}%

\begin{subfigure}[t]{1\textwidth}
\includegraphics[width=\linewidth]{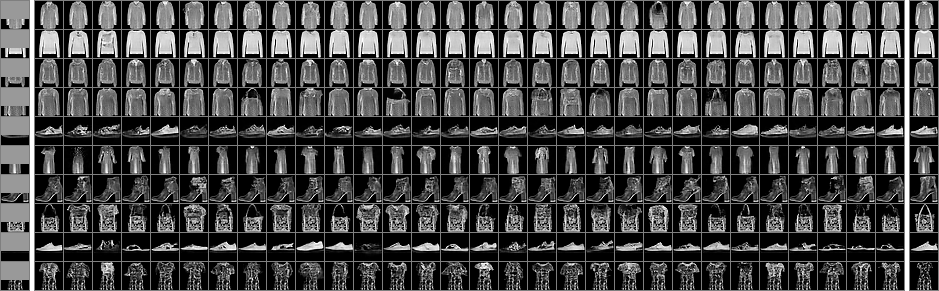}
  \caption{FashionMNIST}
\end{subfigure}

\begin{subfigure}[t]{1\textwidth}
\includegraphics[width=\linewidth]{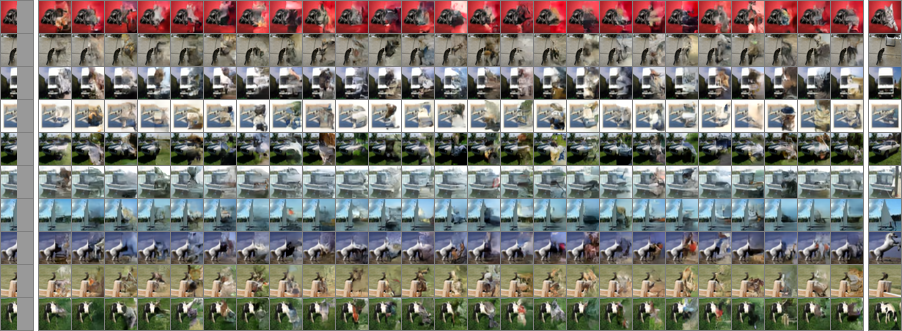}
  \caption{CIFAR10}
\end{subfigure}
\caption{Additional inpainted samples of MNIST, FashionMNIST and CIFAR10.} \label{fig:add_inpainting_samples}
\end{figure}

\begin{figure}[t]
\centering
\begin{subfigure}[t]{1\textwidth}
\includegraphics[width=\linewidth]{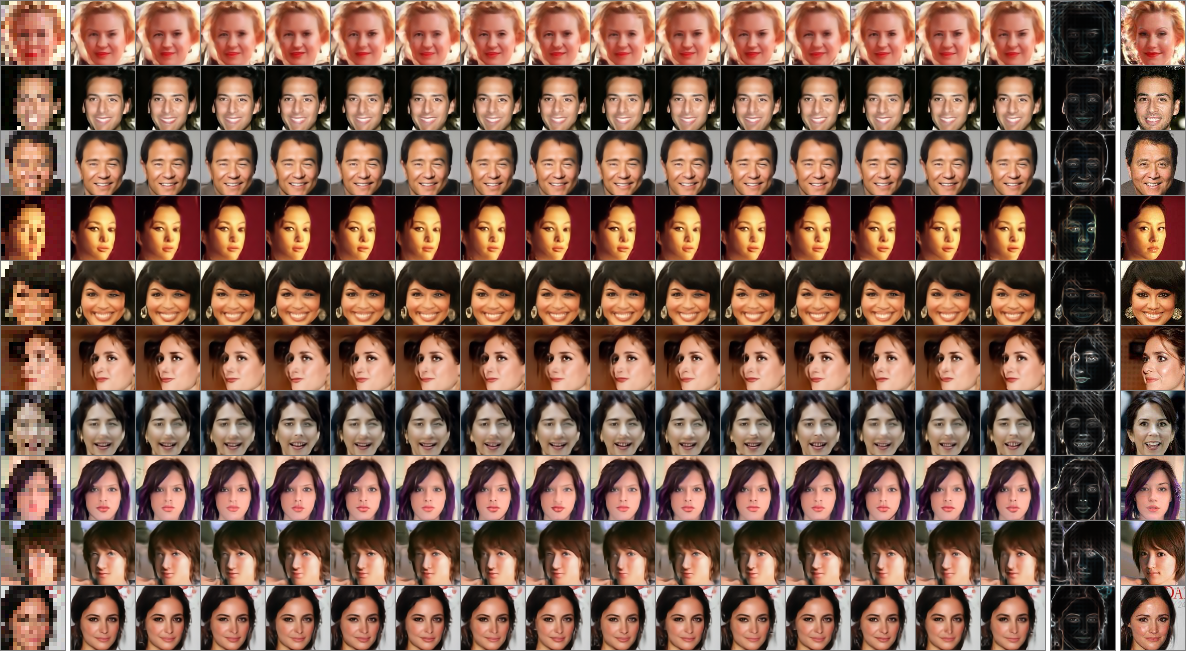}
  \caption{CelebA superresolution}
\end{subfigure}
\caption{Additional superresoluted images of CelebA} \label{fig:add_celebA_samples}
\end{figure}

\begin{figure}[t]
\centering
\begin{subfigure}[t]{.14\textwidth}
\begin{tikzpicture}[spy using outlines={rectangle,white,magnification=12,size=2.145cm, connect spies}]
\node[anchor=south west,inner sep=0]  at (0,0) {\includegraphics[width=\linewidth]{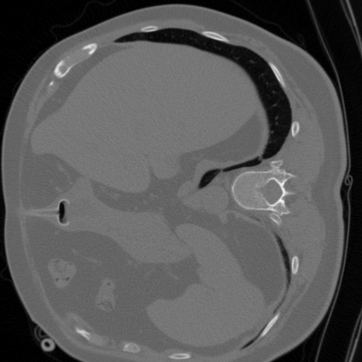}};
 \spy on (1.26,1.88) in node [right] at (-0.01,-1.1);
\end{tikzpicture}
\end{subfigure}%
\hfill
\begin{subfigure}[t]{.14\textwidth}
\begin{tikzpicture}[spy using outlines={rectangle,white,magnification=12,size=2.145cm, connect spies}]
\node[anchor=south west,inner sep=0]  at (0,0) {\includegraphics[width=\linewidth]{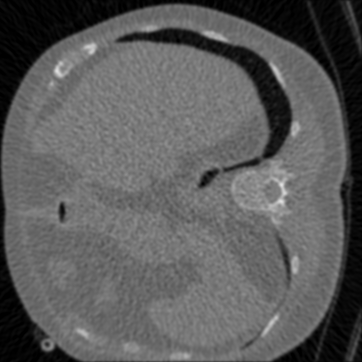}};
\spy on (1.26,1.88) in node [right] at (-0.01,-1.1);
\end{tikzpicture}
\end{subfigure}%
\hfill
\begin{subfigure}[t]{.14\textwidth}
\begin{tikzpicture}[spy using outlines={rectangle,white,magnification=12,size=2.145cm, connect spies}]
\node[anchor=south west,inner sep=0]  at (0,0) {\includegraphics[width=\linewidth]{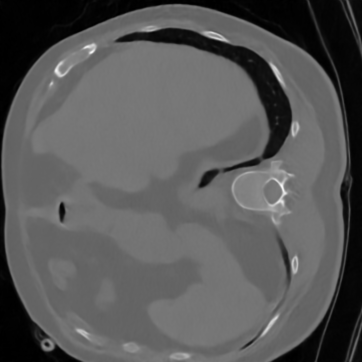}};
\spy on (1.26,1.88) in node [right] at (-0.01,-1.1);
\end{tikzpicture}
\end{subfigure}%
\hfill
\begin{subfigure}[t]{.14\textwidth}
\begin{tikzpicture}[spy using outlines={rectangle,white,magnification=12,size=2.145cm, connect spies}]
\node[anchor=south west,inner sep=0]  at (0,0) {\includegraphics[width=\linewidth]{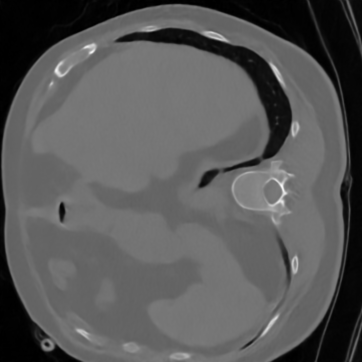}};
\spy on (1.26,1.88) in node [right] at (-0.01,-1.1);
\end{tikzpicture}
\end{subfigure}%
\hfill
\begin{subfigure}[t]{.14\textwidth}
\begin{tikzpicture}[spy using outlines={rectangle,white,magnification=12,size=2.145cm, connect spies}]
\node[anchor=south west,inner sep=0]  at (0,0) {\includegraphics[width=\linewidth]{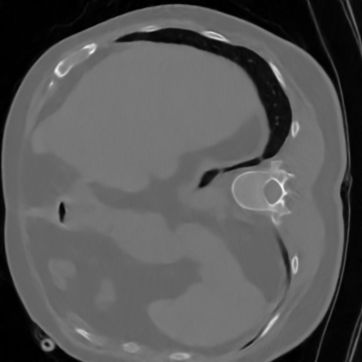}};
\spy on (1.26,1.88) in node [right] at (-0.01,-1.1);
\end{tikzpicture}
\end{subfigure}%
\hfill
\begin{subfigure}[t]{.14\textwidth}
\begin{tikzpicture}[spy using outlines={rectangle,white,magnification=12,size=2.145cm, connect spies}]
\node[anchor=south west,inner sep=0]  at (0,0) {\includegraphics[width=\linewidth]{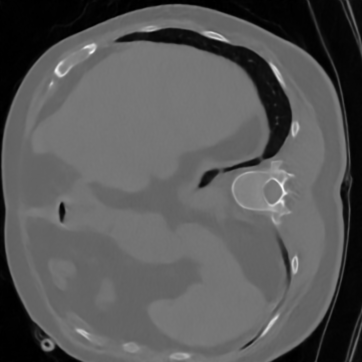}};
\spy on (1.26,1.88) in node [right] at (-0.01,-1.1);
\end{tikzpicture}
\end{subfigure}%
\hfill
\begin{subfigure}[t]{.14\textwidth}
\begin{tikzpicture}[spy using outlines={rectangle,white,magnification=12,size=2.145cm, connect spies}]
\node[anchor=south west,inner sep=0]  at (0,0) {\includegraphics[width=\linewidth]{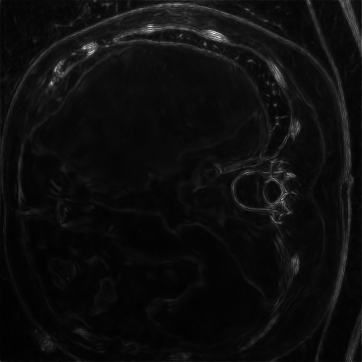}};
\spy on (1.26,1.88) in node [right] at (-0.01,-1.1);
\end{tikzpicture}
\end{subfigure}%

\begin{subfigure}[t]{.14\textwidth}
\begin{tikzpicture}[spy using outlines={rectangle,white,magnification=12,size=2.145cm, connect spies}]
\node[anchor=south west,inner sep=0]  at (0,0) {\includegraphics[width=\linewidth]{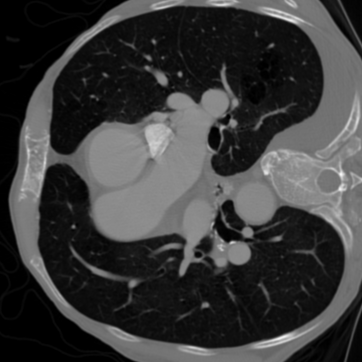}};
 \spy on (.22,1.27) in node [right] at (-0.01,-1.1);
\end{tikzpicture}
\caption*{GT}
\end{subfigure}%
\hfill
\begin{subfigure}[t]{.14\textwidth}
\begin{tikzpicture}[spy using outlines={rectangle,white,magnification=12,size=2.145cm, connect spies}]
\node[anchor=south west,inner sep=0]  at (0,0) {\includegraphics[width=\linewidth]{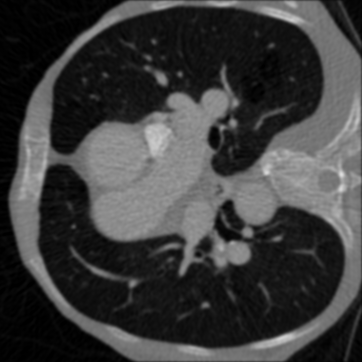}};
 \spy on (.22,1.27) in node [right] at (-0.01,-1.1);
\end{tikzpicture}
  \caption*{FBP}
\end{subfigure}%
\hfill
\begin{subfigure}[t]{.14\textwidth}
\begin{tikzpicture}[spy using outlines={rectangle,white,magnification=12,size=2.145cm, connect spies}]
\node[anchor=south west,inner sep=0]  at (0,0) {\includegraphics[width=\linewidth]{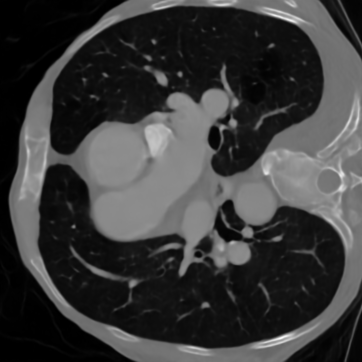}};
 \spy on (.22,1.27) in node [right] at (-0.01,-1.1);
\end{tikzpicture}
  \caption*{Reco 1}
\end{subfigure}%
\hfill
\begin{subfigure}[t]{.14\textwidth}
\begin{tikzpicture}[spy using outlines={rectangle,white,magnification=12,size=2.145cm, connect spies}]
\node[anchor=south west,inner sep=0]  at (0,0) {\includegraphics[width=\linewidth]{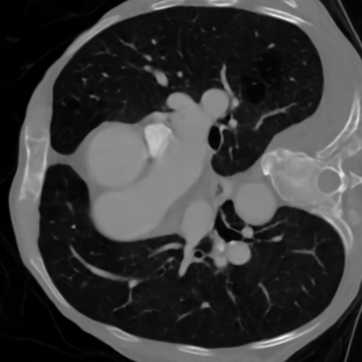}};
 \spy on (.22,1.27) in node [right] at (-0.01,-1.1);
\end{tikzpicture}
 \caption*{Reco 2}
\end{subfigure}%
\hfill
\begin{subfigure}[t]{.14\textwidth}
\begin{tikzpicture}[spy using outlines={rectangle,white,magnification=12,size=2.145cm, connect spies}]
\node[anchor=south west,inner sep=0]  at (0,0) {\includegraphics[width=\linewidth]{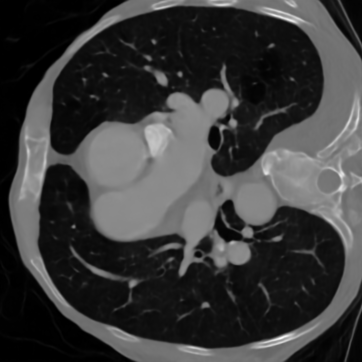}};
 \spy on (.22,1.27) in node [right] at (-0.01,-1.1);
\end{tikzpicture}
  \caption*{Reco 3}
\end{subfigure}%
\hfill
\begin{subfigure}[t]{.14\textwidth}
\begin{tikzpicture}[spy using outlines={rectangle,white,magnification=12,size=2.145cm, connect spies}]
\node[anchor=south west,inner sep=0]  at (0,0) {\includegraphics[width=\linewidth]{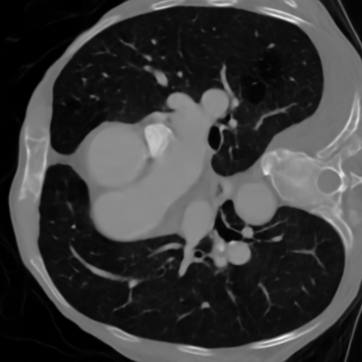}};
 \spy on (.22,1.27) in node [right] at (-0.01,-1.1);
\end{tikzpicture}
  \caption*{Mean}
\end{subfigure}%
\hfill
\begin{subfigure}[t]{.14\textwidth}
\begin{tikzpicture}[spy using outlines={rectangle,white,magnification=12,size=2.145cm, connect spies}]
\node[anchor=south west,inner sep=0]  at (0,0) {\includegraphics[width=\linewidth]{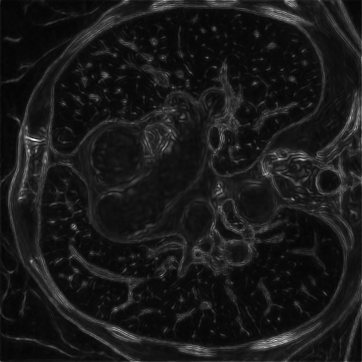}};
 \spy on (.22,1.27) in node [right] at (-0.01,-1.1);
\end{tikzpicture}
  \caption*{Std}
\end{subfigure}%

\begin{subfigure}[t]{.14\textwidth}
\begin{tikzpicture}[spy using outlines={rectangle,white,magnification=10,size=2.145cm, connect spies}]
\node[anchor=south west,inner sep=0]  at (0,0) {\includegraphics[width=\linewidth]{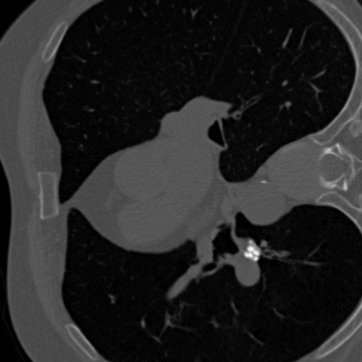}};
 \spy on (1.3,1.35) in node [right] at (-0.01,-1.1);
\end{tikzpicture}
\end{subfigure}%
\hfill
\begin{subfigure}[t]{.14\textwidth}
\begin{tikzpicture}[spy using outlines={rectangle,white,magnification=10,size=2.145cm, connect spies}]
\node[anchor=south west,inner sep=0]  at (0,0) {\includegraphics[width=\linewidth]{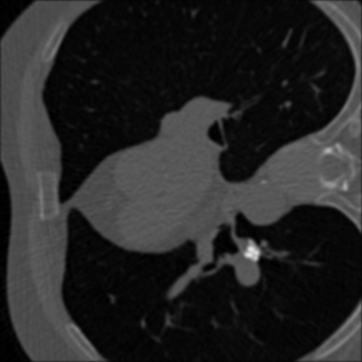}};
\spy on (1.3,1.35) in node [right] at (-0.01,-1.1);
\end{tikzpicture}
\end{subfigure}%
\hfill
\begin{subfigure}[t]{.14\textwidth}
\begin{tikzpicture}[spy using outlines={rectangle,white,magnification=10,size=2.145cm, connect spies}]
\node[anchor=south west,inner sep=0]  at (0,0) {\includegraphics[width=\linewidth]{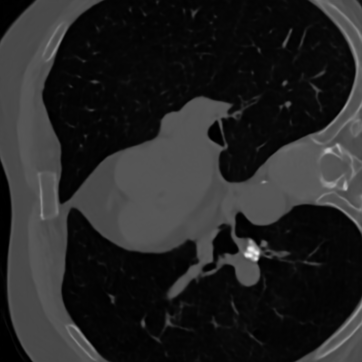}};
\spy on (1.3,1.35) in node [right] at (-0.01,-1.1);
\end{tikzpicture}
\end{subfigure}%
\hfill
\begin{subfigure}[t]{.14\textwidth}
\begin{tikzpicture}[spy using outlines={rectangle,white,magnification=10,size=2.145cm, connect spies}]
\node[anchor=south west,inner sep=0]  at (0,0) {\includegraphics[width=\linewidth]{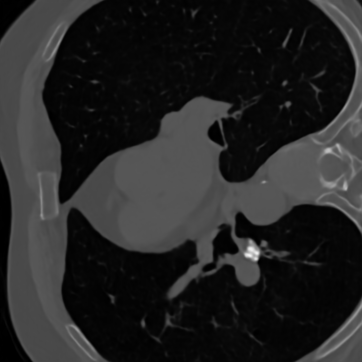}};
\spy on (1.3,1.35) in node [right] at (-0.01,-1.1);
\end{tikzpicture}
\end{subfigure}%
\hfill
\begin{subfigure}[t]{.14\textwidth}
\begin{tikzpicture}[spy using outlines={rectangle,white,magnification=10,size=2.145cm, connect spies}]
\node[anchor=south west,inner sep=0]  at (0,0) {\includegraphics[width=\linewidth]{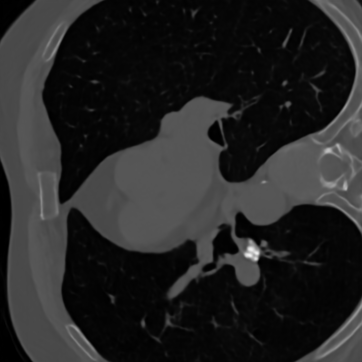}};
\spy on (1.3,1.35) in node [right] at (-0.01,-1.1);
\end{tikzpicture}
\end{subfigure}%
\hfill
\begin{subfigure}[t]{.14\textwidth}
\begin{tikzpicture}[spy using outlines={rectangle,white,magnification=10,size=2.145cm, connect spies}]
\node[anchor=south west,inner sep=0]  at (0,0) {\includegraphics[width=\linewidth]{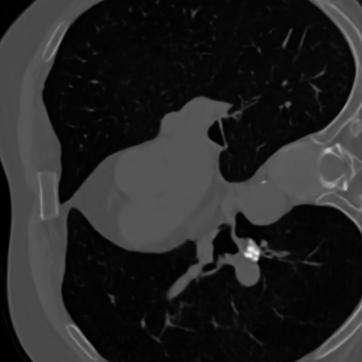}};
\spy on (1.3,1.35) in node [right] at (-0.01,-1.1);
\end{tikzpicture}
\end{subfigure}%
\hfill
\begin{subfigure}[t]{.14\textwidth}
\begin{tikzpicture}[spy using outlines={rectangle,white,magnification=10,size=2.145cm, connect spies}]
\node[anchor=south west,inner sep=0]  at (0,0) {\includegraphics[width=\linewidth]{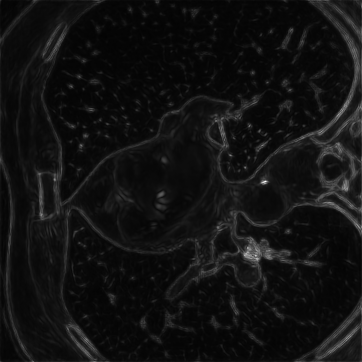}};
\spy on (1.3,1.35) in node [right] at (-0.01,-1.1);
\end{tikzpicture}
\end{subfigure}%

\begin{subfigure}[t]{.14\textwidth}
\begin{tikzpicture}[spy using outlines={rectangle,white,magnification=12,size=2.145cm, connect spies}]
\node[anchor=south west,inner sep=0]  at (0,0) {\includegraphics[width=\linewidth]{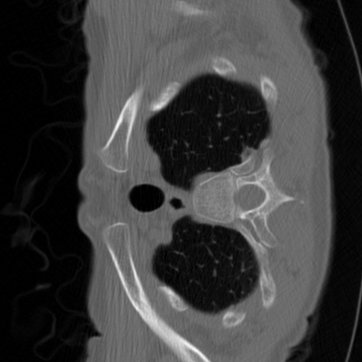}};
 \spy on (.68,1.2) in node [right] at (-0.01,-1.1);
\end{tikzpicture}
\caption*{GT}
\end{subfigure}%
\hfill
\begin{subfigure}[t]{.14\textwidth}
\begin{tikzpicture}[spy using outlines={rectangle,white,magnification=12,size=2.145cm, connect spies}]
\node[anchor=south west,inner sep=0]  at (0,0) {\includegraphics[width=\linewidth]{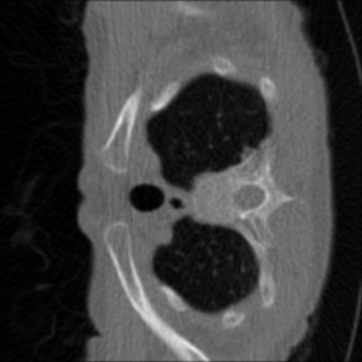}};
 \spy on (.68,1.2) in node [right] at (-0.01,-1.1);
\end{tikzpicture}
  \caption*{FBP}
\end{subfigure}%
\hfill
\begin{subfigure}[t]{.14\textwidth}
\begin{tikzpicture}[spy using outlines={rectangle,white,magnification=12,size=2.145cm, connect spies}]
\node[anchor=south west,inner sep=0]  at (0,0) {\includegraphics[width=\linewidth]{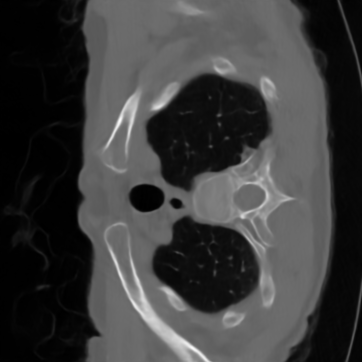}};
 \spy on (.68,1.2) in node [right] at (-0.01,-1.1);
\end{tikzpicture}
  \caption*{Reco 1}
\end{subfigure}%
\hfill
\begin{subfigure}[t]{.14\textwidth}
\begin{tikzpicture}[spy using outlines={rectangle,white,magnification=12,size=2.145cm, connect spies}]
\node[anchor=south west,inner sep=0]  at (0,0) {\includegraphics[width=\linewidth]{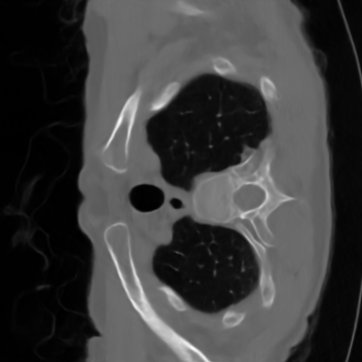}};
 \spy on (.68,1.2) in node [right] at (-0.01,-1.1);
\end{tikzpicture}
 \caption*{Reco 2}
\end{subfigure}%
\hfill
\begin{subfigure}[t]{.14\textwidth}
\begin{tikzpicture}[spy using outlines={rectangle,white,magnification=12,size=2.145cm, connect spies}]
\node[anchor=south west,inner sep=0]  at (0,0) {\includegraphics[width=\linewidth]{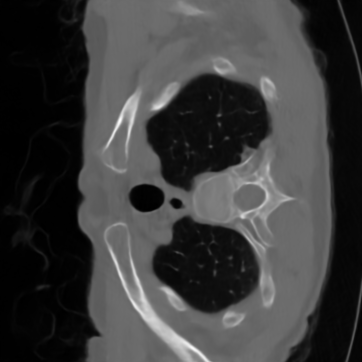}};
 \spy on (.68,1.2) in node [right] at (-0.01,-1.1);
\end{tikzpicture}
  \caption*{Reco 3}
\end{subfigure}%
\hfill
\begin{subfigure}[t]{.14\textwidth}
\begin{tikzpicture}[spy using outlines={rectangle,white,magnification=12,size=2.145cm, connect spies}]
\node[anchor=south west,inner sep=0]  at (0,0) {\includegraphics[width=\linewidth]{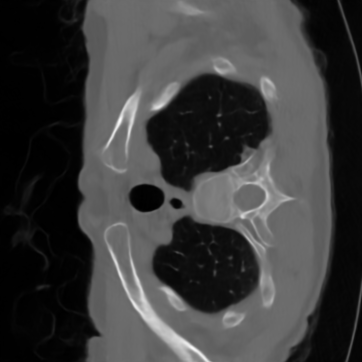}};
 \spy on (.68,1.2) in node [right] at (-0.01,-1.1);
\end{tikzpicture}
  \caption*{Mean}
\end{subfigure}%
\hfill
\begin{subfigure}[t]{.14\textwidth}
\begin{tikzpicture}[spy using outlines={rectangle,white,magnification=12,size=2.145cm, connect spies}]
\node[anchor=south west,inner sep=0]  at (0,0) {\includegraphics[width=\linewidth]{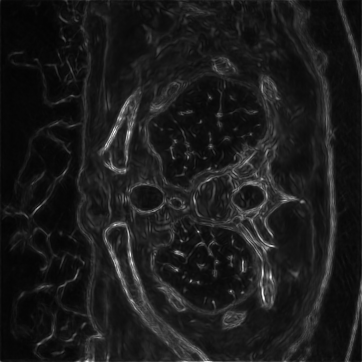}};
 \spy on (.68,1.2) in node [right] at (-0.01,-1.1);
\end{tikzpicture}
  \caption*{Std}
\end{subfigure}%
\caption{Additional generated posterior samples, mean image and pixel-wise standard deviation for low-dose computed tomography using conditional MMD flows.} \label{fig:add_CT_lowdose}
\end{figure}

\begin{figure}[t]
\centering
\begin{subfigure}[t]{.14\textwidth}
\begin{tikzpicture}[spy using outlines={rectangle,white,magnification=5.5,size=2.145cm, connect spies}]
\node[anchor=south west,inner sep=0]  at (0,0) {\includegraphics[width=\linewidth]{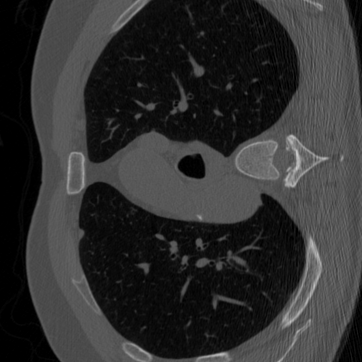}};
 \spy on (1.55,1.2) in node [right] at (-0.01,-1.1);
\end{tikzpicture}
\end{subfigure}%
\hfill
\begin{subfigure}[t]{.14\textwidth}
\begin{tikzpicture}[spy using outlines={rectangle,white,magnification=5.5,size=2.145cm, connect spies}]
\node[anchor=south west,inner sep=0]  at (0,0) {\includegraphics[width=\linewidth]{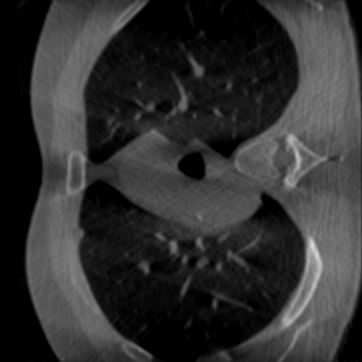}};
\spy on (1.55,1.2) in node [right] at (-0.01,-1.1);
\end{tikzpicture}
\end{subfigure}%
\hfill
\begin{subfigure}[t]{.14\textwidth}
\begin{tikzpicture}[spy using outlines={rectangle,white,magnification=5.5,size=2.145cm, connect spies}]
\node[anchor=south west,inner sep=0]  at (0,0) {\includegraphics[width=\linewidth]{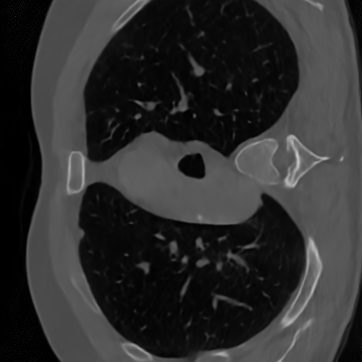}};
\spy on (1.55,1.2) in node [right] at (-0.01,-1.1);
\end{tikzpicture}
\end{subfigure}%
\hfill
\begin{subfigure}[t]{.14\textwidth}
\begin{tikzpicture}[spy using outlines={rectangle,white,magnification=5.5,size=2.145cm, connect spies}]
\node[anchor=south west,inner sep=0]  at (0,0) {\includegraphics[width=\linewidth]{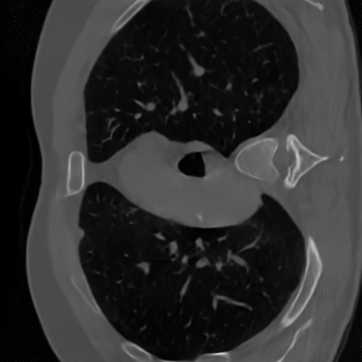}};
\spy on (1.55,1.2) in node [right] at (-0.01,-1.1);
\end{tikzpicture}
\end{subfigure}%
\hfill
\begin{subfigure}[t]{.14\textwidth}
\begin{tikzpicture}[spy using outlines={rectangle,white,magnification=5.5,size=2.145cm, connect spies}]
\node[anchor=south west,inner sep=0]  at (0,0) {\includegraphics[width=\linewidth]{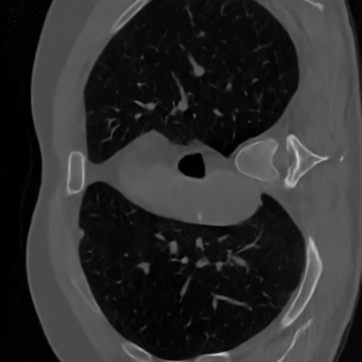}};
\spy on (1.55,1.2) in node [right] at (-0.01,-1.1);
\end{tikzpicture}
\end{subfigure}%
\hfill
\begin{subfigure}[t]{.14\textwidth}
\begin{tikzpicture}[spy using outlines={rectangle,white,magnification=5.5,size=2.145cm, connect spies}]
\node[anchor=south west,inner sep=0]  at (0,0) {\includegraphics[width=\linewidth]{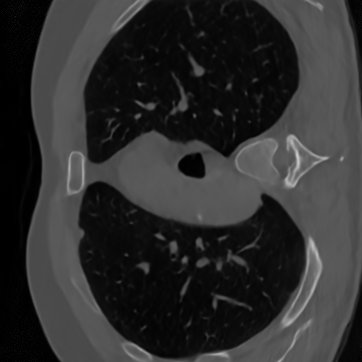}};
\spy on (1.55,1.2) in node [right] at (-0.01,-1.1);
\end{tikzpicture}
\end{subfigure}%
\hfill
\begin{subfigure}[t]{.14\textwidth}
\begin{tikzpicture}[spy using outlines={rectangle,white,magnification=5.5,size=2.145cm, connect spies}]
\node[anchor=south west,inner sep=0]  at (0,0) {\includegraphics[width=\linewidth]{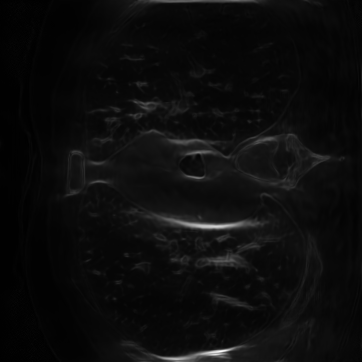}};
\spy on (1.55,1.2) in node [right] at (-0.01,-1.1);
\end{tikzpicture}
\end{subfigure}%

\begin{subfigure}[t]{.14\textwidth}
\begin{tikzpicture}[spy using outlines={rectangle,white,magnification=5.5,size=2.145cm, connect spies}]
\node[anchor=south west,inner sep=0]  at (0,0) {\includegraphics[width=\linewidth]{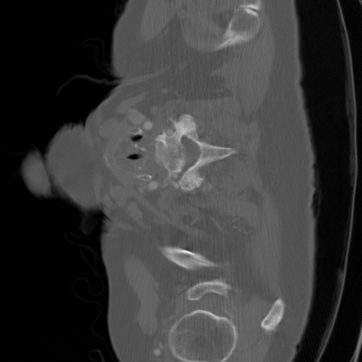}};
 \spy on (1.12,.6) in node [right] at (-0.01,-1.1);
\end{tikzpicture}
\caption*{GT}
\end{subfigure}%
\hfill
\begin{subfigure}[t]{.14\textwidth}
\begin{tikzpicture}[spy using outlines={rectangle,white,magnification=5.5,size=2.145cm, connect spies}]
\node[anchor=south west,inner sep=0]  at (0,0) {\includegraphics[width=\linewidth]{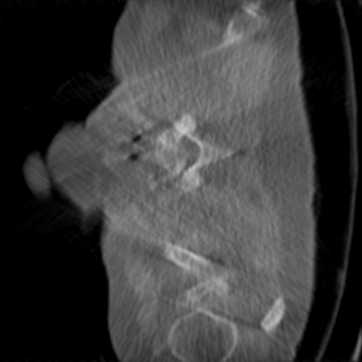}};
\spy on (1.12,.6) in node [right] at (-0.01,-1.1);
\end{tikzpicture}
  \caption*{FBP}
\end{subfigure}%
\hfill
\begin{subfigure}[t]{.14\textwidth}
\begin{tikzpicture}[spy using outlines={rectangle,white,magnification=5.5,size=2.145cm, connect spies}]
\node[anchor=south west,inner sep=0]  at (0,0) {\includegraphics[width=\linewidth]{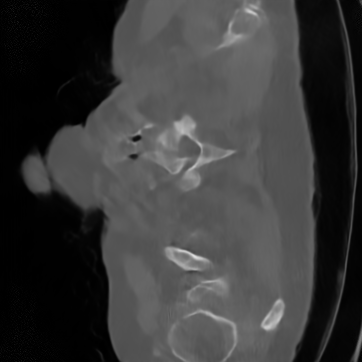}};
\spy on (1.12,.6) in node [right] at (-0.01,-1.1);
\end{tikzpicture}
  \caption*{Reco 1}
\end{subfigure}%
\hfill
\begin{subfigure}[t]{.14\textwidth}
\begin{tikzpicture}[spy using outlines={rectangle,white,magnification=5.5,size=2.145cm, connect spies}]
\node[anchor=south west,inner sep=0]  at (0,0) {\includegraphics[width=\linewidth]{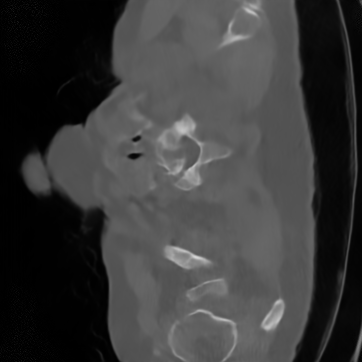}};
\spy on (1.12,.6) in node [right] at (-0.01,-1.1);
\end{tikzpicture}
  \caption*{Reco 2}
\end{subfigure}%
\hfill
\begin{subfigure}[t]{.14\textwidth}
\begin{tikzpicture}[spy using outlines={rectangle,white,magnification=5.5,size=2.145cm, connect spies}]
\node[anchor=south west,inner sep=0]  at (0,0) {\includegraphics[width=\linewidth]{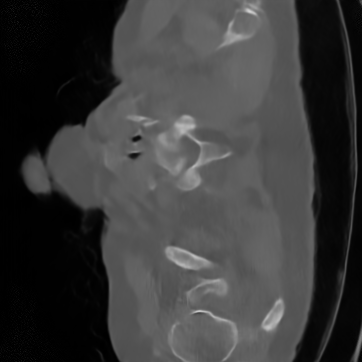}};
\spy on (1.12,.6) in node [right] at (-0.01,-1.1);
\end{tikzpicture}
  \caption*{Reco 3}
\end{subfigure}%
\hfill
\begin{subfigure}[t]{.14\textwidth}
\begin{tikzpicture}[spy using outlines={rectangle,white,magnification=5.5,size=2.145cm, connect spies}]
\node[anchor=south west,inner sep=0]  at (0,0) {\includegraphics[width=\linewidth]{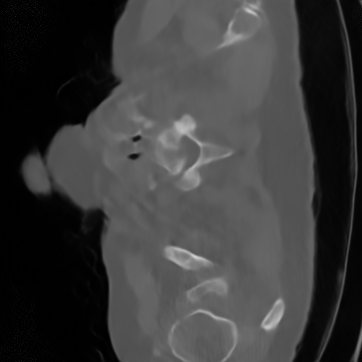}};
\spy on (1.12,.6) in node [right] at (-0.01,-1.1);
\end{tikzpicture}
  \caption*{Mean}
\end{subfigure}%
\hfill
\begin{subfigure}[t]{.14\textwidth}
\begin{tikzpicture}[spy using outlines={rectangle,white,magnification=5.5,size=2.145cm, connect spies}]
\node[anchor=south west,inner sep=0]  at (0,0) {\includegraphics[width=\linewidth]{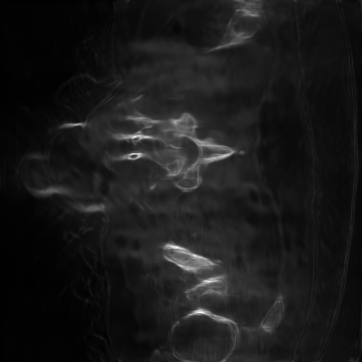}};
\spy on (1.12,.6) in node [right] at (-0.01,-1.1);
\end{tikzpicture}
  \caption*{Std}
\end{subfigure}%

\begin{subfigure}[t]{.14\textwidth}
\begin{tikzpicture}[spy using outlines={rectangle,white,magnification=5.5,size=2.145cm, connect spies}]
\node[anchor=south west,inner sep=0]  at (0,0) {\includegraphics[width=\linewidth]{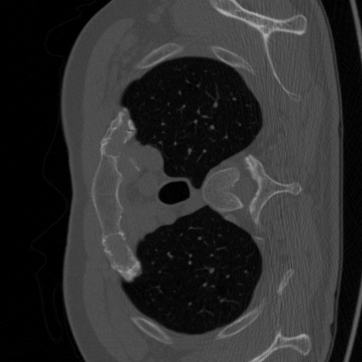}};
 \spy on (1.22,.28) in node [right] at (-0.01,-1.1);
\end{tikzpicture}
\end{subfigure}%
\hfill
\begin{subfigure}[t]{.14\textwidth}
\begin{tikzpicture}[spy using outlines={rectangle,white,magnification=5.5,size=2.145cm, connect spies}]
\node[anchor=south west,inner sep=0]  at (0,0) {\includegraphics[width=\linewidth]{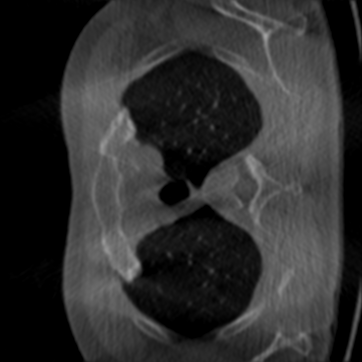}};
\spy on (1.22,.28) in node [right] at (-0.01,-1.1);
\end{tikzpicture}
\end{subfigure}%
\hfill
\begin{subfigure}[t]{.14\textwidth}
\begin{tikzpicture}[spy using outlines={rectangle,white,magnification=5.5,size=2.145cm, connect spies}]
\node[anchor=south west,inner sep=0]  at (0,0) {\includegraphics[width=\linewidth]{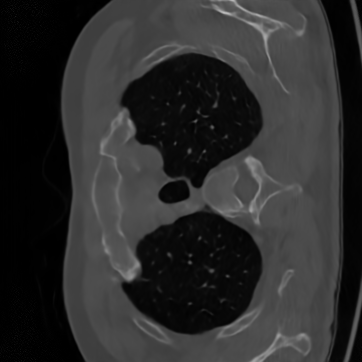}};
\spy on (1.22,.28) in node [right] at (-0.01,-1.1);
\end{tikzpicture}
\end{subfigure}%
\hfill
\begin{subfigure}[t]{.14\textwidth}
\begin{tikzpicture}[spy using outlines={rectangle,white,magnification=5.5,size=2.145cm, connect spies}]
\node[anchor=south west,inner sep=0]  at (0,0) {\includegraphics[width=\linewidth]{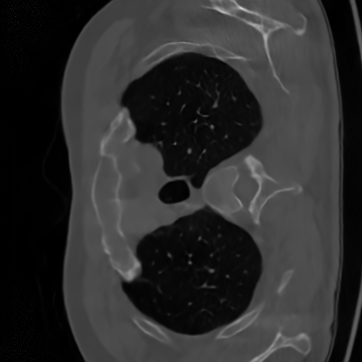}};
\spy on (1.22,.28) in node [right] at (-0.01,-1.1);
\end{tikzpicture}
\end{subfigure}%
\hfill
\begin{subfigure}[t]{.14\textwidth}
\begin{tikzpicture}[spy using outlines={rectangle,white,magnification=5.5,size=2.145cm, connect spies}]
\node[anchor=south west,inner sep=0]  at (0,0) {\includegraphics[width=\linewidth]{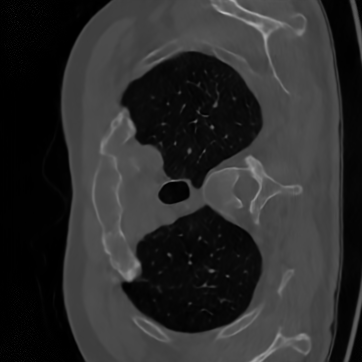}};
\spy on (1.22,.28) in node [right] at (-0.01,-1.1);
\end{tikzpicture}
\end{subfigure}%
\hfill
\begin{subfigure}[t]{.14\textwidth}
\begin{tikzpicture}[spy using outlines={rectangle,white,magnification=5.5,size=2.145cm, connect spies}]
\node[anchor=south west,inner sep=0]  at (0,0) {\includegraphics[width=\linewidth]{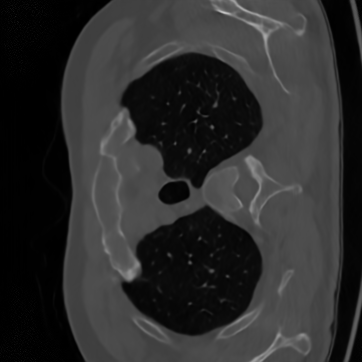}};
\spy on (1.22,.28) in node [right] at (-0.01,-1.1);
\end{tikzpicture}
\end{subfigure}%
\hfill
\begin{subfigure}[t]{.14\textwidth}
\begin{tikzpicture}[spy using outlines={rectangle,white,magnification=5.5,size=2.145cm, connect spies}]
\node[anchor=south west,inner sep=0]  at (0,0) {\includegraphics[width=\linewidth]{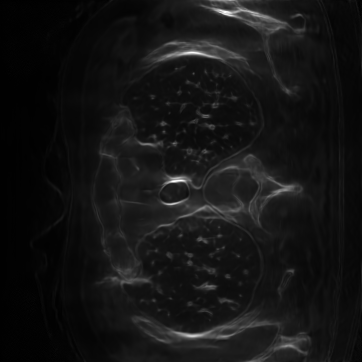}};
\spy on (1.22,.28) in node [right] at (-0.01,-1.1);
\end{tikzpicture}
\end{subfigure}%

\begin{subfigure}[t]{.14\textwidth}
\begin{tikzpicture}[spy using outlines={rectangle,white,magnification=7.,size=2.145cm, connect spies}]
\node[anchor=south west,inner sep=0]  at (0,0) {\includegraphics[width=\linewidth]{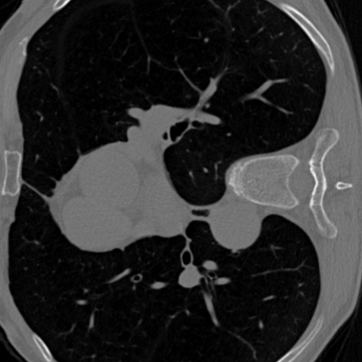}};
 \spy on (1.15,1.4) in node [right] at (-0.01,-1.1);
\end{tikzpicture}
\caption*{GT}
\end{subfigure}%
\hfill
\begin{subfigure}[t]{.14\textwidth}
\begin{tikzpicture}[spy using outlines={rectangle,white,magnification=7,size=2.145cm, connect spies}]
\node[anchor=south west,inner sep=0]  at (0,0) {\includegraphics[width=\linewidth]{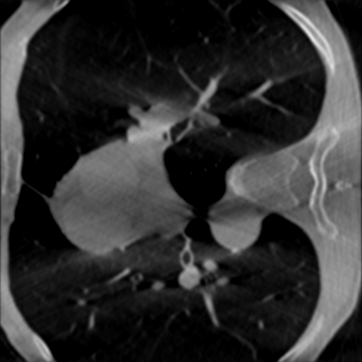}};
 \spy on (1.15,1.4) in node [right] at (-0.01,-1.1);
\end{tikzpicture}
  \caption*{FBP}
\end{subfigure}%
\hfill
\begin{subfigure}[t]{.14\textwidth}
\begin{tikzpicture}[spy using outlines={rectangle,white,magnification=7,size=2.145cm, connect spies}]
\node[anchor=south west,inner sep=0]  at (0,0) {\includegraphics[width=\linewidth]{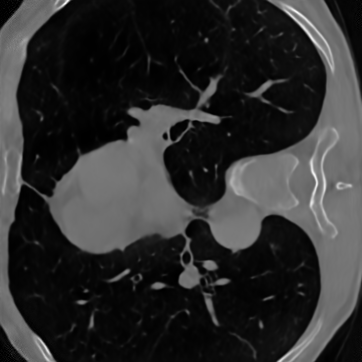}};
 \spy on (1.15,1.4) in node [right] at (-0.01,-1.1);
\end{tikzpicture}
  \caption*{Reco 1}
\end{subfigure}%
\hfill
\begin{subfigure}[t]{.14\textwidth}
\begin{tikzpicture}[spy using outlines={rectangle,white,magnification=7,size=2.145cm, connect spies}]
\node[anchor=south west,inner sep=0]  at (0,0) {\includegraphics[width=\linewidth]{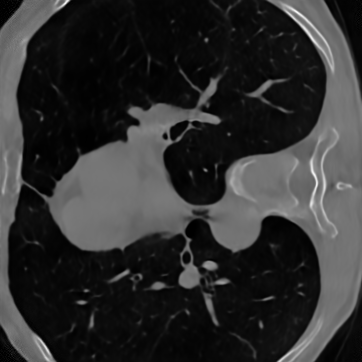}};
 \spy on (1.15,1.4) in node [right] at (-0.01,-1.1);
\end{tikzpicture}
  \caption*{Reco 2}
\end{subfigure}%
\hfill
\begin{subfigure}[t]{.14\textwidth}
\begin{tikzpicture}[spy using outlines={rectangle,white,magnification=7,size=2.145cm, connect spies}]
\node[anchor=south west,inner sep=0]  at (0,0) {\includegraphics[width=\linewidth]{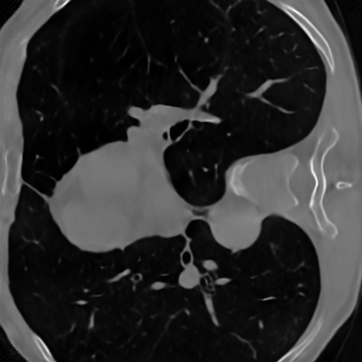}};
 \spy on (1.15,1.4) in node [right] at (-0.01,-1.1);
\end{tikzpicture}
  \caption*{Reco 3}
\end{subfigure}%
\hfill
\begin{subfigure}[t]{.14\textwidth}
\begin{tikzpicture}[spy using outlines={rectangle,white,magnification=7,size=2.145cm, connect spies}]
\node[anchor=south west,inner sep=0]  at (0,0) {\includegraphics[width=\linewidth]{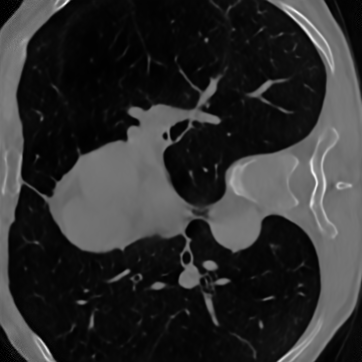}};
 \spy on (1.15,1.4) in node [right] at (-0.01,-1.1);
\end{tikzpicture}
  \caption*{Mean}
\end{subfigure}%
\hfill
\begin{subfigure}[t]{.14\textwidth}
\begin{tikzpicture}[spy using outlines={rectangle,white,magnification=7,size=2.145cm, connect spies}]
\node[anchor=south west,inner sep=0]  at (0,0) {\includegraphics[width=\linewidth]{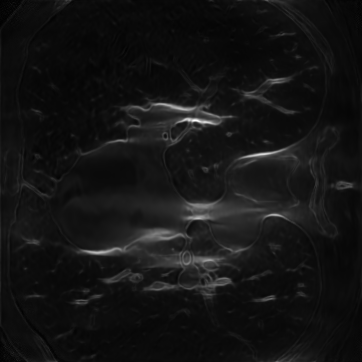}};
 \spy on (1.15,1.4) in node [right] at (-0.01,-1.1);
\end{tikzpicture}
  \caption*{Std}
\end{subfigure}%
\caption{Additional generated posterior samples, mean image and pixel-wise standard deviation for limited angle computed tomography using conditional MMD flows.} \label{fig:add_CT_limangle}
\end{figure}

\section{Superresolution on material data} \label{app:SiC_superres}

In addition to the superresolution example in Section~\ref{sec:exp}, we provide an example for superresolution on real-world material data. The forward operator $f$ consists of a blur operator with a $16 \times 16$ Gaussian blur kernel with standard deviation $2$ and a subsampling with stride 4. The low-resolution images used for reconstruction are generated by artificially downsampling and adding Gaussian noise with standard deviation $0.01$. 
We train the conditional MMD flow with 1000 pairs of high- and low-resolution images of size $100\times 100$ and $25 \times 25$, respectively. 

In Figure~\ref{fig:superres_SiC} we consider an unknown ground truth of size $600 \times 600$ and an observation of size $150 \times 150$. 
We compare the conditional MMD flow with WPPFlow \citep{AH2022} and SRFlow \citep{LDVT2020}. Note that while the conditional MMD flow and the SRFlow are trained under the same setting, the WPPFlow just requires the 1000 observations and the exact knowledge of the forward operator and the noise model for training. We can observe that the conditional MMD flow is able to generate sharper and much more realistic reconstructions for the given low-resolution observation. More examples are given in Figure~\ref{fig:add_SiC_examples}. 
A quantitative comparison of the methods is given in Table~\ref{tab:SiC_superres}. Here we consider 100 unseen test observations and reconstruct 100 samples from the posterior for each observation.

\begin{figure}
\centering
\begin{subfigure}[t]{.14\textwidth}
\begin{tikzpicture}[spy using outlines={rectangle,black,magnification=10,size=2.09cm, connect spies}]
\node[anchor=south west,inner sep=0]  at (0,0) {\includegraphics[width=\linewidth]{imgs/SiC/hr_SiC.png}};
  \spy on (1.26,.84) in node [right] at (0.02,-1.07);
\end{tikzpicture}
\caption*{HR image}
\end{subfigure}%
\hfill
\begin{subfigure}[t]{.14\textwidth}
\begin{tikzpicture}[spy using outlines={rectangle,black,magnification=10,size=2.09cm, connect spies}]
\node[anchor=south west,inner sep=0]  at (0,0) {\includegraphics[width=\linewidth]{imgs/SiC/lr_SiC_x4.png}};
\spy on (1.24,.863) in node [right] at (0.02,-1.07);
\end{tikzpicture}
\caption*{LR image}
\end{subfigure}%
\hfill
\begin{subfigure}[t]{.14\textwidth}
\begin{tikzpicture}[spy using outlines={rectangle,black,magnification=10,size=2.09cm, connect spies}]
\node[anchor=south west,inner sep=0]  at (0,0) {\includegraphics[width=\linewidth]{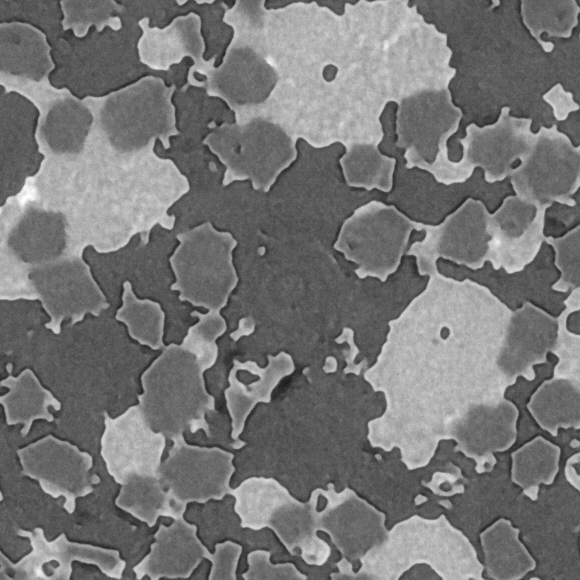}};
\spy on (1.26,.84) in node [right] at (0.02,-1.07);
\end{tikzpicture}
\caption*{Prediction 1}
\end{subfigure}%
\hfill
\begin{subfigure}[t]{.14\textwidth}
\begin{tikzpicture}[spy using outlines={rectangle,black,magnification=10,size=2.09cm, connect spies}]
\node[anchor=south west,inner sep=0]  at (0,0) {\includegraphics[width=\linewidth]{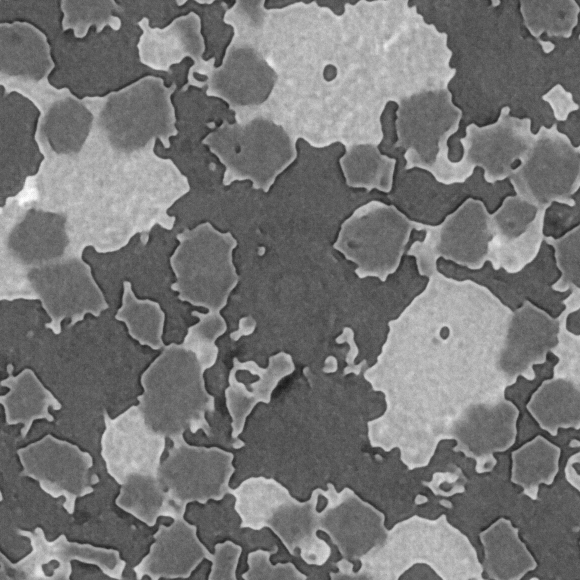}};
\spy on (1.26,.84) in node [right] at (0.02,-1.07);
\end{tikzpicture}
\caption*{Prediction 2}
\end{subfigure}%
\hfill
\begin{subfigure}[t]{.14\textwidth}
\begin{tikzpicture}[spy using outlines={rectangle,black,magnification=10,size=2.09cm, connect spies}]
\node[anchor=south west,inner sep=0]  at (0,0) {\includegraphics[width=\linewidth]{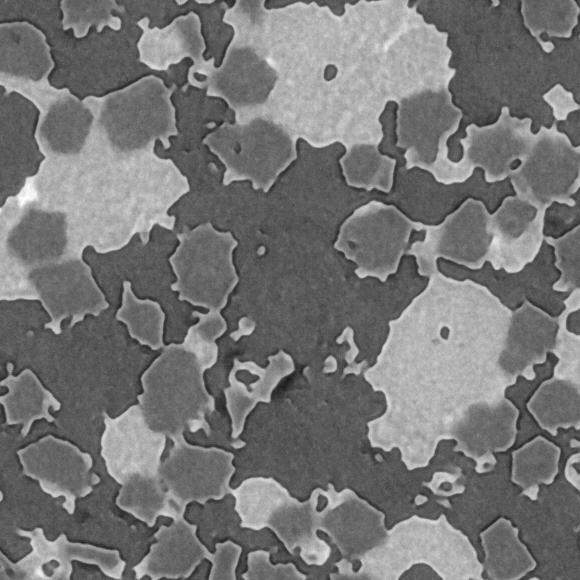}};
\spy on (1.26,.84) in node [right] at (0.02,-1.07);
\end{tikzpicture}
\caption*{Prediction 3}
\end{subfigure}%
\hfill
\begin{subfigure}[t]{.14\textwidth}
\begin{tikzpicture}[spy using outlines={rectangle,black,magnification=10,size=2.09cm, connect spies}]
\node[anchor=south west,inner sep=0]  at (0,0) {\includegraphics[width=\linewidth]{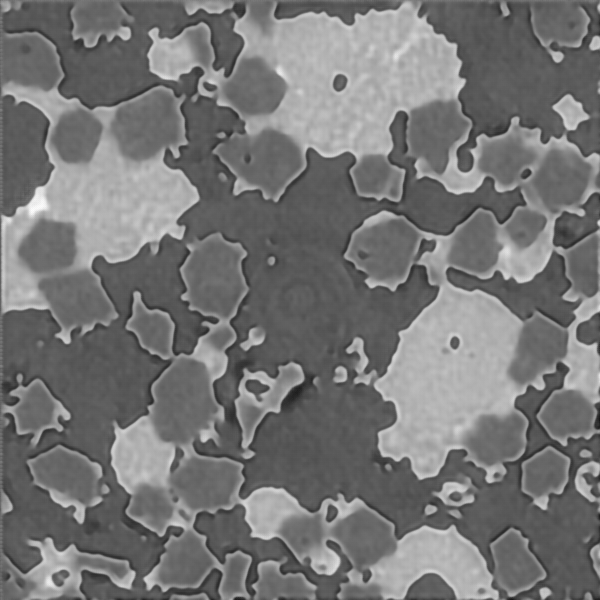}};
\spy on (1.26,.84) in node [right] at (0.02,-1.07);
\end{tikzpicture}
\caption*{Mean image}
\end{subfigure}%
\hfill
\begin{subfigure}[t]{.14\textwidth}
\begin{tikzpicture}[spy using outlines={rectangle,black,magnification=10,size=2.09cm, connect spies}]
\node[anchor=south west,inner sep=0]  at (0,0) {\includegraphics[width=\linewidth]{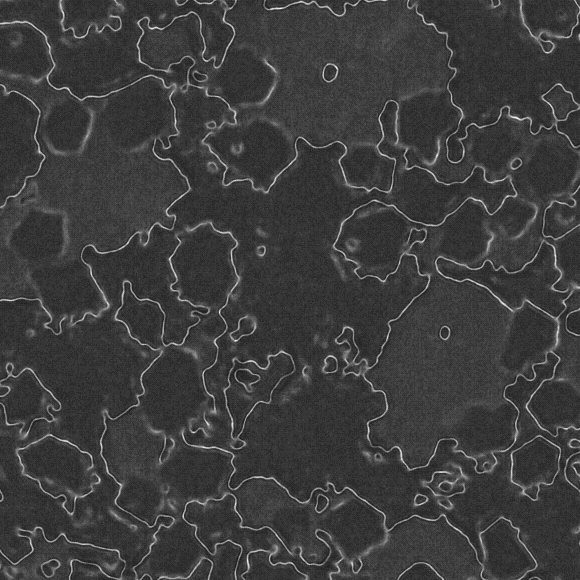}};
\spy on (1.26,.84) in node [right] at (0.02,-1.07);
\end{tikzpicture}
\caption*{Std}
\end{subfigure}%

\begin{subfigure}[t]{.14\textwidth}
\begin{tikzpicture}[spy using outlines={rectangle,black,magnification=10,size=2.09cm, connect spies}]
\node[anchor=south west,inner sep=0]  at (0,0) {\includegraphics[width=\linewidth]{imgs/SiC/hr_SiC.png}};
\spy on (1.26,.84) in node [right] at (0.02,-1.07);
\end{tikzpicture}
\caption*{HR image}
\end{subfigure}%
\hfill
\begin{subfigure}[t]{.14\textwidth}
\begin{tikzpicture}[spy using outlines={rectangle,black,magnification=10,size=2.09cm, connect spies}]
\node[anchor=south west,inner sep=0]  at (0,0) {\includegraphics[width=\linewidth]{imgs/SiC/lr_SiC_x4.png}};
\spy on (1.24,.863) in node [right] at (0.02,-1.07);
\end{tikzpicture}
\caption*{LR image}
\end{subfigure}%
\hfill
\begin{subfigure}[t]{.14\textwidth}
\begin{tikzpicture}[spy using outlines={rectangle,black,magnification=10,size=2.09cm, connect spies}]
\node[anchor=south west,inner sep=0]  at (0,0) {\includegraphics[width=\linewidth]{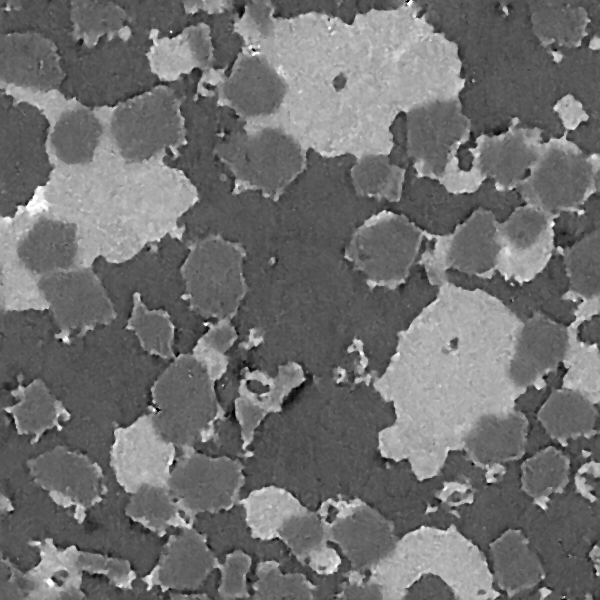}};
\spy on (1.26,.84) in node [right] at (0.02,-1.07);
\end{tikzpicture}
\caption*{Prediction 1}
\end{subfigure}%
\hfill
\begin{subfigure}[t]{.14\textwidth}
\begin{tikzpicture}[spy using outlines={rectangle,black,magnification=10,size=2.09cm, connect spies}]
\node[anchor=south west,inner sep=0]  at (0,0) {\includegraphics[width=\linewidth]{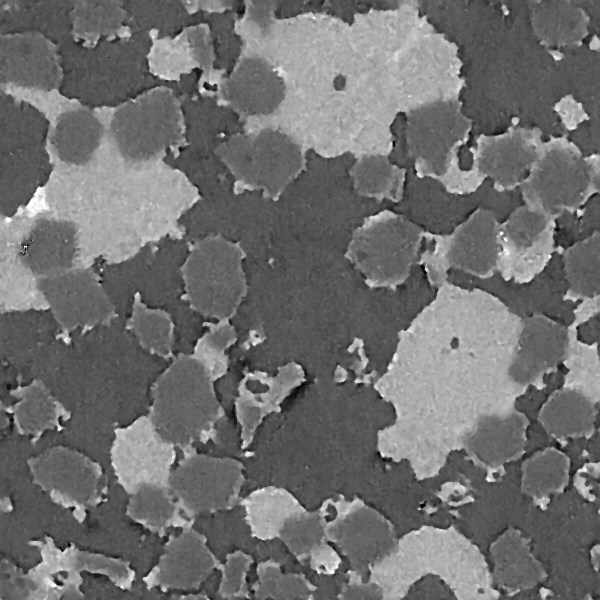}};
\spy on (1.26,.84) in node [right] at (0.02,-1.07);
\end{tikzpicture}
\caption*{Prediction 2}
\end{subfigure}%
\hfill
\begin{subfigure}[t]{.14\textwidth}
\begin{tikzpicture}[spy using outlines={rectangle,black,magnification=10,size=2.09cm, connect spies}]
\node[anchor=south west,inner sep=0]  at (0,0) {\includegraphics[width=\linewidth]{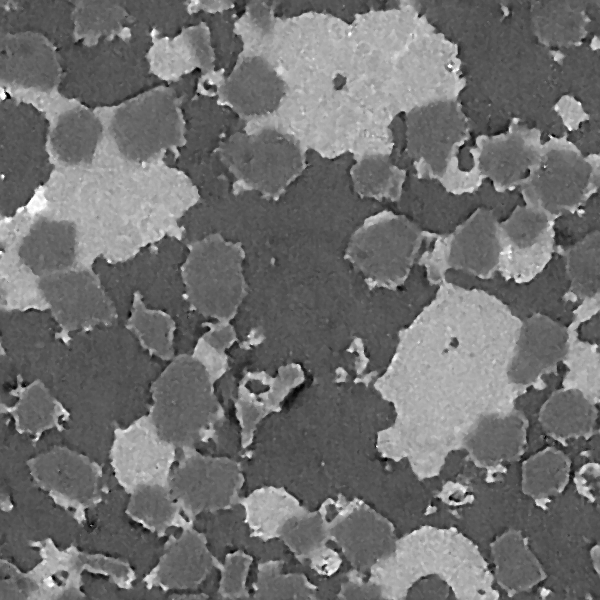}};
\spy on (1.26,.84) in node [right] at (0.02,-1.07);
\end{tikzpicture}
\caption*{Prediction 3}
\end{subfigure}%
\hfill
\begin{subfigure}[t]{.14\textwidth}
\begin{tikzpicture}[spy using outlines={rectangle,black,magnification=10,size=2.09cm, connect spies}]
\node[anchor=south west,inner sep=0]  at (0,0) {\includegraphics[width=\linewidth]{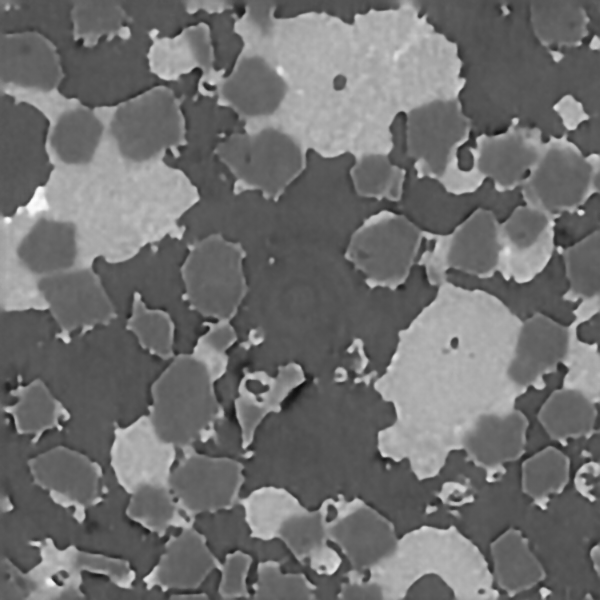}};
\spy on (1.26,.84) in node [right] at (0.02,-1.07);
\end{tikzpicture}
\caption*{Mean image}
\end{subfigure}%
\hfill
\begin{subfigure}[t]{.14\textwidth}
\begin{tikzpicture}[spy using outlines={rectangle,black,magnification=10,size=2.09cm, connect spies}]
\node[anchor=south west,inner sep=0]  at (0,0) {\includegraphics[width=\linewidth]{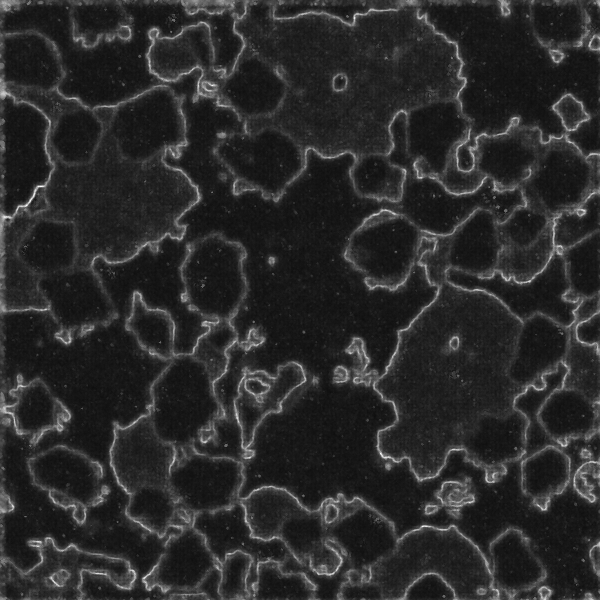}};
\spy on (1.26,.84) in node [right] at (0.02,-1.07);
\end{tikzpicture}
\caption*{Std}
\end{subfigure}%

\begin{subfigure}[t]{.14\textwidth}
\begin{tikzpicture}[spy using outlines={rectangle,black,magnification=10,size=2.09cm, connect spies}]
\node[anchor=south west,inner sep=0]  at (0,0) {\includegraphics[width=\linewidth]{imgs/SiC/hr_SiC.png}};
\spy on (1.26,.84) in node [right] at (0.02,-1.07);
\end{tikzpicture}
\caption*{HR image}
\end{subfigure}%
\hfill
\begin{subfigure}[t]{.14\textwidth}
\begin{tikzpicture}[spy using outlines=
{rectangle,black,magnification=10,size=2.09cm, connect spies}]
\node[anchor=south west,inner sep=0]  at (0,0) {\includegraphics[width=\linewidth]{imgs/SiC/lr_SiC_x4.png}};
\spy on (1.24,.863) in node [right] at (0.02,-1.07);
\end{tikzpicture}
\caption*{LR image}
\end{subfigure}%
\hfill
\begin{subfigure}[t]{.14\textwidth}
\begin{tikzpicture}[spy using outlines={rectangle,black,magnification=10,size=2.09cm, connect spies}]
\node[anchor=south west,inner sep=0]  at (0,0) {\includegraphics[width=\linewidth]{imgs/SiC/img13.png}};
\spy on (1.26,.84) in node [right] at (0.02,-1.07);
\end{tikzpicture}
\caption*{Prediction 1}
\end{subfigure}%
\hfill
\begin{subfigure}[t]{.14\textwidth}
\begin{tikzpicture}[spy using outlines={rectangle,black,magnification=10,size=2.09cm, connect spies}]
\node[anchor=south west,inner sep=0]  at (0,0) {\includegraphics[width=\linewidth]{imgs/SiC/img15.png}};
\spy on (1.26,.84) in node [right] at (0.02,-1.07);
\end{tikzpicture}
\caption*{Prediction 2}
\end{subfigure}%
\hfill
\begin{subfigure}[t]{.14\textwidth}
\begin{tikzpicture}[spy using outlines={rectangle,black,magnification=10,size=2.09cm, connect spies}]
\node[anchor=south west,inner sep=0]  at (0,0) {\includegraphics[width=\linewidth]{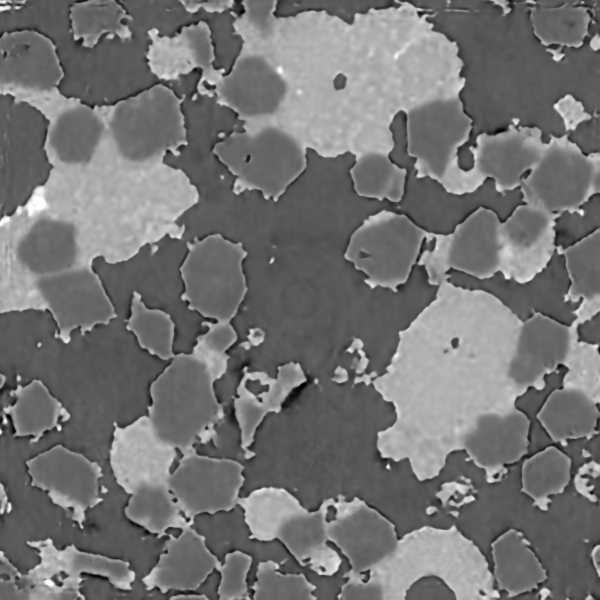}};
\spy on (1.26,.84) in node [right] at (0.02,-1.07);
\end{tikzpicture}
\caption*{Prediction 3}
\end{subfigure}%
\hfill
\begin{subfigure}[t]{.14\textwidth}
\begin{tikzpicture}[spy using outlines={rectangle,black,magnification=10,size=2.09cm, connect spies}]
\node[anchor=south west,inner sep=0]  at (0,0) {\includegraphics[width=\linewidth]{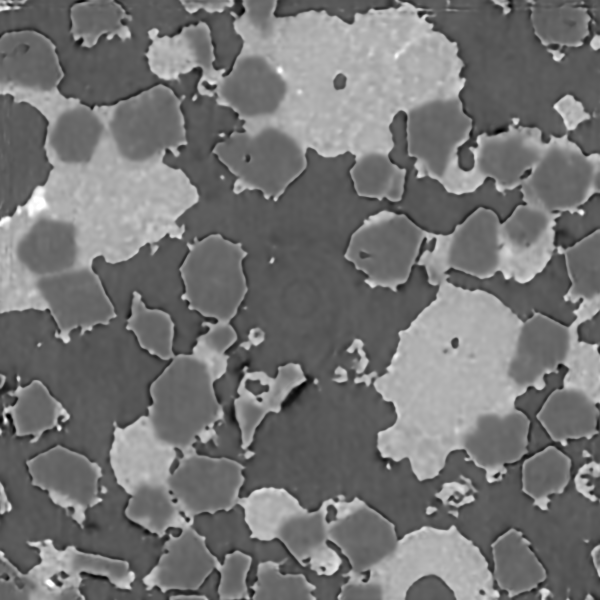}};
\spy on (1.26,.84) in node [right] at (0.02,-1.07);
\end{tikzpicture}
\caption*{Mean image}
\end{subfigure}%
\hfill
\begin{subfigure}[t]{.14\textwidth}
\begin{tikzpicture}[spy using outlines={rectangle,black,magnification=10,size=2.09cm, connect spies}]
\node[anchor=south west,inner sep=0]  at (0,0) {\includegraphics[width=\linewidth]{imgs/SiC/std_normalized.png}};
\spy on (1.26,.84) in node [right] at (0.02,-1.07);
\end{tikzpicture}
\caption*{Std}
\end{subfigure}%

\caption{Different  WPPFlow (top), SRFlow (middle) and conditional MMD flow (bottom) reconstructions of the HR image and their mean and pixelwise standard deviation (normalized). The zoomed-in part is marked with a black box.} \label{fig:superres_SiC}
\end{figure}

\begin{table}[t]
\begin{center}
\scalebox{.85}{
\begin{tabular}[t]{c|ccc} 
              & WPPFlow (mean img)   & SRFlow (mean img) & Sliced MMD Flow (mean img) \\
\hline
PSNR         & 25.89 (26.92)  & 25.11 (27.36)  &  \textbf{27.21} (27.93)      \\ 
SSIM \citep{WBSS04}         & 0.657 (0.766)  & 0.637 (0.771)  &  \textbf{0.774} (0.790)     \\ 
\end{tabular}} 
\caption{Comparison of superresolution results of different reconstruction schemes. The values are meaned of 100 reconstructions for each observation. The value of the resulting mean image is in brackets. The best value is marked in bold.}
\label{tab:SiC_superres}
\end{center}
\end{table}

\begin{figure}
\centering    
\begin{subfigure}[t]{.14\textwidth}
\begin{tikzpicture}[spy using outlines={rectangle,black,magnification=10,size=2.09cm, connect spies}]
\node[anchor=south west,inner sep=0]  at (0,0) {\includegraphics[width=\linewidth]{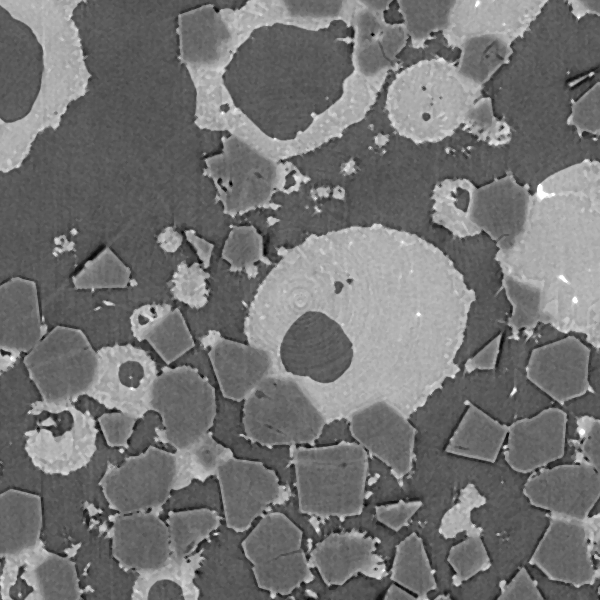}};
\spy on (.56,.98) in node [right] at (0.02,-1.07);
\end{tikzpicture}
\end{subfigure}%
\hfill
\begin{subfigure}[t]{.14\textwidth}
\begin{tikzpicture}[spy using outlines={rectangle,black,magnification=10,size=2.09cm, connect spies}]
\node[anchor=south west,inner sep=0]  at (0,0) {\includegraphics[width=\linewidth]{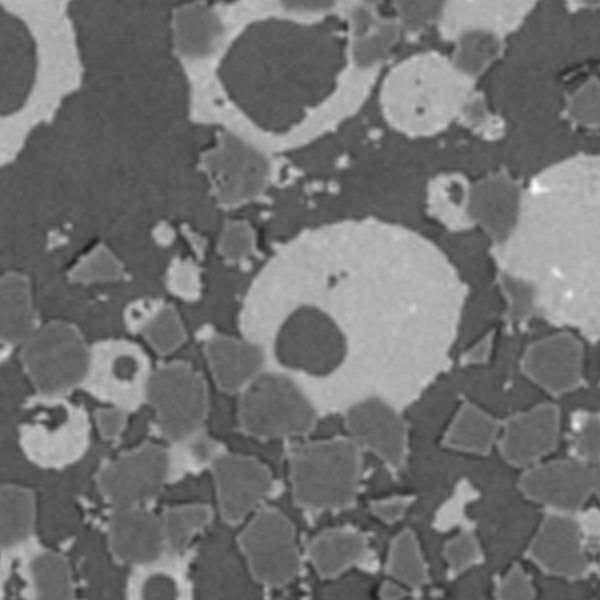}};
\spy on (.537,.999) in node [right] at (0.02,-1.07);
\end{tikzpicture}
\end{subfigure}%
\hfill
\begin{subfigure}[t]{.14\textwidth}
\begin{tikzpicture}[spy using outlines={rectangle,black,magnification=10,size=2.09cm, connect spies}]
\node[anchor=south west,inner sep=0]  at (0,0) {\includegraphics[width=\linewidth]{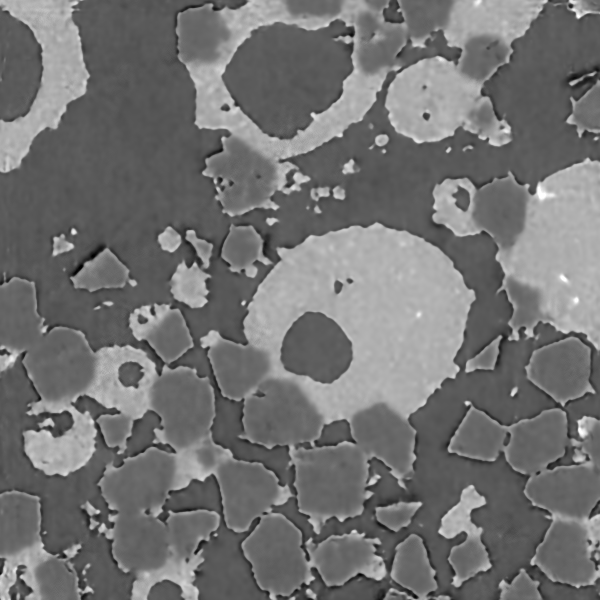}};
\spy on (.56,.98) in node [right] at (0.02,-1.07);
\end{tikzpicture}
\end{subfigure}%
\hfill
\begin{subfigure}[t]{.14\textwidth}
\begin{tikzpicture}[spy using outlines={rectangle,black,magnification=10,size=2.09cm, connect spies}]
\node[anchor=south west,inner sep=0]  at (0,0) {\includegraphics[width=\linewidth]{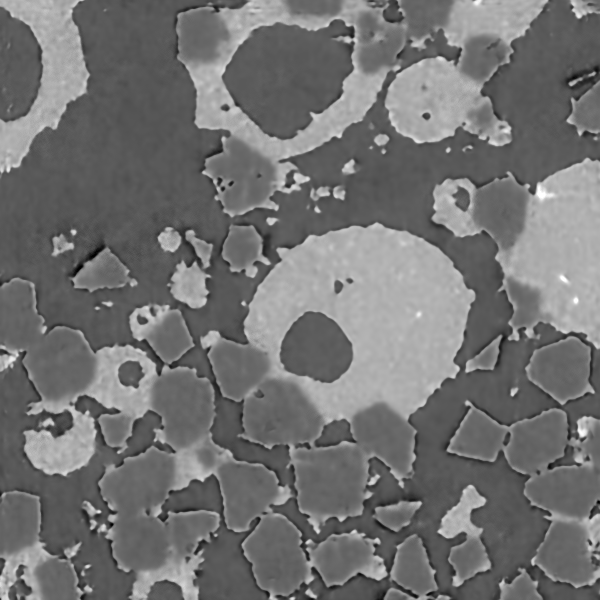}};
\spy on (.56,.98) in node [right] at (0.02,-1.07);
\end{tikzpicture}
\end{subfigure}%
\hfill
\begin{subfigure}[t]{.14\textwidth}
\begin{tikzpicture}[spy using outlines={rectangle,black,magnification=10,size=2.09cm, connect spies}]
\node[anchor=south west,inner sep=0]  at (0,0) {\includegraphics[width=\linewidth]{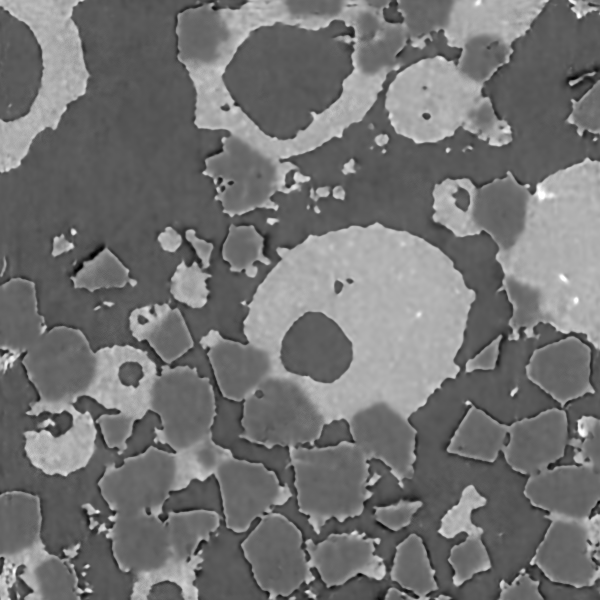}};
\spy on (.56,.98) in node [right] at (0.02,-1.07);
\end{tikzpicture}
\end{subfigure}%
\hfill
\begin{subfigure}[t]{.14\textwidth}
\begin{tikzpicture}[spy using outlines={rectangle,black,magnification=10,size=2.09cm, connect spies}]
\node[anchor=south west,inner sep=0]  at (0,0) {\includegraphics[width=\linewidth]{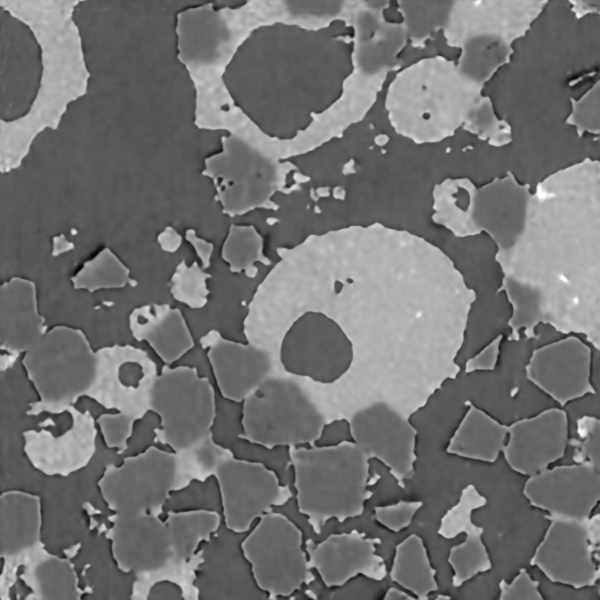}};
\spy on (.56,.98) in node [right] at (0.02,-1.07);
\end{tikzpicture}
\end{subfigure}%
\hfill
\begin{subfigure}[t]{.14\textwidth}
\begin{tikzpicture}[spy using outlines={rectangle,black,magnification=10,size=2.09cm, connect spies}]
\node[anchor=south west,inner sep=0]  at (0,0) {\includegraphics[width=\linewidth]{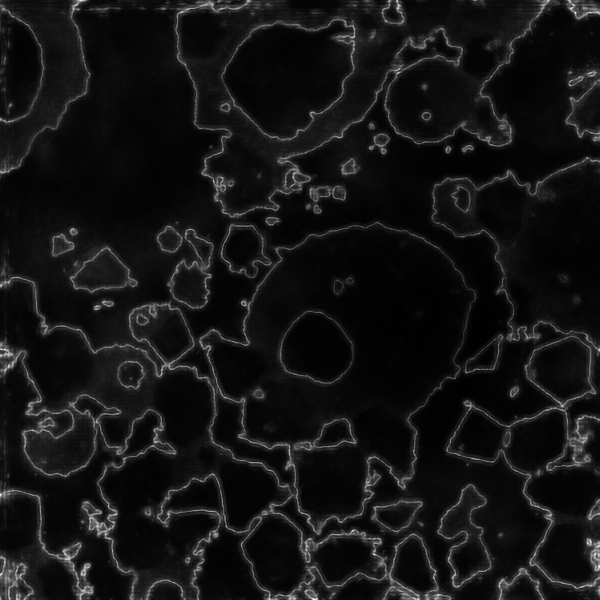}};
\spy on (.56,.98) in node [right] at (0.02,-1.07);
\end{tikzpicture}
\end{subfigure}%

\begin{subfigure}[t]{.14\textwidth}
\begin{tikzpicture}[spy using outlines={rectangle,black,magnification=10,size=2.09cm, connect spies}]
\node[anchor=south west,inner sep=0]  at (0,0) {\includegraphics[width=\linewidth]{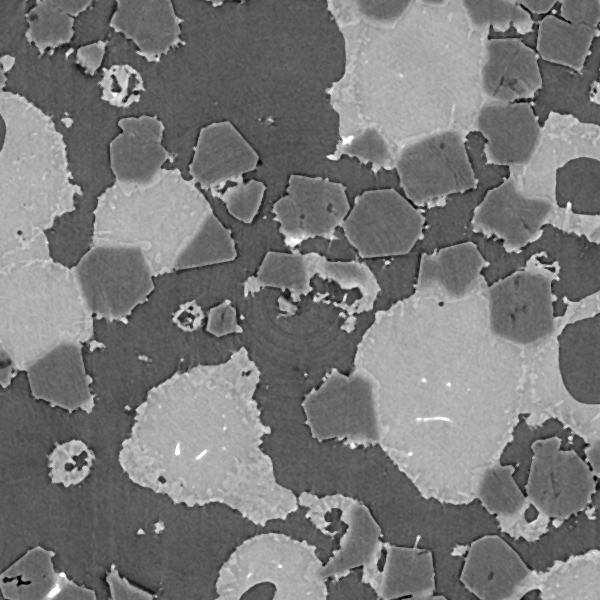}};
\spy on (1.2,1.06) in node [right] at (0.02,-1.07);
\end{tikzpicture}
\caption*{HR image}
\end{subfigure}%
\hfill
\begin{subfigure}[t]{.14\textwidth}
\begin{tikzpicture}[spy using outlines={rectangle,black,magnification=10,size=2.09cm, connect spies}]
\node[anchor=south west,inner sep=0]  at (0,0) {\includegraphics[width=\linewidth]{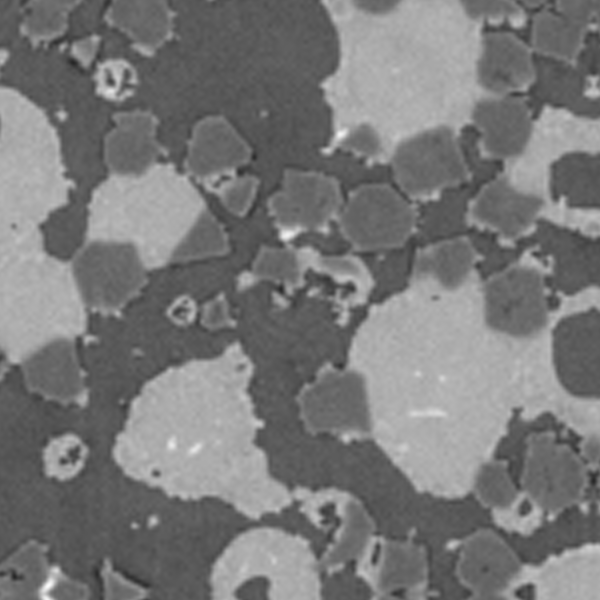}};
\spy on (1.178,1.083) in node [right] at (0.02,-1.07);
\end{tikzpicture}
\caption*{LR image}
\end{subfigure}%
\hfill
\begin{subfigure}[t]{.14\textwidth}
\begin{tikzpicture}[spy using outlines={rectangle,black,magnification=10,size=2.09cm, connect spies}]
\node[anchor=south west,inner sep=0]  at (0,0) {\includegraphics[width=\linewidth]{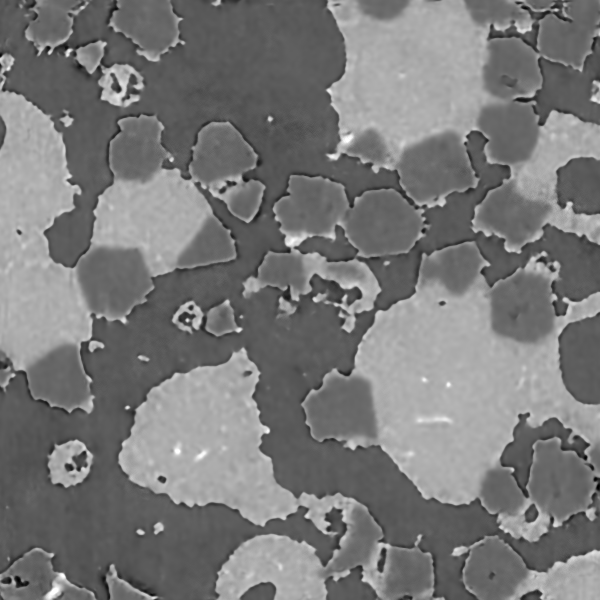}};
\spy on (1.2,1.06) in node [right] at (0.02,-1.07);
\end{tikzpicture}
\caption*{Prediction 1}
\end{subfigure}%
\hfill
\begin{subfigure}[t]{.14\textwidth}
\begin{tikzpicture}[spy using outlines={rectangle,black,magnification=10,size=2.09cm, connect spies}]
\node[anchor=south west,inner sep=0]  at (0,0) {\includegraphics[width=\linewidth]{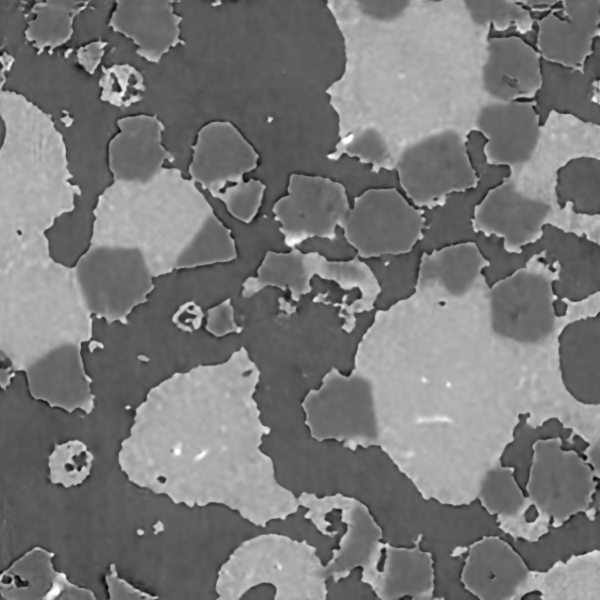}};
\spy on (1.2,1.06) in node [right] at (0.02,-1.07);
\end{tikzpicture}
\caption*{Prediction 2}
\end{subfigure}%
\hfill
\begin{subfigure}[t]{.14\textwidth}
\begin{tikzpicture}[spy using outlines={rectangle,black,magnification=10,size=2.09cm, connect spies}]
\node[anchor=south west,inner sep=0]  at (0,0) {\includegraphics[width=\linewidth]{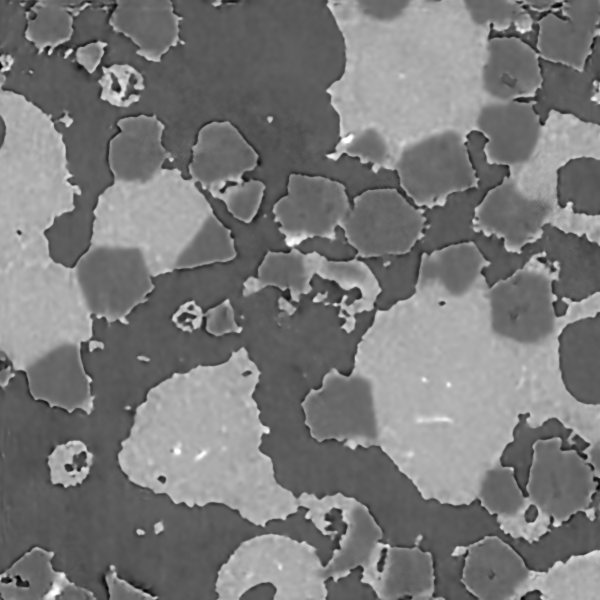}};
\spy on (1.2,1.06) in node [right] at (0.02,-1.07);
\end{tikzpicture}
\caption*{Prediction 3}
\end{subfigure}%
\hfill
\begin{subfigure}[t]{.14\textwidth}
\begin{tikzpicture}[spy using outlines={rectangle,black,magnification=10,size=2.09cm, connect spies}]
\node[anchor=south west,inner sep=0]  at (0,0) {\includegraphics[width=\linewidth]{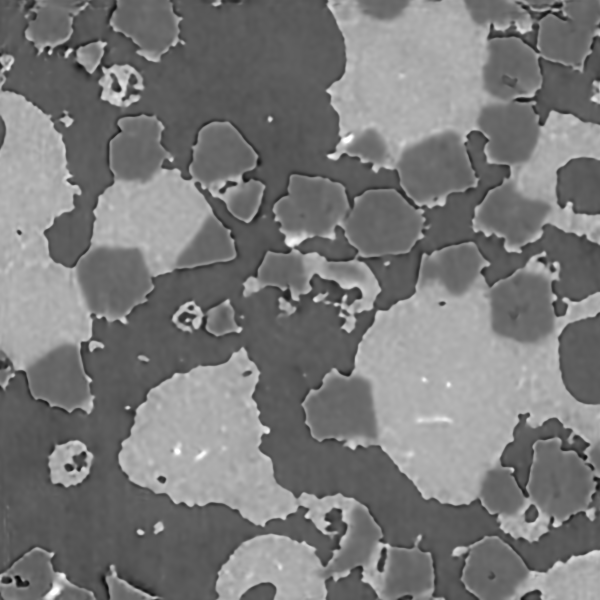}};
\spy on (1.2,1.06) in node [right] at (0.02,-1.07);
\end{tikzpicture}
\caption*{Mean image}
\end{subfigure}%
\hfill
\begin{subfigure}[t]{.14\textwidth}
\begin{tikzpicture}[spy using outlines={rectangle,black,magnification=10,size=2.09cm, connect spies}]
\node[anchor=south west,inner sep=0]  at (0,0) {\includegraphics[width=\linewidth]{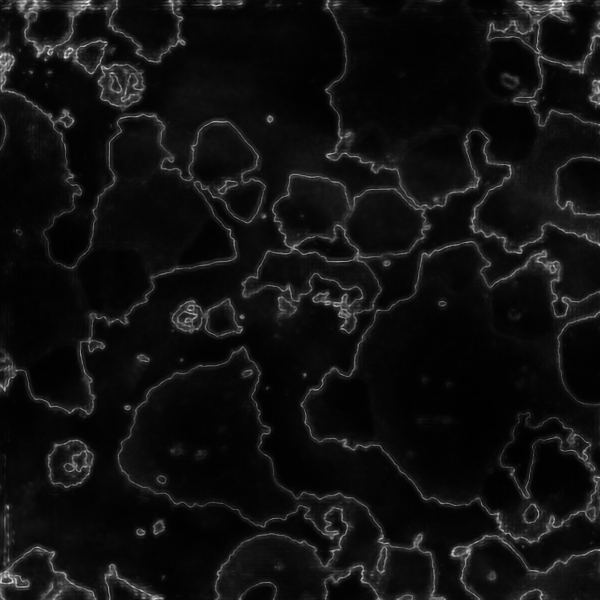}};
\spy on (1.2,1.06) in node [right] at (0.02,-1.07);
\end{tikzpicture}
\caption*{Std}
\end{subfigure}%

\begin{subfigure}[t]{.14\textwidth}
\begin{tikzpicture}[spy using outlines={rectangle,black,magnification=12,size=2.09cm, connect spies}]
\node[anchor=south west,inner sep=0]  at (0,0) {\includegraphics[width=\linewidth]{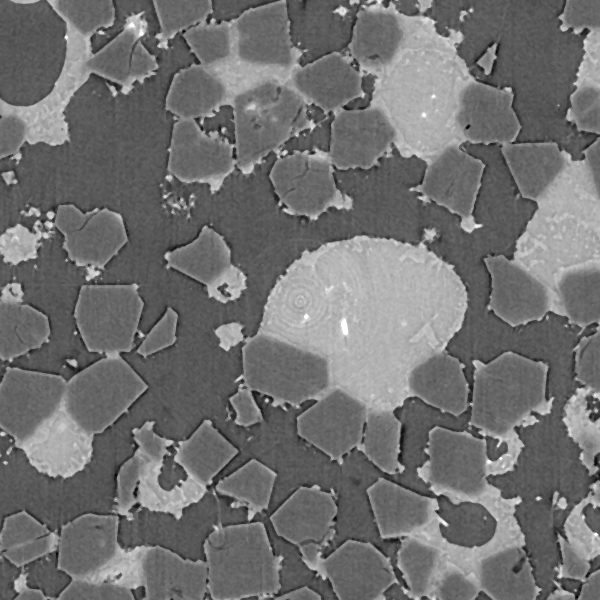}};
\spy on (.77,1.68) in node [right] at (0.02,-1.07);
\end{tikzpicture}
\end{subfigure}%
\hfill
\begin{subfigure}[t]{.14\textwidth}
\begin{tikzpicture}[spy using outlines={rectangle,black,magnification=12,size=2.09cm, connect spies}]
\node[anchor=south west,inner sep=0]  at (0,0) {\includegraphics[width=\linewidth]{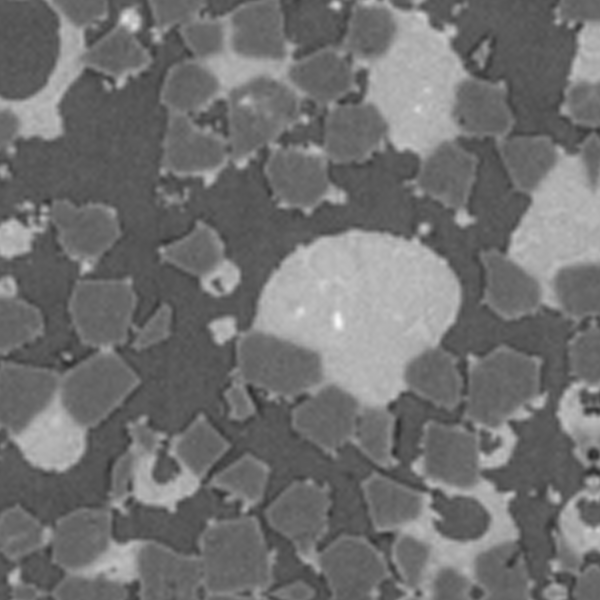}};
\spy on (.747,1.703) in node [right] at (0.02,-1.07);
\end{tikzpicture}
\end{subfigure}%
\hfill
\begin{subfigure}[t]{.14\textwidth}
\begin{tikzpicture}[spy using outlines={rectangle,black,magnification=12,size=2.09cm, connect spies}]
\node[anchor=south west,inner sep=0]  at (0,0) {\includegraphics[width=\linewidth]{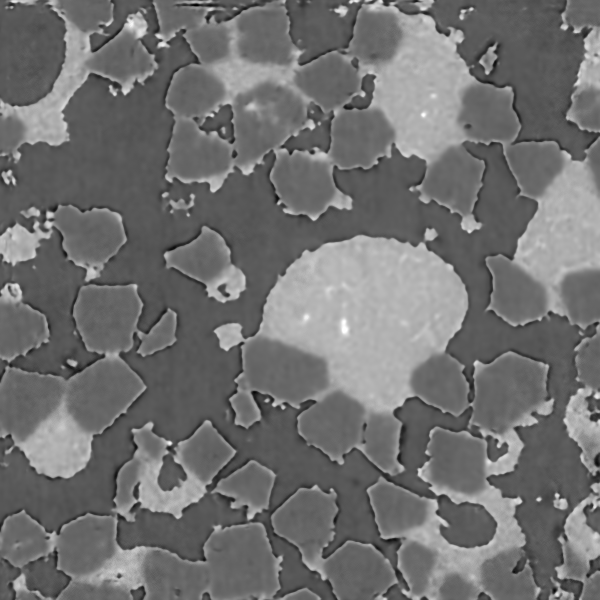}};
\spy on (.77,1.68) in node [right] at (0.02,-1.07);
\end{tikzpicture}
\end{subfigure}%
\hfill
\begin{subfigure}[t]{.14\textwidth}
\begin{tikzpicture}[spy using outlines={rectangle,black,magnification=12,size=2.09cm, connect spies}]
\node[anchor=south west,inner sep=0]  at (0,0) {\includegraphics[width=\linewidth]{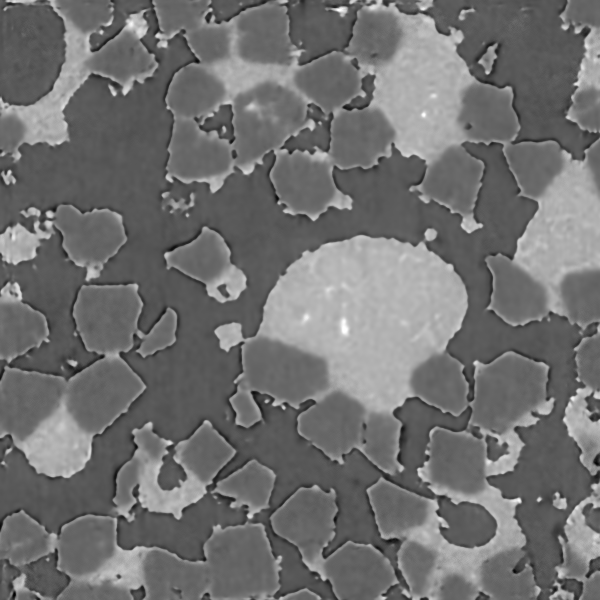}};
\spy on (.77,1.68) in node [right] at (0.02,-1.07);
\end{tikzpicture}
\end{subfigure}%
\hfill
\begin{subfigure}[t]{.14\textwidth}
\begin{tikzpicture}[spy using outlines={rectangle,black,magnification=12,size=2.09cm, connect spies}]
\node[anchor=south west,inner sep=0]  at (0,0) {\includegraphics[width=\linewidth]{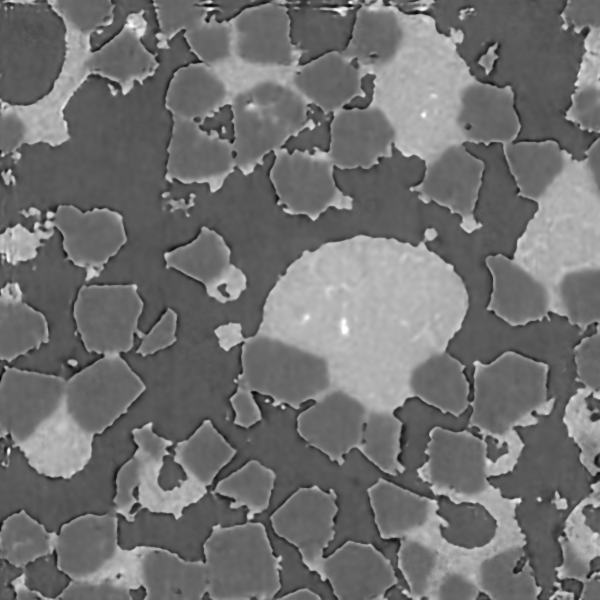}};
\spy on (.77,1.68) in node [right] at (0.02,-1.07);
\end{tikzpicture}
\end{subfigure}%
\hfill
\begin{subfigure}[t]{.14\textwidth}
\begin{tikzpicture}[spy using outlines={rectangle,black,magnification=12,size=2.09cm, connect spies}]
\node[anchor=south west,inner sep=0]  at (0,0) {\includegraphics[width=\linewidth]{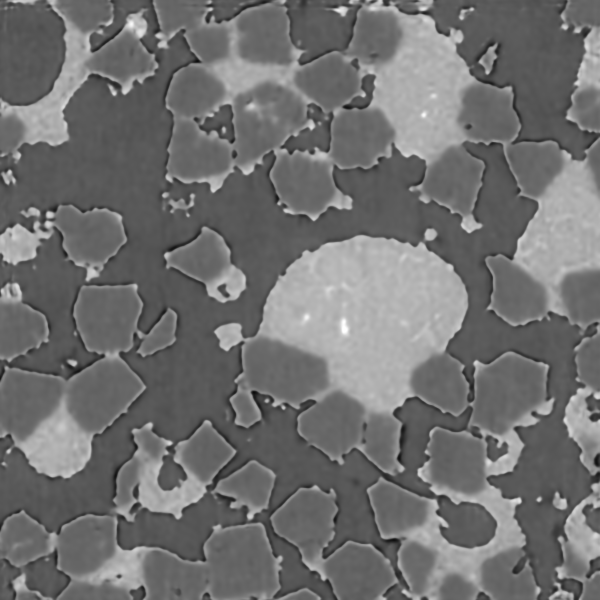}};
\spy on (.77,1.68) in node [right] at (0.02,-1.07);
\end{tikzpicture}
\end{subfigure}%
\hfill
\begin{subfigure}[t]{.14\textwidth}
\begin{tikzpicture}[spy using outlines={rectangle,black,magnification=12,size=2.09cm, connect spies}]
\node[anchor=south west,inner sep=0]  at (0,0) {\includegraphics[width=\linewidth]{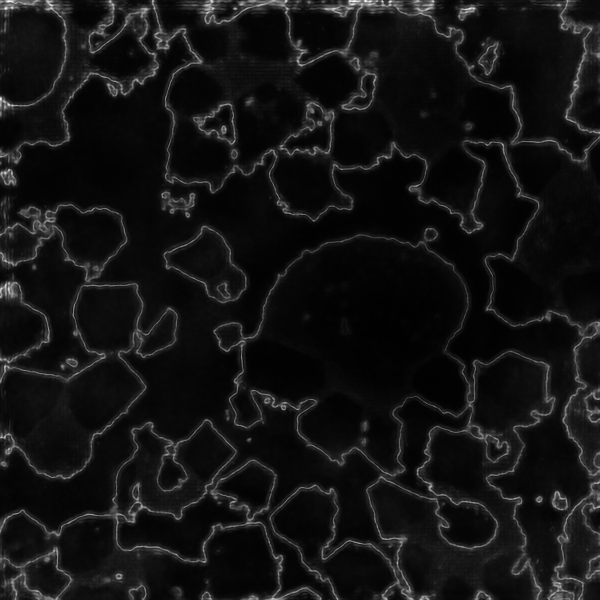}};
\spy on (.77,1.68) in node [right] at (0.02,-1.07);
\end{tikzpicture}
\end{subfigure}%

\begin{subfigure}[t]{.14\textwidth}
\begin{tikzpicture}[spy using outlines={rectangle,black,magnification=11,size=2.09cm, connect spies}]
\node[anchor=south west,inner sep=0]  at (0,0) {\includegraphics[width=\linewidth]{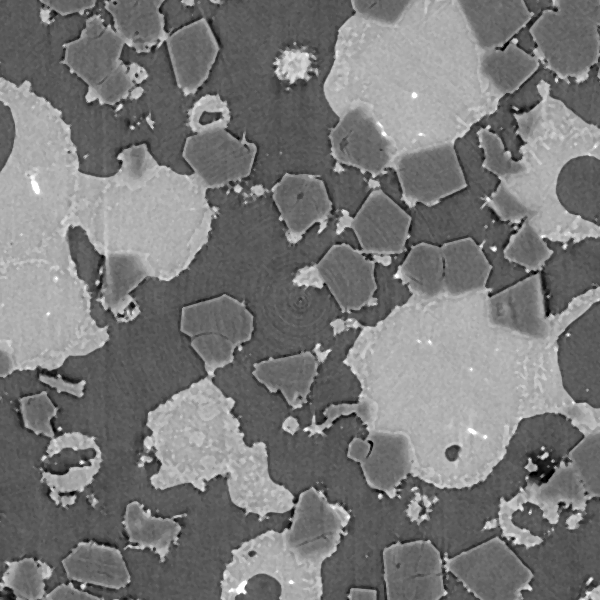}};
\spy on (1.47,.36) in node [right] at (0.02,-1.07);
\end{tikzpicture}
\caption*{HR image}
\end{subfigure}%
\hfill
\begin{subfigure}[t]{.14\textwidth}
\begin{tikzpicture}[spy using outlines={rectangle,black,magnification=11,size=2.09cm, connect spies}]
\node[anchor=south west,inner sep=0]  at (0,0) {\includegraphics[width=\linewidth]{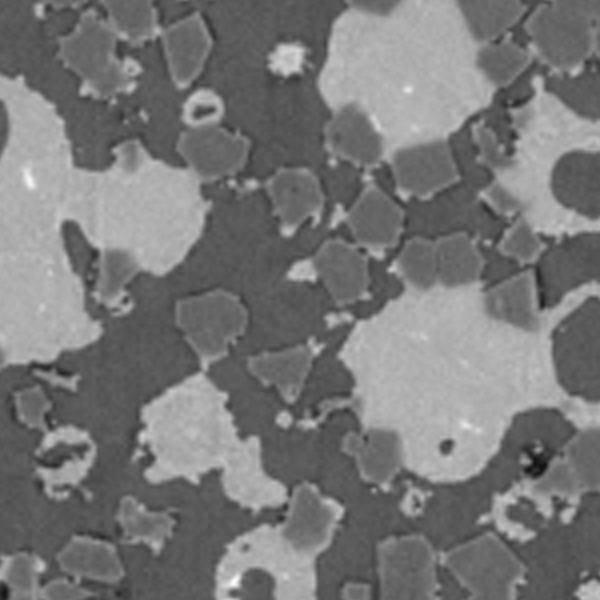}};
\spy on (1.447,.383) in node [right] at (0.02,-1.07);
\end{tikzpicture}
\caption*{LR image}
\end{subfigure}%
\hfill
\begin{subfigure}[t]{.14\textwidth}
\begin{tikzpicture}[spy using outlines={rectangle,black,magnification=11,size=2.09cm, connect spies}]
\node[anchor=south west,inner sep=0]  at (0,0) {\includegraphics[width=\linewidth]{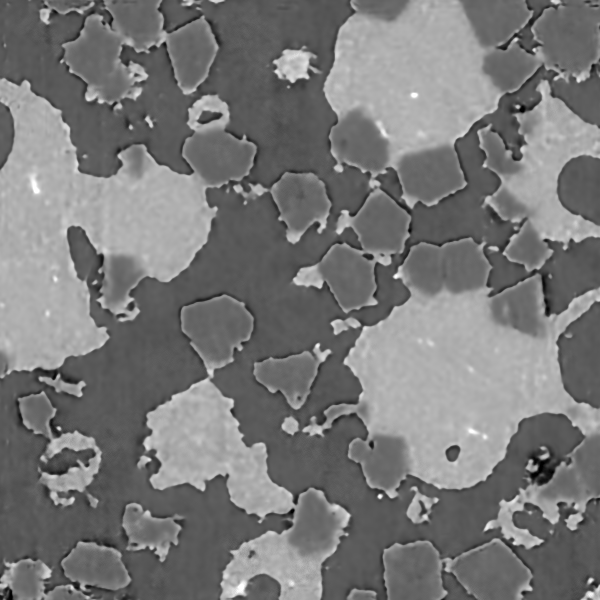}};
\spy on (1.47,.36) in node [right] at (0.02,-1.07);
\end{tikzpicture}
\caption*{Prediction 1}
\end{subfigure}%
\hfill
\begin{subfigure}[t]{.14\textwidth}
\begin{tikzpicture}[spy using outlines={rectangle,black,magnification=11,size=2.09cm, connect spies}]
\node[anchor=south west,inner sep=0]  at (0,0) {\includegraphics[width=\linewidth]{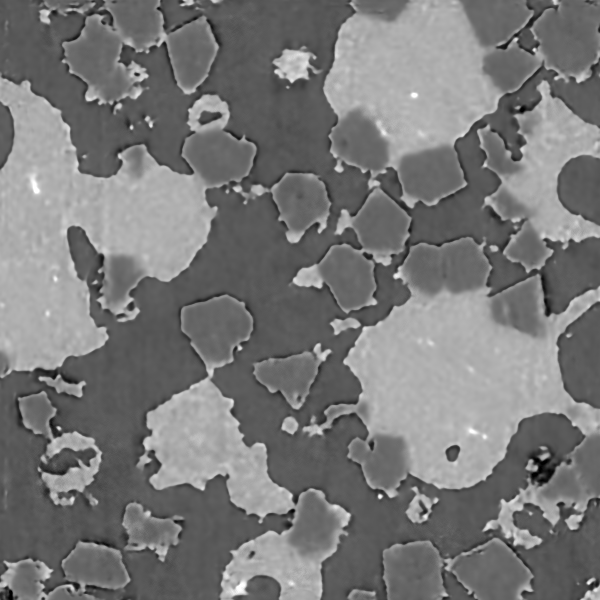}};
\spy on (1.47,.36) in node [right] at (0.02,-1.07);
\end{tikzpicture}
\caption*{Prediction 2}
\end{subfigure}%
\hfill
\begin{subfigure}[t]{.14\textwidth}
\begin{tikzpicture}[spy using outlines={rectangle,black,magnification=11,size=2.09cm, connect spies}]
\node[anchor=south west,inner sep=0]  at (0,0) {\includegraphics[width=\linewidth]{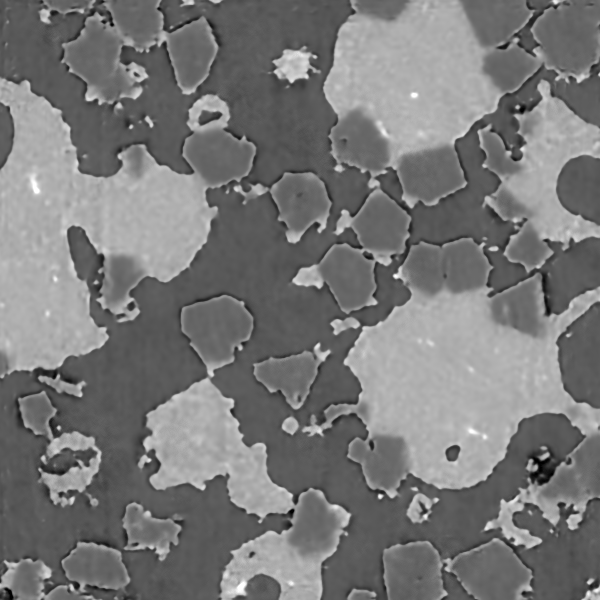}};
\spy on (1.47,.36) in node [right] at (0.02,-1.07);
\end{tikzpicture}
\caption*{Prediction 3}
\end{subfigure}%
\hfill
\begin{subfigure}[t]{.14\textwidth}
\begin{tikzpicture}[spy using outlines={rectangle,black,magnification=11,size=2.09cm, connect spies}]
\node[anchor=south west,inner sep=0]  at (0,0) {\includegraphics[width=\linewidth]{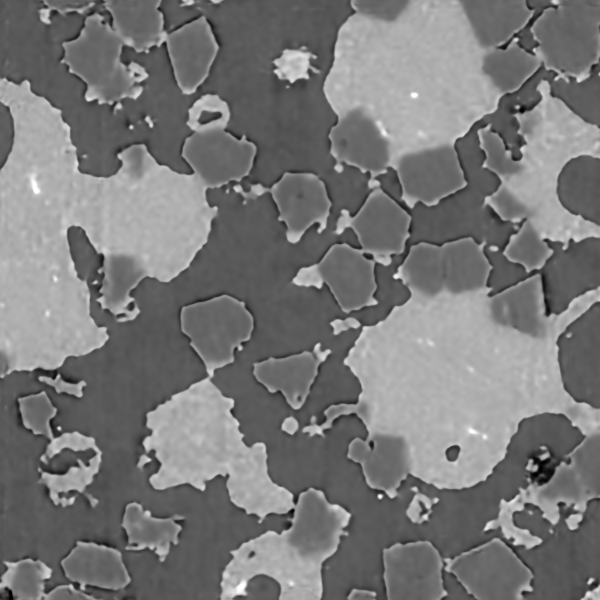}};
\spy on (1.47,.36) in node [right] at (0.02,-1.07);
\end{tikzpicture}
\caption*{Mean image}
\end{subfigure}%
\hfill
\begin{subfigure}[t]{.14\textwidth}
\begin{tikzpicture}[spy using outlines={rectangle,black,magnification=11,size=2.09cm, connect spies}]
\node[anchor=south west,inner sep=0]  at (0,0) {\includegraphics[width=\linewidth]{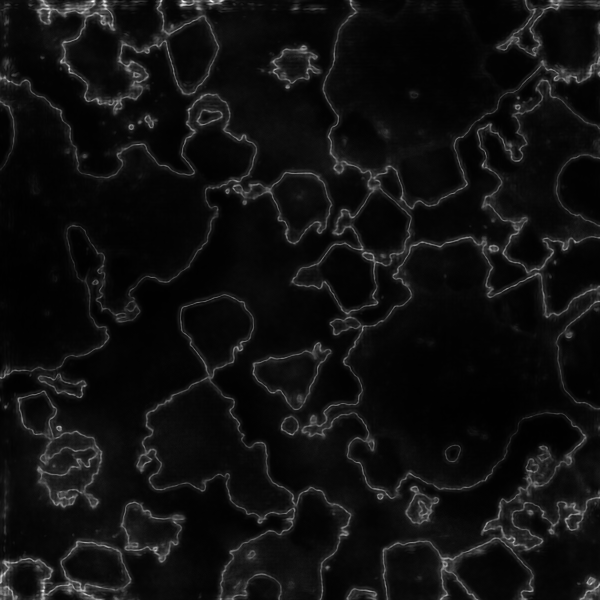}};
\spy on (1.47,.36) in node [right] at (0.02,-1.07);
\end{tikzpicture}
\caption*{Std}
\end{subfigure}%
\caption{Additional conditional MMD flow reconstructions of the HR image and their mean and pixelwise standard deviation (normalized). The zoomed-in part is marked with a black box.} \label{fig:add_SiC_examples}
\end{figure}

\end{document}